\documentclass[11pt,twoside]{article}
\usepackage{subcaption}
\usepackage{natbib}

\usepackage{amsmath,amsfonts,bm}

\def\eqref#1{equation~(\ref{#1})}
\def\Eqref#1{Equation~(\ref{#1})}
\def\1{\bm{1}}

\DeclareMathAlphabet{\mathsfit}{\encodingdefault}{\sfdefault}{m}{sl}
\SetMathAlphabet{\mathsfit}{bold}{\encodingdefault}{\sfdefault}{bx}{n}

\DeclareMathOperator*{\argmin}{arg\,min}

\newcommand{\azsj}{\alpha^\star_{0j}}
\newcommand{\aosj}{\alpha^\star_{1j}}
\newcommand{\bzsj}{\beta^\star_{0j}}

\newcommand{\bsj}{\beta^\star_{j}}
\newcommand{\ssj}{\sigma^{\star2}_{j}}

\newcommand{\azni}{\alpha_{n,0i}}
\newcommand{\aoni}{\alpha_{n,1i}}

\newcommand{\bni}{\beta_{n,i}}
\newcommand{\sni}{\sigma^{2}_{n,i}}

\newcommand{\aznj}{\alpha_{n,0j}}
\newcommand{\aonj}{\alpha_{n,1j}}

\newcommand{\bnj}{\beta_{n,j}}
\newcommand{\snj}{\sigma^{2}_{n,j}}

\newcommand{\dazni}{\Delta\alpha_{n,0ij}}
\newcommand{\daoni}{\Delta\alpha_{n,1ij}}

\newcommand{\dbni}{\Delta\beta_{n,ij}}
\newcommand{\dsni}{\Delta\sigma^2_{n,ij}}

\newcommand{\daonj}{\Delta\alpha_{n,1j}}

\newcommand{\dbnj}{\Delta\beta_{n,j}}
\newcommand{\dsnj}{\Delta\sigma^2_{n,j}}

\usepackage[utf8]{inputenc} % allow utf-8 input
\usepackage[T1]{fontenc}    % use 8-bit T1 fonts
\usepackage{booktabs}       % professional-quality tables
\usepackage{amsfonts}       % blackboard math symbols
\usepackage{nicefrac}       % compact symbols for 1/2, etc.
\usepackage{microtype}      % microtypography

\usepackage{epsf}
\usepackage{epsfig}
\usepackage{fancyhdr}
\usepackage{graphics}
\usepackage{graphicx}
\usepackage{psfrag}
\usepackage{fullpage}
\usepackage{pdfpages}

\usepackage{url}% for url's in bib
\usepackage[colorlinks,linkcolor=black,citecolor=blue, pagebackref=true]{hyperref}
\renewcommand*{\backrefalt}[4]{%
    \ifcase #1 \footnotesize{(Not cited.)}%
    \or        \footnotesize{(Cited on page~#2.)}%
    \else      \footnotesize{(Cited on pages~#2.)}%
    \fi}
\usepackage{color}
\usepackage{amsthm}
\usepackage{amsmath}
\usepackage{amssymb,bbm}
\usepackage{caption}
\usepackage{algorithmic}
\usepackage{algorithm}
\usepackage{textcomp}
\usepackage{siunitx}
\usepackage{wrapfig}
\usepackage{algorithmic}
\usepackage{algorithm}
\usepackage{mathrsfs}  
\usepackage{multirow}
\usepackage{multicol}% colors

\newcommand{\kl}{\textnormal{K}}

\newtheorem{assumption}{Assumption}
\newtheorem{example}{Example}
\newtheorem{theorem}{Theorem}
\newtheorem{proposition}{Proposition}
\newtheorem{definition}{Definition}
\newtheorem{corollary}{Corollary}%opening

\def\argmin{\textnormal{arg} \min}

\def\RR{\mathbb{R}}

\def\NN{\mathbb{N}}
\def\PP{\mathbb{P}}

\def\EE{\mathbb{E}}

\def\XX{\mathbb{X}}

\newcommand{\iid}{\,\,\overset{\text{iid}}{\sim}\,\,}
\begin{document}

\begin{center}

{\bf{\LARGE{On Bayesian Softmax-Gated Mixture-of-Experts Models}}}
  
\vspace*{.2in}
{\large{
\begin{tabular}{cccc}
Nicola Bariletto\footnotemark & Huy Nguyen\footnotemark[1] & Nhat Ho & Alessandro Rinaldo
\end{tabular}
}}

\footnotetext{Equal contribution.}

\vspace*{.2in}

\begin{tabular}{c}
The University of Texas at Austin
\end{tabular}

\today

\vspace*{.2in}

\begin{abstract}
Mixture-of-experts models provide a flexible framework for learning complex probabilistic input-output relationships by combining multiple expert models through an input-dependent gating mechanism. These models have become increasingly prominent in modern machine learning, yet their theoretical properties in the Bayesian framework remain largely unexplored. In this paper, we study Bayesian mixture-of-experts models, focusing on the ubiquitous softmax-based gating mechanism. Specifically, we investigate the asymptotic behavior of the posterior distribution for three fundamental statistical tasks: density estimation, parameter estimation, and model selection. First, we establish posterior contraction rates for density estimation, both in the regimes with a fixed, known number of experts and with a random learnable number of experts. We then analyze parameter estimation and derive convergence guarantees based on tailored Voronoi-type losses, which account for the complex identifiability structure of mixture-of-experts models. Finally, we propose and analyze two complementary strategies for selecting the number of experts. Taken together, these results provide one of the first systematic theoretical analyses of Bayesian mixture-of-experts models with softmax gating, and yield several theory-grounded insights for practical model design.
\end{abstract}

\end{center}

\title{}

\section{Introduction}

Mixture-of-experts (MoE) models, originally introduced by \cite{jacobs1991adaptive}, extend classical mixture models to supervised learning settings such as regression and classification. In these models, predictions are produced by combining several specialized submodels, called \emph{experts}, whose contributions are weighted through a learnable \emph{gating} mechanism that depends on the input variables. As a result, the model adaptively allocates different regions of the input space to different experts, allowing complex conditional relationships to be represented through a collection of simpler components. This architecture is particularly appealing when the conditional distribution of the response varies substantially across the input domain. \cite{jacobs1991adaptive} introduced the original formulation and showed that such models can automatically decompose prediction tasks into simpler subtasks handled by specialized experts.

The MoE framework has since been extended in several directions. \cite{jordan1994hierarchical} introduced hierarchical mixtures-of-experts, in which experts are organized in a tree structure that enables further specialization across the input space. More recently, large-scale machine learning applications have renewed interest in these architectures. In particular, sparse routing mechanisms have been proposed to activate only a subset of experts for each input, enabling models with very large capacity while maintaining manageable computational costs \citep{shazeer2017topk}. Such architectures have been employed in a wide range of modern deep learning applications, including language modeling \citep{jiang2024mixtral,deepseekv3}, computer vision \citep{riquelme2021scaling}, multi-task learning \citep{hazimeh2021dselect}, multi-modal learning \citep{han2024fusemoe}, domain adaptation \citep{li2023sparse,nguyen2025cosine}, continual learning \citep{le2024mixture}, and parameter-efficient fine-tuning \citep{le2025revisiting,diep2025zero}. Despite their empirical success, these models are typically estimated using maximum likelihood methods, often based on the Expectation--Maximization algorithm, which may lead to overfitting and requires the model complexity to be fixed in advance \citep{masoudnia2014mixture}. Uncertainty quantification also remains an important open challenge, particularly in scientific domains where reliable assessment of predictive and parameter uncertainty is essential. Addressing this issue within the maximum likelihood paradigm is often difficult, since likelihood-based procedures typically produce point estimates and do not naturally provide a coherent characterization of uncertainty in complex models such as MoE.

Bayesian formulations of MoE models provide a natural framework for addressing these limitations by enabling uncertainty quantification and principled model selection. Early contributions include the fully Bayesian treatment of \cite{peng1996bayesian}, who used Markov chain Monte Carlo methods to infer both expert parameters and gating parameters. In parallel, \cite{waterhouse1996bayesian} proposed a variational approach based on the minimization of variational free energy, providing a scalable approximation to the posterior distribution. These works highlighted the advantages of combining the modular structure of MoE models with Bayesian probabilistic inference, allowing uncertainty in both predictions and model structure to be quantified.

Subsequent work has explored Bayesian approaches to model selection in MoE architectures. \cite{Ueda2002} proposed a variational Bayesian framework that allows the number of experts to be inferred from data through optimization of an evidence lower bound. Nonparametric Bayesian extensions have also been considered. For example, \cite{rasmussen2002infinite} proposed an infinite MoE model combining Dirichlet processes with Gaussian process experts, allowing the number of mixture components to grow with the size of the data. More broadly, several Bayesian nonparametric constructions have been developed to introduce covariate-dependent mixture weights, including dependent Dirichlet processes \citep{maceachern1999dependent}, hierarchical Dirichlet processes \citep{teh2006hierarchical}, nested Dirichlet processes \citep{rodriguez2008nested}, and logistic stick-breaking processes \citep{ren2011logistic}, among many others; indeed, although many of these models were originally developed to capture dependence across a finite set of disjoint groups, they can also be interpreted as mechanisms for introducing input-dependent mixture weights, establishing a close conceptual connection with MoE models.

Posterior inference in Bayesian MoE models remains challenging due to the complexity of the resulting posterior distributions, particularly in nonconjugate or hierarchical settings. Markov chain Monte Carlo methods provide asymptotically exact inference but may become computationally demanding in high-dimensional or large-scale problems. Variational inference methods provide a computationally attractive alternative by approximating the posterior distribution through optimization. For example, \cite{Bishop2003} developed a variational treatment of MoE models that yields efficient approximate inference while enabling model comparison through the evidence lower bound.

In this paper, we focus on a widely used parametric formulation of MoE models, in which the gating mechanism is given by softmax weights and the experts are Gaussian distributions with input-dependent means. This formulation directly addresses regression tasks (see Section~\ref{sec:classification} for an extension to classification), which are among the most fundamental applications of MoE models and statistical learning in general. At the same time, concentrating on regression keeps the analysis tractable while still yielding insights that may extend to other settings, including unsupervised tasks where MoE models are also commonly employed. Against this backdrop, we study the theoretical properties of Bayesian MoE models with softmax gating. In particular, we investigate the asymptotic behavior of the posterior distribution with respect to density estimation, parameter estimation, and model selection. In fact, although Bayesian MoE models have been widely used in practice, their theoretical properties remain comparatively less understood. Therefore, our goal is to provide one of the first systematic analyses of these models from a Bayesian perspective, establishing conditions under which posterior consistency and related guarantees can be obtained and, as a consequence, deriving practical insights for model design and specification.

\paragraph{Plan of the aricle.} The remainder of the article is organized as follows. In Section~\ref{sec:literature}, we provide an overview of the literature related to our work. Section~\ref{sec:notation}, we introduce the notation and mathematical framework used throughout the paper, together with some preliminary results on the identifiability of MoE models. Section~\ref{sec:density_estimation} studies posterior convergence for density estimation with MoE models, while Section~\ref{sec:param_estimation} focuses on parameter estimation. Section~\ref{sec:model_selection} presents theoretical results and empirical evidence on model selection criteria for MoE models. Although most of the paper concentrates on MoE models for regression tasks, Section~\ref{sec:classification} briefly discusses extensions to classification. Section~\ref{sec:discussion} concludes the article. The proofs of all theoretical results and additional experimental details are collected in Appendix~\ref{app:theorem} and Appendix~\ref{app:elbo_experiments}, respectively.

\section{Related literature}\label{sec:literature}

% {\color{red}Assumptions + Identifiability of experts (Huy) + seq. identifiability for consistency + rates (Nico) (talk about classification)}

%{\color{red}Huy can you add a survey of your many works, Nhat's works, and all other theory work on MoEs?}

Beyond its connections to the literature on MoE models \citep[see][for a recent review in the statistical context]{gormley2019mixture}, our work is also closely related to a well-established body of research in Bayesian asymptotics. In particular, our analysis builds on foundational results on posterior consistency \citep{schwartz1965bayes,barron1999,ghosal1999,walker2004squarerootsum} and posterior contraction rates \citep{ghosal2000convergence,shen2001rates,walker2007rates}. Following these contributions, a substantial literature has studied consistency and contraction properties for mixture models \citep[][among others]{ghosal2001entropies,ghosal2007posterior,lijoi2005consistency}. Our results parallel these analyses, although in a setting that is both more complex, due to the presence of covariates, and somewhat simpler, since we focus on finite mixtures rather than infinite-dimensional ones.

Our analysis of parameter estimation is also related to recent developments on the estimation of mixing measures in mixture models, where Wasserstein distances have played a central role \citep{nguyen2013mixingmeasure,ho2016strong,guha2021posterior}. In the presence of covariates, however, standard Wasserstein-based approaches become less effective. In particular, \cite{ho2022gaussian} used a generalized Wasserstein distance to characterize the convergence behavior of parameter estimation in Gaussian MoE with covariate-free gating and the number of experts being over-specified. They showed that the rates for estimating exactly-specified parameters, which were fitted by one component, shared the same order as those for estimating over-specified parameters, which were fitted by at least two components. However, in practice, the former should be faster than the latter.
This motivated the use of more refined tools based on Voronoi cells, following the recent work of \cite{manole22refined}, where they distinguished the estimation rates of exactly-specified parameters from those of over-specified parameters in the context of finite mixture models. Subsequently, \cite{nguyen2023demystifying} generalized the use of Voronoi-based losses to a more challenging setting of Gaussian MoE with covariate-dependent gating function, showing that exactly-specified parameters enjoy significantly faster estimation rates than over-specified parameters. More recently, \cite{nguyen2026convergence} examined the convergence of parameter estimation in a regression framework where the regression function took either the form of a softmax gating MoE or its hierarchical version. In addition to Wasserstein distance and Voronoi loss function, the Kullback-Leibler (KL) divergence was also utilized to study the maximum likelihood estimator for parameters of the MoE with each expert being a polynomial regression model by
\cite{mendes2011convergence} studied, leading to useful insights on how many experts should be chosen.

Finally, our study of model selection is closely connected to recent work on the behavior of posterior distributions when inferring the number of mixture components. This includes settings where consistency can be established \citep{miller2023consistency,ascolani2023clustering}, as well as well-known examples demonstrating that Dirichlet process mixtures can be inconsistent for this task \citep{miller2014inconsistency}. Similarly, our findings on over-specified MoE models relate to the asymptotic behavior of overfitted mixture models, where the posterior tends to empty redundant components \citep{rousseau2011asymptotic}. As we will clarify, our preliminary empirical evidence supporting model selection via variational inference is also motivated by the recent theoretical advances of \cite{zhang2024bayesian,wang2024estimating}.

\section{Formal setup and identifiability}\label{sec:notation}

To begin our formal analysis, assume that we observe a $\XX\times\RR$-valued sample $(Y_i, X_i)_{i=1}^n$ of input-response pairs, where $\XX := [-1,1]^d$ and $Y \in \RR$. The class of models we consider is that of Softmax-gated Mixture of Gaussian Experts (SMoGE) densities, from which we assume the data are generated:
\begin{equation*}
    (Y_i, X_i) \iid g_{G^\star}, \qquad  g_{G^\star}(Y, X) := f_{G^\star} (Y \mid X)\times p(X),
\end{equation*}
where $p(\cdot)$ is the uniform distribution\footnote{Since our focus is on the conditional distribution of $Y$ given $X$, we fix the marginal distribution of $X$ to a simple choice, although more general distributions could also be considered.} on $\XX$ and
\begin{align*}
    G^\star & := \sum_{j=1}^{K^\star} \exp(\alpha_{0j}^\star) \delta_{(\alpha_{1j}^\star, \beta_{j}^\star, \sigma_j^{\star 2})}, \\
    f_{G^\star}(Y \mid X) & := \sum_{j=1}^{K^\star} \frac{\exp(\alpha_{0j}^\star + X^\top\alpha_{1j}^\star)}{\sum_{\ell=1}^{K^\star}\exp(\alpha_{0\ell}^\star + X^\top\alpha_{1\ell}^\star)} \mathcal{N}(Y \mid E(X,\beta_{j}^\star), \sigma_j^{\star 2}).
\end{align*}
Here $\mathcal{N}(\cdot\mid \mu, \sigma^2)$ denotes the normal distribution with mean $\mu\in\RR$ and variance $\sigma^2>0$, $E:\RR^d\times\RR^p\to\RR$ is the expert function describing the covariate-dependent mean of each component density, parametrized by $\beta\in\RR^p$, and $G^\star$ is a positive mixing measure that completely characterizes the distribution of the data. For any $K\in\NN$, denote by $\mathscr G_{K}$ the set of positive measures with $K$ support points $\omega_j = (\alpha_{1j}, \beta_{j}, \sigma_j^2) \in \Omega\subset \RR^{d+p+1}$ and weights $\exp(\alpha_{0j})$, where $\alpha_{0j}\in A\subset \RR$. Let $\Theta := \Omega\times A$.

The specification of a Bayesian model in this setting starts with the choice of (i) a set of integers $\mathcal K$, and (ii) a prior $\Pi$ supported on $\bigcup_{k\in\mathcal K}\mathscr G_{k}$. The posterior distribution is then obtained via Bayes' rule
\begin{equation}\label{eq:posterior}
    \Pi(\mathrm d G\mid (Y_i,X_i)_{i=1}^n) \propto \Pi(\mathrm dG)\times\prod_{i=1}^nf_{G}(Y_i\mid X_i).
\end{equation}
It is important to note that, while more compact notationally, the convention of placing a prior and deriving a posterior on $\bigcup_{k\in\mathcal K}\mathscr G_{k}$ is equivalent to placing a prior on the Euclidean space $\bigcup_{k\in\mathcal K}(\Omega\times A)^k$ of SMoGE parameters. At times we will switch to this alternative notation.

\subsection{Losses for density estimation}

Since we study the asymptotic behavior of the posterior distribution around the data-generating density, we first introduce several notions of divergence between SMoGE densities. In particular, for any two probability densities $f$ and $f'$ defined on a Borel set $S\subseteq\RR^m$, define the squared Hellinger distance
\begin{equation*}
    d_H^2(f,f') := \int_S\left(\sqrt{f}-\sqrt{f'}\right)^2 \mathrm d\lambda,
\end{equation*}
where $\lambda$ denotes a $\sigma$-finite measure on $S$. For any two SMoGE densities $g_G$ and $g_{G'}$ with $G\in\mathscr G_K$, $G'\in\mathscr G_{K'}$ this definition specializes to
\begin{align*}
    d_H^2(g_G,g_{G'}) 
    & = \int_{\XX \times \RR}\left(\sqrt{g_G(Y, X)}-\sqrt{g_{G'}(Y, X)}\right)^2 \mathrm d(Y,X) \\
    & \equiv \EE_{X\sim p} \big[d_H^2(f_G(\cdot\mid X), f_{G'}(\cdot\mid X))\big].
\end{align*}
Also define the Kullback-Leibler (KL) divergence
\begin{align*}
    \kl(f,f') := \int_S\log(f/f')f\mathrm d\lambda,
\end{align*}
which yields the following important notion of KL neighborhood (restricted to $\mathscr G_{K^\star}$):
\begin{equation*}
    B(G^\star, \varepsilon) := \left\{ G\in \mathscr G_{K^\star} : \kl(g_{G^\star}, g_{G})\leq \varepsilon^2\right\}
\end{equation*}
for any $\varepsilon>0$. Finally, denote by $N(\varepsilon, \mathcal M, d)$ the $\varepsilon$-covering number of a set $\mathcal M$ endowed with the metric $d$, while by $\gtrsim$ and $\lesssim$ we will mean inequalities up to constants.

\subsection{Losses for parameter estimation}

In the previous subsection we considered losses that quantify discrepancies between SMoGE, or more general, density functions. In our theoretical analysis, however, we will also study the asymptotic behavior of the posterior distribution in terms of the SMoGE parameters themselves. This requires losses that are tailored to parameter estimation.

SMoGE models present several challenges that require care in the construction of such losses. First, since $K^\star$ is unknown and the statistical model may contain a different number of experts, we need a way to compare SMoGE parameters belonging to models with different numbers of experts. Second, even when the number of experts coincides, several identifiability issues arise, for instance permutation invariance. These observations, which will be discussed in detail later, prevent the use of simple Euclidean-type metrics between parameter vectors as meaningful measures of discrepancy. To solve this issue, we adopt a refined approach based on \emph{Voronoi cells} \citep{manole22refined,nguyen2023demystifying}.

In particular, given any mixing measure $G\in\mathscr G_K$ with $K\geq K^{\star}$ experts, we partition its components into the following Voronoi cells generated by the components of the ground-truth mixing measure $G^{\star}$:
\begin{align}
    \label{eq:Voronoi_cells}
    \mathcal{C}_j\equiv\mathcal{C}_j(G):=\{i\in\{1,2,\ldots,K\}:\|\omega_i-\omega^{\star}_j\|\leq\|\omega_i-\omega^{\star}_{\ell}\|,\ \forall \ell\neq j\},
\end{align}
where $\omega_i:=(\alpha_{1i},\beta_i,\sigma^2_i)$ and $\omega^{\star}_j:=(\alpha^{\star}_{1j},\beta^{\star}_j,\sigma^{\star 2}_j)$ for any $j\in\{1,2,\ldots,K^\star\}$. Based on this construction, we define the following Voronoi losses, which quantify the discrepancy between the parameters of $G$ and $G^\star$ and will play a central role in our posterior asymptotic analysis:
\begin{align}\label{eq:voronoi_loss_1}
    \mathcal L_1(G,G^\star)& :=\sum_{j=1}^{K^\star}\Bigg|\sum_{i\in\mathcal{C}_j}\exp(\alpha_{0i})-\exp(\alpha^\star_{0j})\Bigg|\nonumber\\
    & +\sum_{j=1}^{K^\star}\sum_{i\in\mathcal{C}_j}\exp(\alpha_{0i})\Big(\|\alpha_{1i}-\alpha^\star_{1j}\|+\|\beta_{i}-\beta^\star_{j}\|+|\sigma^2_{i}-\sigma^{\star2}_{j}|\Big)
\end{align}
and
\begin{align}
    \label{eq:voronoi_loss_2}
    &\mathcal{L}_{2}(G,G^\star):=\sum_{j=1}^{K^\star}\Big|\sum_{i\in\mathcal{C}_j}\exp(\alpha_{0i})-\exp(\alpha^\star_{0j})\Big|\nonumber\\
    &+\sum_{j\in[K^\star]:|\mathcal{C}_j|=1}\sum_{i\in\mathcal{C}_j}\exp(\alpha_{0i})(\|\alpha_{1i}-\alpha^\star_{1j}\|+\|\beta_{i}-\beta^\star_{j}\|+|\sigma^2_{i}-\sigma^{\star2}_{j}|)\nonumber\\
    &+\sum_{j\in[K^\star]:|\mathcal{C}_j|>1}\sum_{i\in\mathcal{C}_j}\exp(\alpha_{0i})(\|\alpha_{1i}-\alpha^\star_{1j}\|^2+\|\beta_{i}-\beta^\star_{j}\|^2+|\sigma^2_{i}-\sigma^{\star2}_{j}|^2),
\end{align}
where $[m]:=\{1,\ldots,m\}$ for any $m\in\NN$.

Although the geometry induced by these losses will be discussed in detail in Section~\ref{sec:param_estimation}, it is clear that both $\mathcal L_1$ and $\mathcal L_2$ measure discrepancies between the component parameters of $G$ and the closest component parameters of $G^\star$. In the context of statistical estimation, this can be interpreted as comparing the parameters of the statistical model implied by $G$ with the parameters of the component of the data-generating model implied by $G^\star$ that they estimate. As a preview of our analysis, $\mathcal L_1$ will be useful both for characterizing the support of the prior and for analyzing posterior convergence at the level of the SMoGE parameters when the model contains exactly the same number of components as $G^\star$. In contrast, $\mathcal L_2$ will be crucial for characterizing parameter convergence when the statistical model contains potentially more experts than the ground-truth model.

\subsection{Basic assumptions and model identifiability}

As a final step before proceeding with the analysis, we state several assumptions that will be used throughout the paper and discuss some of their immediate consequences.

\begin{assumption}\label{ass:prior_compact_support}
    The parameter space $\Theta:=\Omega\times A$ is compact, and the variance space is bounded away from $0$.
\end{assumption}

\begin{assumption}\label{ass:lipschitz_expert}
    There exists $L>0$ such that, for all $X\in\XX$ and $\beta,\beta'\in\RR^p$ in the parameter space,
    \[
    |E(X,\beta)-E(X,\beta')|\leq L\|\beta-\beta'\|.
    \]
\end{assumption}

\begin{assumption}\label{ass:exp_identifiability}
    The expert function $E:[0,1]^d\times\mathbb{R}^{p}\to\mathbb{R}$ is identifiable. That is, for each $K\in\mathbb{N}$ the following holds. If there exist distinct parameter collections $(\beta_1,\beta_2,\ldots,\beta_K)$ and $(\beta'_1,\beta'_2,\ldots,\beta'_K)$ such that, for almost every $x\in\mathcal{X}$, there exists a permutation function $\sigma_x:\{1,2,\ldots,K\}\to\{1,2,\ldots,K\}$ satisfying $E(X,\beta_{\sigma_x(i)})=E(X,\beta'_i)$ for all $1\le i\le K$, then $\{\beta_1,\beta_2,\ldots,\beta_K\}\equiv\{\beta'_1,\beta'_2,\ldots,\beta'_K\}$.
\end{assumption}

\begin{assumption}\label{ass:gating_biases}
    The discrepancies between gating parameters $\alpha^{\star}_{1i}-\alpha^{\star}_{1j}$, for $1\leq i<j\leq K^\star$, are pairwise distinct. In addition, expert parameters $(\beta^{\star}_{j},\sigma^{\star2}_{j})$, for $1\leq j\leq K^\star$, are also pairwise different. The same holds for the parameter configuration associated to $\Pi$-almost every mixing measure $G$.\footnote{This holds, for instance, as long as the prior $\Pi$ is of product form across experts and diffuse.}
\end{assumption}
The first two assumptions are crucial to ensure expert parameter estimation consistency. As for the third assumption, it can be verified that, for example, if $\varphi:\mathbb{R}\to\mathbb{R}$ is any injective function, then an expert function of the form $E(X,\beta)=\varphi(\beta^{\top}X)$, with $\beta\in\mathbb{R}^d$, is identifiable.  This assumption and the fourth one have important consequences for the identifiability of mixing measures, as shown in the following result.

\begin{proposition} 
    \label{prop:identifiability_up_to_translation}
    %Suppose that the expert function $E$ is identifiable. 
    Assume $G$ is a mixing measure
    \[
    G=\sum_{j=1}^{K}\exp(\alpha_{0j})\delta_{(\alpha_{1j},\beta_{j},\sigma^2_{j})}
    % \quad\text{and}\quad
    % G'=\sum_{j=1}^{K'}\exp(\alpha'_{0j})\delta_{(\alpha'_{1j},\beta'_{j},\sigma^{'2}_j)}
    % 
    \]
    such that $g_{G}(y,x)=g_{G^\star}(y,x)$ for almost every pair $(y,x)\in\RR\times \mathbb X$. Assume also that Assumption~\ref{ass:exp_identifiability} holds, and Assumption~\ref{ass:gating_biases} also holds for both $G^\star$ and $G$. Then $K=K^\star$ and
    \begin{align*}
    G\equiv G^\star_{t_0,t_1}:=\sum_{j=1}^{K^\star}\exp(\alpha^\star_{0j}+t_0)\delta_{(\alpha^\star_{1j}+t_1,\beta^\star_j,\sigma^{\star 2}_{j})},
    \end{align*}
    for some $t_0\in\mathbb{R}$ and $t_1\in\mathbb{R}^d$.
\end{proposition}
Proposition~\ref{prop:identifiability_up_to_translation}, whose proof is in Appendix~\ref{appendix:identifiability}, shows that SMoGE densities are identifiable up to translation of the gating parameters. This dependence arises from the softmax gating structure, as softmax weights remain unchanged if the gating parameters $\alpha^{\star}_{0j}$ and $\alpha^{\star}_{1j}$ are translated to $\alpha^{\star}_{0j}+t_0$ and $\alpha^{\star}_{1j}+t_1$, respectively, for some $t_0\in\mathbb{R}$ and $t_1\in\mathbb{R}^d$. Formally,
\begin{align}   
    \label{eq:softmax_translation}
    \frac{\exp(\alpha_{0j}^\star+t_0 + X^\top(\alpha_{1j}^\star+t_1))}{\sum_{\ell=1}^{K^\star}\exp(\alpha_{0\ell}^\star+t_0 + X^\top(\alpha_{1\ell}^\star+t_1))}
    =
    \frac{\exp(\alpha_{0j}^\star + X^\top\alpha_{1j}^\star)}{\sum_{\ell=1}^{K^\star}\exp(\alpha_{0\ell}^\star + X^\top\alpha_{1\ell}^\star)},
\end{align}
for all $j=1,2,\ldots,K^\star$. 

In light of this, if one fixes a pair of gating parameters to zero, for example $\alpha^{\star}_{0K^\star}=0$ and $\alpha^{\star}_{1K^\star}=0$, then the numerator of the last softmax weight becomes equal to one, that is, $\exp(\alpha^{\star}_{0K^\star}+X^{\top}\alpha^{\star}_{1K^\star})=1$. In this case, the softmax weights are no longer invariant to translations of the gating parameters and, as a consequence, the SMoGE density becomes identifiable, which is crucial for parameter estimation.
In light of this, we formulate the following additional assumption.

\begin{assumption}
    \label{ass:zero_gating_param}
    Assume that the last pair of ground-truth gating parameters are zero, that is, $\alpha^{\star}_{0K^\star}=0$ and $\alpha^{\star}_{1K^\star}=0$.
\end{assumption}

\section{Density estimation}\label{sec:density_estimation}

We are now ready to tackle one of the core objectives of the paper, which is to study the asymptotic properties of the posterior distribution in the SMoGE model we have described. Recall that we assume that the observed data are generated iid according to a fixed, yet unknown SMoGE model $g_{G^\star}$. A natural question is therefore whether, and at what rate, the posterior concentrates around $G^\star$, and the losses introduced in Section~\ref{sec:notation} allow us to formalize the notion of being ``close'' to $G^\star$.

In this section, we begin by asking whether the posterior distribution in \Eqref{eq:posterior} asymptotically concentrates within arbitrary Hellinger neighborhoods of the data-generating density $g_{G^\star}$. Two classical approaches address this question. The first studies \emph{posterior consistency} in the Hellinger metric, that is, whether the posterior concentrates its mass in every neighborhood of the form $\{G : d_{H}(g_G,g_{G^\star})<\varepsilon\}$ for all fixed $\varepsilon>0$ \citep{barron1999,ghosal1999,walker2001,walker2004squarerootsum,ghosal2017fundamentals,bariletto2025posterior}. The second studies contraction within shrinking neighborhoods of the form $\{G : d_{H}(g_G,g_{G^\star})<\varepsilon_n\}$ for a decreasing \emph{contraction rate} $(\varepsilon_n)_{n\in\NN}$ \citep{ghosal2000convergence,shen2001rates,walker2007rates,ghosal2007noniid,ghosal2017fundamentals}. 

Our aim is to investigate these properties under two specifications of interest for the integer set $\mathcal K$ that determines the support of the prior on the number of experts: the \emph{exactly-specified} case where $\mathcal K = \{K^\star\}$, and the \emph{over-specified} case where $\mathcal K = \{1,\ldots,K\}$ for some large enough $K\geq K^\star$ (featuring a random-a-priori number of experts).

\subsection{Exactly-specified number of experts}

We first consider the exactly-specified case, that is, when the set of SMoGE models supported by the prior comprises SMoGE densities with exactly $K^\star$ components. While precise knowledge of $K^\star$ is not necessarily a tenable assumption in practice, first analyzing this setting provides a useful benchmark to compare our subsequent results of more realistic scenarios, especially in terms of parameter estimation. %Before proceeding with the statement of our theoretical results, we make the following remark: for the purpose of density estimation, the  case $K=K^\star$ and the  case $K>K^\star$ behave similarly, and we therefore do not distinguish between them in the statements of our results. This distinction will instead play an important role in the study of parameter estimation in Section~\ref{sec:param_estimation}.
To this end, the following theorem deals with posterior consistency in arbitrary Hellinger neighborhoods of the true density.

\begin{theorem}\label{thm:consistency}
    Assume that $\Pi(\mathscr G_{K^\star})=1$. If $\Pi\left(\left\{G\in\mathscr G_{K^\star} : \kl(g_{G^\star}, g_G)<  \delta\right\}\right)>0$ for all $\delta>0$, then for all $\varepsilon>0$
    \begin{equation}\label{eq:consistency}
    \lim_{n\to\infty} \Pi\left(\left\{G\in\mathscr G_{K^\star} : \,d_H(g_G, g_{G^\star})>\varepsilon\right\} \mid (X_i, Y_i)_{i=1}^n\right)  = 0 \quad \text{a.s.-}g_{G^\star}^\infty.
    \end{equation}
    If in addition Assumptions~\ref{ass:prior_compact_support} and \ref{ass:lipschitz_expert} hold, then \Eqref{eq:consistency} holds for all $\varepsilon>0$ provided that
    \begin{equation}\label{eq:voronoi_schwartz_support}
        \forall t>0,\quad\Pi\left(\left\{G\in\mathscr G_{K^\star} : \mathcal L_1(G^\star, G)<  t\right\}\right)>0.
    \end{equation}
\end{theorem}

Theorem~\ref{thm:consistency} shows that a minimal support condition on the prior, expressed either in terms of the KL divergence \citep{schwartz1965bayes} or in terms of the Voronoi loss $\mathcal L_1$, is sufficient to ensure that the posterior concentrates within arbitrary Hellinger neighborhoods of the true density. Inspection of the proof in Appendix~\ref{proof:consistency} reveals that the argument combines Schwartz's KL support theorem with the observation that, despite their flexibility, SMoGE models satisfy the simple sequential identifiability condition introduced by \cite{bariletto2025posterior}, which ensures Hellinger consistency.

We now turn to contraction rates. Before stating the main result, we introduce an assumption describing a form of prior concentration around the true mixing measure.

\begin{assumption}\label{ass:prior_condition}
    There exists $c>0$ such that the following holds for all sufficiently small $t>0$:
    \begin{equation}
        \Pi\left(\left\{G\in\mathscr G_{K^\star} : \mathcal L_1(G^\star, G)\leq  ct\right\}\right)\geq t.
    \end{equation}
\end{assumption}

Assumption~\ref{ass:prior_condition} can be interpreted as a quantitative version of the support condition stated in \Eqref{eq:voronoi_schwartz_support}. While that condition only requires every $\mathcal L_1$ neighborhood of $G^\star$ to be charged positive prior mass, Assumption~\ref{ass:prior_condition} specifies how this mass must scale with the neighborhood radius. In particular, it requires that $\mathcal L_1(G^\star,G)$ be approximately uniformly distributed near zero under the prior, ensuring that the latter does not place too little mass near $G^\star$.

With the aid of Assumption~\ref{ass:prior_condition}, we formulate the next theorem dealing with posterior contraction rates in Hellinger metric.

\begin{theorem}\label{thm:our_abstract_rate_SMoGE}
    Consider a SMoGE model with $\Pi(\mathscr G_{K^\star})=1$ and satisfying Assumptions \ref{ass:prior_compact_support}, \ref{ass:lipschitz_expert}, and \ref{ass:prior_condition}. Then for any sequence $(M_n)_{n\in\NN}$ such that $\lim_{n\to\infty}M_n=\infty$,
    \begin{equation*}
    \lim_{n\to\infty} \Pi\left(\left\{G\in\mathscr G_{K^\star} : \,d_H(g_G, g_{G^\star}) \geq M_n\sqrt{\log(n)/n}\right\} \mid (X_i, Y_i)_{i=1}^n\right)  = 0
    \end{equation*}
    in $g_{G^\star}^n$-probability.
\end{theorem}

Theorem~\ref{thm:our_abstract_rate_SMoGE} shows that, for density estimation, the posterior achieves a parametric contraction rate up to a logarithmic factor. The proof, reported in Appendix~\ref{proof:our_abstract_rate_SMoGE}, relies on the general theory of posterior contraction developed in \cite{ghosal1999} \citep[see also][]{ghosal2017fundamentals}. The argument proceeds by constructing an explicit cover of the model class and linking the KL neighborhood $B(G^\star,\varepsilon)$ with a corresponding neighborhood defined through the Voronoi loss $\mathcal L_1$.

\subsection{Over-specified random-a-priori number of experts}\label{sub:random_k_density}

An appealing feature of Bayesian methods is that, in principle, they allow inference on any quantity for which a prior can be specified. In SMoGE models, a key quantity influencing both density and parameter estimation is the number of experts $\kappa$. In the previous subsection we assumed that $\kappa=K^\star$ was fixed a priori. In practice, however, one rarely has exact knowledge of the data-generating number of experts, and may therefore wish to consider a range of possible values for $\kappa\in\mathcal K$, so to achieve greater model flexibility or to hedge against model misspecification. The Bayesian framework is uniquely suited to accommodate this situation, as we show next.

Let $\Pi_k$ denote the prior on the mixing measure $G$ considered in Theorem~\ref{thm:our_abstract_rate_SMoGE}, when the number of experts $\kappa$ equals $k\in\NN$ (in Theorem~\ref{thm:our_abstract_rate_SMoGE}, $k=K^\star$). To allow for multiple values of $\kappa$, we enlarge the support of the prior to $\mathscr O_K:=\bigcup_{j=1}^K \mathscr G_j$ and place a prior $\pi_\kappa$ directly on $\kappa$:
\begin{align}\label{eq:prior_on_k}
    G\mid \kappa = k &\sim \Pi_k,\nonumber\\
    \kappa &\sim \pi_\kappa,
\end{align}
with $\pi_\kappa(\{1,\ldots,K\})=1$ for some $K\in\NN$, chosen large enough so that $K\geq K^\star$. That is, instead of requiring exact knowledge of the unknown parameter $K^\star$ as in the exactly-specified case, we now only assume to be able to choose an arbitrarily large (but finite) upper-bound $K$. The resulting prior on the mixing measure is $\Pi = \sum_{k=1}^K\pi_\kappa(\{k\})\Pi_k$, which induces the posterior $\Pi(\cdot\mid (X_i,Y_i)_{i=1}^n)$. This formulation also allows inference on $\kappa$ itself, although a thorough asymptotic analysis of this aspect is deferred to Section~\ref{sec:model_selection}.

We now deal with density estimation, and specifically posterior contraction rates, in this random-$\kappa$ setting.\footnote{Although we omit a formal statement, a result analogous to Theorem~\ref{thm:consistency} can also be obtained in this setting, ensuring posterior consistency under a simple KL or $\mathcal L_1$ support condition for $\Pi_{K^\star}$.}

\begin{theorem}\label{thm:density_estim_finite_prior}
    Consider a SMoGE model with a prior structure as in Equation~(\ref{eq:prior_on_k}), such that Assumptions~\ref{ass:prior_compact_support}, \ref{ass:lipschitz_expert} and \ref{ass:prior_condition} hold. Then for any sequence $(M_n)_{n\in\NN}$ such that $\lim_{n\to\infty}M_n=\infty$,
    \begin{equation*}
    \lim_{n\to\infty} \Pi\left(\left\{G\in\mathscr O_{K} : \,d_H(g_G, g_{G^\star}) \geq M_n\sqrt{\log(n)/n}\right\} \mid (X_i, Y_i)_{i=1}^n\right)  = 0
    \end{equation*}
    in $g_{G^\star}^n$-probability. 
\end{theorem}

According to Theorem~\ref{thm:density_estim_finite_prior}, allowing the number of experts to be random a priori does not affect the Hellinger contraction rate of the posterior around the true density, provided that the prior on the number of experts has finite support, and a parametric rate (up to a logarithmic factor) is recovered in this setting as well.

\section{Parameter and expert estimation}\label{sec:param_estimation}
%Exactly specified + Over-specified + expert estimation (Huy)

% In this section, we examine the convergence properties of parameter estimation within the SMoGE model. Recall that during the training of MoE architectures in deep learning, each expert is expected to specialize in distinct subproblems or latent regions of the data distribution. For example, in large-scale language models based on the MoE paradigm \citep{jiang2024mixtral,liu2024deepseek,gemini25pushing}, certain experts will capture syntax, another might focus on semantics, while the rest concentrate on context.
Having tackled density estimation, this section is devoted to the analysis of posterior convergence in terms of parameter estimation. To further motivate the need for such a refined analysis, we recall that, during the training of MoE architectures in deep learning, individual experts are meant to specialize in different subproblems or latent regions of the input distribution. For instance, in large language models built upon the MoE framework \citep{jiang2024mixtral,liu2024deepseek,gemini25pushing}, some experts may learn syntactic patterns, others may focus on semantic representations, while the remaining ones capture broader contextual information. In our Bayesian setting, expert specialization can be formalized as posterior concentration on models whose distinct experts are close to unknown heterogeneous regimes in the underlying data-generating process. From a statistical perspective, this makes the rate of specialization a critical subject of study \citep{dai2024deepseekmoe,oldfield2024specialize,nguyen2024hmoe}. To operationalize this problem, an effective strategy is to focus on parameter estimation convergence, which in turn indicates how fast the parametrized experts specialize. In terms of sample complexity, faster convergence rates indicate that experts need less data to achieve a given degree of estimation accuracy or, in MoE jargon, specialization. Consequently, analyzing the asymptotics of parameter estimation not only provides theoretical insights into the dynamics of expert learning, but also informs strategies to improve the design of sample-efficient MoE models.

There are two notable existing works on parameter estimation in Gaussian MoE models. First, \cite{ho2022gaussian} study the effects of mean and variance experts on parameter estimation rates in Gaussian MoE with covariate-free gating. In particular, they show that if mean and variance experts meet some specific algebraic independence conditions, then the ground-truth parameters can be estimated at rates of order $(\log(n)/n)^{1/4}$. Otherwise, the rates turn out to be substantially slower due to the interaction between mean and variance parameters expressed in terms of some underlying partial differential equations (PDEs). Second, \cite{nguyen2023demystifying} consider a more challenging SMoGE model (in a maximum likelihood estimation framework) but limited to linear experts. They capture a PDE-type interaction between gating and expert parameters and demonstrate that parameter convergence rates depend on the solvability of a complex system of polynomial equations derived from that interaction. In this work, we further investigate the SMoGE model by characterizing expert structures that lead to improved parameter estimation rates compared to those in \cite{nguyen2023demystifying}, also connecting them to our overarching Bayesian approach. 

Before delving into the main findings of this section, it is worth noting that the posterior convergence on the density space established in Theorems~\ref{thm:our_abstract_rate_SMoGE} and \ref{thm:density_estim_finite_prior} is a key component in characterizing the convergence behavior of parameter estimation. More specifically, given the results of Theorems~\ref{thm:our_abstract_rate_SMoGE} and \ref{thm:density_estim_finite_prior}, the convergence analysis for parameter estimation boils down to constructing a discrepancy measure between the model parameters and their ground-truth counterparts, say $\mathcal{L}(G,G^{\star})$, which is bounded above by the Hellinger distance between model and ground-ground truth densities. That is, more concretely:
\begin{align}
    \label{eq:hellinger_lower_bound}
    d_H(g_{G},g_{G^{\star}})\gtrsim\mathcal{L}(G,G^{\star}).
\end{align}
A bound of the above type, combined with Theorems~\ref{thm:our_abstract_rate_SMoGE} and \ref{thm:density_estim_finite_prior}, immediately translates into convergence rates for parameter estimation.

In the sequel, we will conduct our convergence analysis of parameter estimation in two different settings, namely the exactly-specified and over-specified settings considered in the context of density estimation in Section~\ref{sec:density_estimation}.

\subsection{Exactly-specified setting}

To begin with, we investigate the exactly-specified setting in which the number of ground-truth experts $K^{\star}$ is known. In doing so, we face some technical challenges, which we briefly discuss. In our proofs, a key step in characterizing parameter estimation rates is to decompose the density discrepancy $g_{G}-g_{G^\star}$ into a combination of linearly independent terms. This can be achieved by applying a Taylor expansion to the function $F(Y|X,\alpha,\beta,\sigma^2):=\exp(X^{\top}\alpha)\mathcal{N}(Y|E(X,\beta),\sigma^2)$. In the process, it is important to ensure that the function $F$ and its partial derivatives with respect to its parameters are linearly independent. If that is the case, when the Hellinger distance goes to zero, the coefficients of these terms in the decomposition, which are represented as parameter discrepancies, will also converge to zero. For that purpose, we need to introduce a non-trivial algebraic
independence condition on the expert function $E(X,\beta)$, as in Definition~\ref{def:weak_identifiability} below, which we refer to as \emph{first-order strong identifiability} under the exactly-specified setting. 

% \begin{align}
% \label{eq:Voronoi_loss}
% &\mathcal{L}_{1}(G,G^{\star}):=\sum_{j=1}^{K^\star}\Big|\sum_{i\in\mathcal{C}_j}\exp(\alpha_{0i})-\exp(\alpha^{\star}_{0j})\Big|\nonumber\\
% &\hspace{2cm}+\sum_{j=1}^{K^\star}\sum_{i\in\mathcal{C}_j}\exp(\beta_{0i})\Big[\|\alpha_{1i}-\alpha^{\star}_{1j}\|+\| \beta_{i}-\beta^{\star}_{j}\|+|\sigma^2_{i}-\sigma^{\star2}_{j}|\Big],
% \end{align}
% where we denote $\Delta\alpha_{1ij}:=\alpha_{1i}-\alpha^{\star}_{1j}$, $\Delta \beta_{ij}:=\beta_i-\beta^{\star}_j$, and $\Delta\sigma^2_{ij}:=\sigma^2_{i}-\sigma^{\star 2}_{j}$. 

\begin{definition}
    \label{def:weak_identifiability}
    An expert function $X\mapsto E(X,\beta)$ is said to be first-order strongly identifiable if it is differentiable with respect to $\beta$, and the set $\Big\{\frac{\partial E(X,\beta)}{\partial\beta^{(u)}} : u\in[d]\Big\}$ is linearly independent with respect to $X$,
    % :
    % \begin{align*}
    %     \Bigg\{\frac{\partial E(X,\beta)}{\partial\beta^{(u)}} : u\in[d]\Bigg\},
    % \end{align*}
    where we denote $\beta=(\beta^{(u)})_{u=1}^{d}$. That is, for any $\beta$, if there exist $t_u\in\mathbb{R}$, for $1\leq u\leq d$, such that 
    \begin{align*}
        \sum_{u=1}^{d}t^{(u)}\frac{\partial E(X,\beta)}{\partial\beta^{(u)}}=0,
    \end{align*}
    for almost every $X$, then we have $t^{(u)}=0$ for all $1\leq u\leq d$.
\end{definition}

\begin{example}
    It can be verified that 
%the function $E(X;\beta)=\mathrm{GELU}(x^{\top}\beta)$ is first-order strongly identifiable. This claim still holds true when replacing $\mathrm{GELU}$ with other non-linear activation functions such as $\mathrm{sigmoid}$, and $\tanh$. On the other hand, 
a linear expert $E(X,\beta_1,\beta_0)=X^{\top}\beta_1+\beta_0$ satisfies the first-order strong identifiability condition. On the other hand, a constant expert $E(X,\beta)=c$ is not first-order strongly identifiable.
\end{example} 
%By contrast, linear experts $E(X;\beta_1,\beta_0)=x^{\top}\beta_1+\beta_0$ are not strongly identifiable since $\frac{\partial^2g}{\partial\theta_0^2}(x;\boldsymbol{\theta}_1,\theta_0)=0$.

% Let us interpret the first-order strong identifiability condition in both technical and intuitive ways as follows. Technically, it helps eliminate any linear independence among terms in the aforementioned Taylor expansion of the density discrepancy. Intuitively, the above condition in Definition~\ref{def:weak_identifiability}
% helps prevent potential adverse interactions among parameters, namely, where gating parameters $\alpha$ interact with expert parameters $\beta$ as in \Eqref{}, negatively affecting the convergence rate of parameter estimation. 
%By contrast, when the expert function is of linear form, which violates the first-order strong identifiability condition, \cite{nguyen2023demystifying} showed that parameter estimation rates became significantly slow since they hinged on the solvability of a system of polynomial equations.

Next, it should be noted that the mixing measure $G^\star= \sum_{j=1}^{K^\star} \exp(\alpha_{0j}^\star) \delta_{(\alpha_{1j}^\star, \beta_{j}^\star, \sigma_j^{\star 2})}$ is only identifiable up to a permutation of parameters. In particular, if two mixing measures are equivalent, $G= G'$, we can only deduce that $(\alpha_{0j},\alpha_{1j},\beta_j,\sigma^2_j)=(\alpha^{\star}_{0\phi(j)},\alpha^{\star}_{1\phi(j)},\beta^{\star}_{\phi(j)},\sigma^{\star 2}_{\phi(j)})$ for some permutation $\phi$ of the set $\{1,2,\ldots,K^\star\}$. To resolve this permutation invariance, we employ a Voronoi cell-based approach, which was briefly introduced in Section~\ref{sec:notation} and which we recall and expand on here. Given an arbitrary mixing measure $G$ with $K^{\star}$ components, we partition its components into the following Voronoi cells, which are generated by the components of ground-truth mixing measure $G^{\star}$:
\begin{align}
    \mathcal{C}_j\equiv\mathcal{C}_j(G):=\{i\in\{1,2,\ldots,K^\star\}:\|\omega_i-\omega^{\star}_j\|\leq\|\omega_i-\omega^{\star}_{\ell}\|,\ \forall \ell\neq j\},
\end{align}
where $\omega_i:=(\alpha_{1i},\beta_i,\sigma^2_i)$ and $\omega^{\star}_j:=(\alpha^{\star}_{1j},\beta^{\star}_j,\sigma^{\star 2}_j)$ for any $j\in\{1,2,\ldots,K^\star\}$. Notably, when the components of $G$ are sufficiently close to those of $G^{\star}$, then each of these Voronoi cells will contain one element. As a result, we can deal with the aforementioned permutation-invariance of mixing measures with respect to parameters. The corresponding Voronoi-based loss capturing parameter discrepancies is given by
\begin{align*}
    \mathcal L_1(G,G^\star)& :=\sum_{j=1}^{K^\star}\Bigg|\sum_{i\in\mathcal{C}_j}\exp(\alpha_{0i})-\exp(\alpha^\star_{0j})\Bigg|\nonumber\\
    & +\sum_{j=1}^{K^\star}\sum_{i\in\mathcal{C}_j}\exp(\alpha_{0i})\Big(\|\alpha_{1i}-\alpha^\star_{1j}\|+\|\beta_{i}-\beta^\star_{j}\|+|\sigma^2_{i}-\sigma^{\star2}_{j}|\Big).
\end{align*}
Above, if a Voronoi cell $\mathcal{C}_j$ is empty, then we let the corresponding summation term be zero. Additionally, it can be checked that $\mathcal{L}_{1}(G,G^{\star})=0$ if and only if $G\equiv G^{\star}$. Thus, when $\mathcal{L}_{1}(G,G^{\star})$ is sufficiently small, the differences $\alpha_{1i}-\alpha^\star_{1j}$, $\beta_{i}-\beta^\star_{j}$, and $\sigma^2_{i}-\sigma^{\star2}_{j}$ are also small. This property indicates that $\mathcal{L}_{1}(G,G^{\star})$ is an appropriate loss function for measuring parameter discrepancies. However, since the loss $\mathcal{L}_{1}(G,G^{\star})$ is not symmetric, it is not a proper metric. The Voronoi loss function $\mathcal{L}_{1}(G,G^{\star})$ also enjoys computational benefits, as it features a low computational complexity, namely $\mathcal{O}(K^{\star
2})$. We finally note that the key benefits of Voronoi-based loss functions are more evident in the over-specified setting where the rates for estimating over-specified parameters (parameters of $G^\star$ fitted by at least two parameters) are slower than those for estimating exactly-specified parameters (parameters of $G^\star$ fitted by one parameter). In particular, \citet{nguyen2023demystifying} have shown that Voronoi-based loss functions are capable of distinguishing these rates, while the generalized Wasserstein loss used in \citep{ho2022gaussian} cannot capture the rate difference.

Given the Voronoi-based loss in~\Eqref{eq:voronoi_loss_1}, we are ready to present our results on the posterior convergence rates for parameter estimation in exactly-fitted SmoGE models.

\begin{theorem}
    \label{theorem:exact_parameter_rate}
    Suppose that the expert function $X\mapsto E(X,\beta)$ is first-order strongly identifiable. Under Assumptions~\ref{ass:prior_compact_support}, \ref{ass:lipschitz_expert}, \ref{ass:exp_identifiability}, \ref{ass:gating_biases}, \ref{ass:zero_gating_param}, and \ref{ass:prior_condition}, the lower bound $d_H(g_{G},g_{G^\star})\gtrsim\mathcal{L}_{1}(G,G^\star)$ holds for any mixing measure $G\in\mathscr G_{K^\star}$. Then, the following holds for any sequence $(M_n)_{n\in\NN}$ such that $\lim_{n\to\infty}M_n=\infty$:
    \begin{equation}
    \label{eq:parameter_estimation_rate}
    \lim_{n\to\infty} \Pi\left(\left\{G\in\mathscr G_{K^\star} : \, \mathcal{L}_{1}(G,G^\star) \geq M_n\sqrt{\log n/n}\right\} \mid (X_i, Y_i)_{i=1}^n\right)  = 0.
    \end{equation}
    in $g_{G^\star}^n$-probability. 
    %{\color{red} Add part on expert estimation}
\end{theorem}
The proof of Theorem~\ref{theorem:exact_parameter_rate} is in Appendix~\ref{appendix:exact_parameter_rate}. A few remarks regarding the results of Theorem~\ref{theorem:exact_parameter_rate} are in order.

\begin{enumerate}
    \item[(i)] \emph{Parameter estimation rates:} The result in \Eqref{eq:parameter_estimation_rate}, together with the structure of the Voronoi-based loss $\mathcal{L}_1(G,G^{\star})$ in~\Eqref{eq:voronoi_loss_1}, indicates that the posterior rates for estimating parameters $\exp(\alpha^{\star}_{0j})$, $\alpha^{\star}_{1j}$, $\beta^{\star}_{j}$, and $\sigma^{\star2}_{j}$ are of the same $\sqrt{\log(n)/n}$ order, which is parametric on the sample size $n$.

    \item[(ii)] \emph{Expert estimation rates:} Recall that the expert function $E(X,\beta)$ is Lipschitz continuous with respect to its parameter $\beta$. Thus, we have
\begin{align}
    \label{eq:expert_inequality}
    |E(X,\beta_{i})-E(X,\beta^{\star}_{j})|\leq L\|\beta_{i}-\beta^{\star}_{j}\|,
\end{align}
for all $i\in\mathcal{C}_{j}$ and $j\in[K^{\star}]$. 
%It should be noted that the above Lipschitz continuity condition is satisfied by several expert choices in practice, including two-layer feed-forward networks used in the Transformer architecture in large language models \citep{vaswani2017attention}. 
Then, the estimation rates for parameters $\beta^{\star}_{j}$ in part (i) indicate that the estimators of experts $E(X,\beta^*_{j})$ also enjoy the posterior contraction rates of the parametric order $\sqrt{\log(n)/n}$, for $j\in\{1,2,\ldots,K^\star\}$. Therefore, our results reveal that one needs a polynomial number $\mathcal{O}(\epsilon^{-2})$ of data points to estimate an expert $E(X,\beta^{\star}_j)$ with a given error tolerance $\epsilon>0$.
\end{enumerate}

\subsection{Over-specified setting}
We now turn to the over-specified setting, in which the number of ground-truth experts is unknown and allowed to range up to an over-specified number $K>K^\star$.

Similar to the exactly-specified setting, we also need to introduce a strong identifiability condition for the expert function in this setting to guarantee linear independence among terms in the Taylor expansion of $F(Y|X,\alpha,\beta,\sigma^2)=\exp(X^{\top}\alpha)\mathcal{N}(Y|E(X,\beta),\sigma^2)$ when decomposing the density discrepancy $g_{G}-g_{G^\star}$. However, since the number of ground-truth experts is unknown, the first-order strong identifiability condition in Definition~\ref{def:weak_identifiability} is not sufficient. Instead, it is necessary to control all derivatives of the expert function up to the second order. For this purpose, let us formally present the following \emph{second-order strong identifiability} condition.
\begin{definition}
    \label{def:strong_identifiability}
    An expert function $X\mapsto E(X,\beta)$ is said to be second-order strongly identifiable if it is twice differentiable with respect to $\beta$, and the following set is linearly independent with respect to $X$:
    \begin{align*}
        \Bigg\{\frac{\partial E(X,\beta)}{\partial\beta^{(u)}}, \ \frac{\partial^2E(X,\beta)}{\partial\beta^{(u)}\partial\beta^{(v)}}, \ X^{(u)}\frac{\partial E(X,\beta)}{\partial\beta^{(v)}} :  u,v\in[d]\Bigg\}.
    \end{align*}
    That is, for any $\beta$, if there exist $t_1^{(u)},t_2^{(uv)},t_3^{(uv)}\in\mathbb{R}$ such that
    \begin{align*}
        t_1^{(u)}\frac{\partial E(X,\beta)}{\partial\beta^{(u)}}+t_2^{(uv)}\frac{\partial^2E(X,\beta)}{\partial\beta^{(u)}\partial\beta^{(v)}}+t_3^{(uv)}X^{(u)}\frac{\partial E(X,\beta)}{\partial\beta^{(v)}}=0,
    \end{align*}
    for almost every $X$, then we have $t_1^{(u)}=t_2^{(uv)}=t_3^{(uv)}=0$ for all $1\leq u,v\leq d$.
\end{definition}

\begin{example}
    We can validate that 
the function $E(X,\beta)=\mathrm{\exp}(X^{\top}\beta)/(1+\exp(X^\top\beta))$ is second-order strongly identifiable. 
%This claim still holds true when replacing $\mathrm{GELU}$ with other non-linear activation functions such as $\mathrm{sigmoid}$, and $\tanh$.
By contrast, linear experts $E(X,\beta_1,\beta_0)=X^{\top}\beta_1+\beta_0$ are not second-order strongly identifiable because of the PDE relation $\frac{\partial E(X;\beta_1,\beta_0)}{\partial\beta_1}=X^{\top}\frac{\partial E(X;\beta_1,\beta_0)}{\partial\beta_0}$, which leads to the undesired linear dependence 
\begin{align}
    \label{eq:linear_dependence}
    \frac{\partial F}{\partial\beta_1}=\frac{\partial^2F}{\partial\alpha\partial\beta_0}.
\end{align}
\end{example}

It can be seen that second-order strong identifiability implies the first-order version in Definition~\ref{def:weak_identifiability}. In addition, this condition has substantive technical and intuitive interpretations. Technically, it helps eliminate any linear dependence among terms in the aforementioned Taylor expansion of the density discrepancy, which is crucial for our proofs. Intuitively, second-order strong identifiability
helps prevent potential adverse interactions among parameters, namely, where gating parameters $\alpha$ interact with expert parameters $\beta_1$ and $\beta_0$ as in \Eqref{eq:linear_dependence}, thereby negatively parameter estimation efficiency. 
In fact, when the expert function is of linear form, which violates the first-order strong identifiability condition, \cite{nguyen2023demystifying} showed that parameter estimation rates became significantly slow since, they hinge on the solvability of a system of polynomial equations originating precisely in such linear dependence interactions.

Next, by relying on the following Voronoi loss (recall it from Section~\ref{sec:notation})
\begin{align}
    &\mathcal{L}_{2}(G,G^\star):=\sum_{j=1}^{K^\star}\Big|\sum_{i\in\mathcal{C}_j}\exp(\alpha_{0i})-\exp(\alpha^\star_{0j})\Big|\nonumber\\
    &+\sum_{j\in[K^\star]:|\mathcal{C}_j|=1}\sum_{i\in\mathcal{C}_j}\exp(\alpha_{0i})(\|\alpha_{1i}-\alpha^\star_{1j}\|+\|\beta_{i}-\beta^\star_{j}\|+|\sigma^2_{i}-\sigma^{\star2}_{j}|)\nonumber\\
    &+\sum_{j\in[K^\star]:|\mathcal{C}_j|>1}\sum_{i\in\mathcal{C}_j}\exp(\alpha_{0i})(\|\alpha_{1i}-\alpha^\star_{1j}\|^2+\|\beta_{i}-\beta^\star_{j}\|^2+|\sigma^2_{i}-\sigma^{\star2}_{j}|^2),
\end{align}
we are ready to capture the convergence behavior of parameter estimation in Theorem~\ref{theorem:parameter_rate}.
\begin{theorem}
    \label{theorem:parameter_rate}
    Suppose that the expert function $X\mapsto E(X,\beta)$ is second-order strongly identifiable. Under Assumptions~\ref{ass:prior_compact_support}, \ref{ass:lipschitz_expert}, \ref{ass:exp_identifiability}, \ref{ass:gating_biases}, \ref{ass:zero_gating_param}, and \ref{ass:prior_condition}, the lower bound $d_H(g_{G},g_{G^\star})\gtrsim\mathcal{L}_{2}(G,G^\star)$ holds for any mixing measure $G\in\mathscr O_{K}$. As a result, we obtain 
    \begin{equation*}
    \lim_{n\to\infty} \Pi\left(\left\{G\in\mathscr O_{K} : \, \mathcal{L}_{2}(G,G^\star) \geq M_n\sqrt{\log n/n}\right\} \mid (X_i, Y_i)_{i=1}^n\right)  = 0.
    \end{equation*}
    in $g_{G^\star}^n$-probability.
\end{theorem}
The proof of Theorem~\ref{theorem:parameter_rate} is in Appendix~\ref{appendix:parameter_rate}. There are three main implications of the above result which are worth highlighting.  Before going into their details, let us introduce some preliminary notions. We call the ground-truth parameters $\exp(\alpha^{\star}_{0j})$, $\alpha^{\star}_{1j}$, $\beta^{\star}_{j}$, $\sigma^{\star2}_{j}$ and experts $E(X,\beta^\star_j)$ \emph{exactly-specified} if they are fitted by one component, that is, $|\mathcal{C}_{j}|=1$, and \emph{over-specified} if they are fitted by more than one component, that is, $|\mathcal{C}_j|>1$.

\begin{enumerate}
    \item[(i)]\emph{Estimation rates for exactly-specified parameters and experts:} Combining the result of Theorem~\ref{theorem:parameter_rate} and the formulation of the Voronoi loss $\mathcal{L}_2(G,G^\star)$, we deduce that the rates for estimating exactly-specified parameters  $\exp(\alpha^{\star}_{0j})$, $\alpha^{\star}_{1j}$, $\beta^{\star}_{j}$, $\sigma^{\star2}_{j}$ are of order $\sqrt{\log(n)/n}$, which is identical to the rate uncovered under the exactly-specified setting. As a result, from the inequality in \Eqref{eq:expert_inequality}, it follows that exactly-specified experts $E(X,\beta^\star_j)$ enjoy posterior estimation rates of the same order.

    \item[(ii)] \emph{Estimation rates for over-specified parameters and experts:} On the other hand, when fitted by more than one component, the rates for estimating over-specified parameters $\exp(\alpha^{\star}_{0j})$, $\alpha^{\star}_{1j}$, $\beta^{\star}_{j}$, $\sigma^{\star2}_{j}$ become slower, in particular of order $(\log(n)/n)^{-1/4}$, and the same rate holds for estimating the over-specified experts $E(X,\beta^\star_j)$. Consequently, our results suggest that one needs a faster-growing number of data points to estimate over-specified experts than exactly-specified experts with the same error $\epsilon>0$: $\mathcal{O}(\epsilon^{-4})$ compared to $\mathcal{O}(\epsilon^{-2})$.

    \item[(iii)] \emph{Sample efficiency of second-order strongly identifiable experts vs. linear experts:} From the above remarks, we see that the estimation rates for second-order strongly identifiable experts are either of order $\sqrt{\log(n)/n}$ or $(\log(n)/n)^{1/4}$. These rates are significantly faster than those for linear experts studied by \cite{nguyen2023demystifying}, which violate the second-order strong identifiability condition. In particular, \cite{nguyen2023demystifying} showed that when fitted by two and three components, the rates for estimating linear experts are substantially slower, of orders $(\log(n)/n)^{1/8}$ and $(\log(n)/n)^{1/12}$, respectively. Therefore, our analysis uncovers that using second-order strongly identifiable experts is more sample-efficient than linear experts.
\end{enumerate}

We conclude this section by remarking that, while expressed in the framework of Bayesian posterior convergence, our parameter convergence results, because they arise from a fundamental inequality between Hellinger and Voronoi losses, may be applied to any other estimation procedure (e.g., maximum likelihood) that may be analyzed through the lenses of density convergence.

\section{Model selection}\label{sec:model_selection}

So far we have focused on density and parameter estimation. Another fundamental statistical task is selecting the model class that best describes the data-generating process. In the SMoGE setting, perhaps the most important quantity to determine is the number of experts $\kappa$. Indeed, if $\kappa$ is too small, the model is misspecified, while on the other hand, as shown in Section~\ref{sec:param_estimation}, an over-specified number of experts negatively affects posterior convergence for parameter estimation, as reflected by the need to rely on the loss $\mathcal L_2$ rather than the more favorable $\mathcal L_1$. Aside from theoretical considerations, model selection may also be relevant in applications where interpretability is important. For example, the number of experts may correspond to the number of latent populations across which input-response relationships differ, and identifying this number can therefore be of direct scientific interest.

It should be noted from the outset that, in our context, the model selection problem may be ill-defined or meaningful only within specific applications. Unless one assumes that the data-generating process itself belongs to a SMoGE family with $K^\star$ components, a ``true'' number of experts need not exist. Nevertheless, one may still wish to select the number of experts that yields the best predictive performance, or particular applications may motivate ad hoc but practically effective criteria. For instance, \cite{ludziejewski2024scalinglaw} study how the configuration of MoE models should scale with model size and training data in large language models. Rather than directly selecting the number of experts as a fixed architectural hyperparameter, the paper introduces a granularity parameter that controls the size of each expert relative to the feed-forward layer in the Transformer architecture \citep{vaswani2017attention}. Increasing granularity effectively splits a standard expert into multiple smaller experts, thereby increasing the total number of experts while keeping the overall parameter budget comparable. Using extensive experiments, the authors derive scaling laws that relate model loss to the number of training tokens, total model parameters, and granularity, enabling the computation of compute-optimal configurations for a given training budget. Their results show that the commonly used design—setting expert size equal to the feed-forward layer width—is generally suboptimal; instead, the optimal number of experts should be determined jointly with model size and data scale through these scaling laws, balancing the benefits of finer expert specialization against the additional routing overhead.

%{\color{red} Huy, can you please add references to papers that discuss the number of experts in the context of LLMs and whatnot? Thank you!!}

In what follows, instead, we adopt the same theoretical perspective used throughout the paper and define model selection relative to the data-generating process, investigating conditions under which \emph{model selection consistency} can be achieved. As in the rest of our analysis, this requires assuming that a true value $K^\star$ exists, and therefore that the data-generating distribution is of SMoGE type. We view this as a natural first step toward understanding model selection in this flexible class of models.

Before presenting the two strategies we consider, we note that both approaches are inherently Bayesian. In particular, both (i) learning $\kappa$ directly from the data and (ii) using the evidence lower bound obtained from variational inference as a measure of model fit do not readily translate to the frequentist procedures commonly used to estimate mixture-of-expert models. This highlights a distinctive advantage of Bayesian methods, where coherent inference can be carried out for any parameter of interest, including parameters that determine the model size, and can therefore be naturally adapted to the task of model selection.

\subsection{Model selection via a random-a-priori number of experts}

Our first approach, similarly to Subsection~\ref{sub:random_k_density}, treats the number of experts $\kappa$ as an unknown parameter and places a prior $\pi_\kappa$ on it. In contrast to the earlier setting, we now allow $\pi_\kappa$ to have potentially full support on $\NN$, so to enable model selection over an unbounded number of components. The parameter space\footnote{Without loss of generality, and as anticipated in Section~\ref{sec:notation}, we now adopt notation in which priors and posteriors are distributions on $\Theta_\infty$ rather than on the space of mixing measures.} on which we place the prior $\Pi$ is therefore
\[
\Theta_\infty:=\bigcup_{k\in\NN} \Theta^k.
\]

An application of Doob's posterior consistency theorem \citep{doob1949application,miller2018detailed}, together with arguments inspired by the analysis of traditional mixtures in \cite{nobile1994bayesian,miller2023consistency}, yields the following result on posterior consistency for the number of experts.

\begin{theorem}\label{thm:consistency_number_experts}
    Let Assumptions \ref{ass:prior_compact_support} and \ref{ass:lipschitz_expert} hold. Moreover, assume the following:
    \begin{enumerate}
        \item The model is identifiable, meaning that $g_{G}=g_{G'}$\footnote{With a slight abuse of notation, for any mixing measure $G\in\bigcup_{j\in\NN}\mathscr G_j$, we denote by $g_G$ both the resulting SMoGE joint density and the corresponding probability distribution. In the statement of the theorem we mean equality of the two probability measures, which implies equality of the densities almost everywhere. See Proposition~\ref{prop:identifiability_up_to_translation} above for a characterization of the scope of this assumption.} implies $G=G'$ for all $G, G'\in\bigcup_{k\in\NN}\mathscr G_k$;
        \item All priors are continuous and factorize across experts.
    \end{enumerate}
    Then there exists a set $\Theta_\star \subseteq \Theta_\infty$ such that $\Pi(\Theta_\star) = 1$ and, for all $\theta_\star\in\Theta_\star$,
    \begin{equation*}
    \lim_{n\to\infty}\Pi(K=K(\theta_\star)\mid (X_i, Y_i)_{i=1}^n) = 1
    \end{equation*}
    a.s.-$g_{G(\theta_\star)}^\infty$, where, for any $\theta \in \Theta_\infty$, $K(\theta)$ and $G(\theta)$ denote respectively the number of experts and the mixing measure associated with $\theta$.
\end{theorem}

Theorem~\ref{thm:consistency_number_experts} implies that, whenever the data-generating parameter belongs to a set $\Theta_\star$ of full prior measure, the posterior asymptotically recovers the true number of experts. In other words, the Bayesian procedure is consistent for model selection for such data-generating processes. The caveat, which is typical of results based on Doob's consistency theorem \citep{doob1949application}, is that the set $\Theta_\star$ is not characterized explicitly. Rather, it arises from the proof, which relies on martingale almost sure convergence arguments, and as a consequence, verifying whether a specific parameter value belongs to this set is generally impossible. This limitation is the price paid for obtaining a result of such generality and strength in terms of model selection consistency \citep{nobile1994bayesian,miller2023consistency}.

We conclude with a natural corollary of Theorem~\ref{thm:consistency_number_experts}, which shows that, under the same setting with a prior on $\kappa$ having potentially full support on $\NN$, posterior Hellinger contraction also holds for those parameter values in $\Theta_\star$.

\begin{corollary}\label{cor:density_estimation_doobs}
    Let $\theta_\star\in\Theta_\star$, $G^\star = G(\theta_\star)$, and $K^\star=K(\theta_\star)$, with $\Theta_\star$ as defined in Theorem~\ref{thm:consistency_number_experts}. Consider the SMoGE model with a prior on the number of experts having full support on $\NN$. Then, under the assumptions of Theorem~\ref{thm:our_abstract_rate_SMoGE} and Theorem~\ref{thm:consistency_number_experts}, the following holds for any sequence $M_n\to\infty$:
    \begin{equation*}
    \lim_{n\to\infty} \Pi\left(\left\{G\in \bigcup_{j\in\NN}\mathscr G_{j} : \,d_H(g_G, g_{G^\star}) \geq M_n\sqrt{\log n/n}\right\} \mid (X_i, Y_i)_{i=1}^n\right)  = 0.
    \end{equation*}
    in $g_{G^\star}^n$-probability.
\end{corollary}

The parameter estimation results of Section~\ref{sec:param_estimation} may also be extended to this setting, but we do not do so explicitly here as such extensions are minimal in light of Corollary~\ref{cor:density_estimation_doobs}.

\subsection{Model selection by variational inference}

While computation is not the primary focus of this paper, fitting Bayesian SMoGE models would typically rely on two main approaches: Markov chain Monte Carlo (MCMC), which constructs a Markov chain whose stationary measure is the target posterior \citep{robert2004monte, gelman2013bayesian}, and Variational Inference (VI), which approximates the posterior by finding the closest distribution within a tractable family \citep{jordan1999introduction, blei2017variational}. As we detail below, VI naturally provides a mechanism for performing model selection as a byproduct of the fitting process.

Specifically, VI proceeds by selecting a \emph{variational family} $\mathcal{Q} = \{q_\mu : \mu \in \mathcal{M}\}$ of distributions over the parameter space, and optimizing the evidence lower bound (ELBO),
$$\mathcal{E}(\mu) := \mathbb{E}_{\theta\sim q_\mu}\left[\log p\left((X_i,Y_i)_{i=1}^n, \theta\right) - \log q_\mu(\theta)\right],$$
with respect to the variational parameters $\mu\in\mathcal{M}\subseteq \mathbb{R}^{d_\mu}$; here, $p\left((X_i,Y_i)_{i=1}^n, \theta\right)$ denotes the joint distribution of the data and parameters under the Bayesian prior-likelihood model. For simpler models and specific choices of $\mathcal Q$, $\mathcal{E}(\mu)$ can often be maximized via closed-form coordinate ascent updates \citep{bishop2006pattern, wainwright2008graphical}. However, for more complex likelihoods like the SMoGE models considered here, optimization typically requires stochastic gradient-based methods. In this setting, the gradient of the ELBO is approximated using Monte Carlo samples from $q_\mu$, often in conjunction with variance reduction techniques \citep{ranganath2014black, kucukelbir2017automatic}.

Returning to the problem of model selection, the ELBO provides a highly practical criterion. Because the ELBO bounds the log marginal likelihood, it serves as a computationally tractable proxy for model evidence and can guide the selection of model complexity, such as the number of experts in a SMoGE architecture. In our framework, we advocate complementing the prior over the number of experts with an ELBO-based heuristic: by fitting several SMoGE models across a plausible range of expert counts, $\kappa = K$, we select the model that maximizes the ELBO as the one exhibiting the highest approximate evidence.

While employing the ELBO for model selection in mixtures of experts is not a new strategy \citep{waterhouse1996bayesian,Ueda2002}, rigorous analyses establishing the frequentist consistency of this procedure have only recently emerged for non-singular models \citep{zhang2024bayesian} and specific mixture specifications \citep{wang2024estimating}. Extending these highly technical consistency results from simple mixtures to our more complex SMoGE formulation is beyond the scope of this paper. Nevertheless, we view this as a highly promising strategy for future theoretical validation, primarily because VI scales significantly better than the transdimensional MCMC algorithms \citep{green1995reversible} that would otherwise be required to fit SMoGE models with a random number of experts. To support this direction, we conclude the section by reporting preliminary empirical evidence validating the ELBO-based model selection procedure.

We conducted a series of stylized synthetic experiments, whose results we describe next. Figure \ref{fig:elbo_selection} illustrates the procedure's performance across two scenarios: a simpler setup with $K^\star=2$ true linear experts in $d=2$ dimensions (Figure \ref{fig:elbo_selection}(a)), and a more complex $K^\star=4$ linear expert architecture in $d=6$ dimensions (Figure \ref{fig:elbo_selection}(b)). For both setups, we evaluate a range of candidate model sizes by maximizing the ELBO using black-box variational inference \citep{ranganath2014black, kucukelbir2017automatic} over 50 independent simulations. For each sample size, we record the proportion of trials in which each candidate $K$ achieved the highest final evidence lower bound. The results provide empirical evidence for consistency: at small sample sizes, the ELBO's inherent complexity penalty conservatively favors fewer experts, but as the sample size increases, the selection procedure overwhelmingly identifies the true number of experts $K^\star$. It is worth noting that the true data-generating mechanisms in these simulations, as described in Appendix~\ref{app:elbo_experiments}, are not strictly SMoGE models, as they employ deterministic expert assignments that introduce sharp discontinuities in the conditional mean structure. Nonetheless, model selection consistency, in this looser sense of identifying the correct number of underlying linear regimes, appears to hold robustly.

\begin{figure}
    \centering
    \begin{subfigure}{0.48\textwidth}
        \centering
        \includegraphics[width=\linewidth]{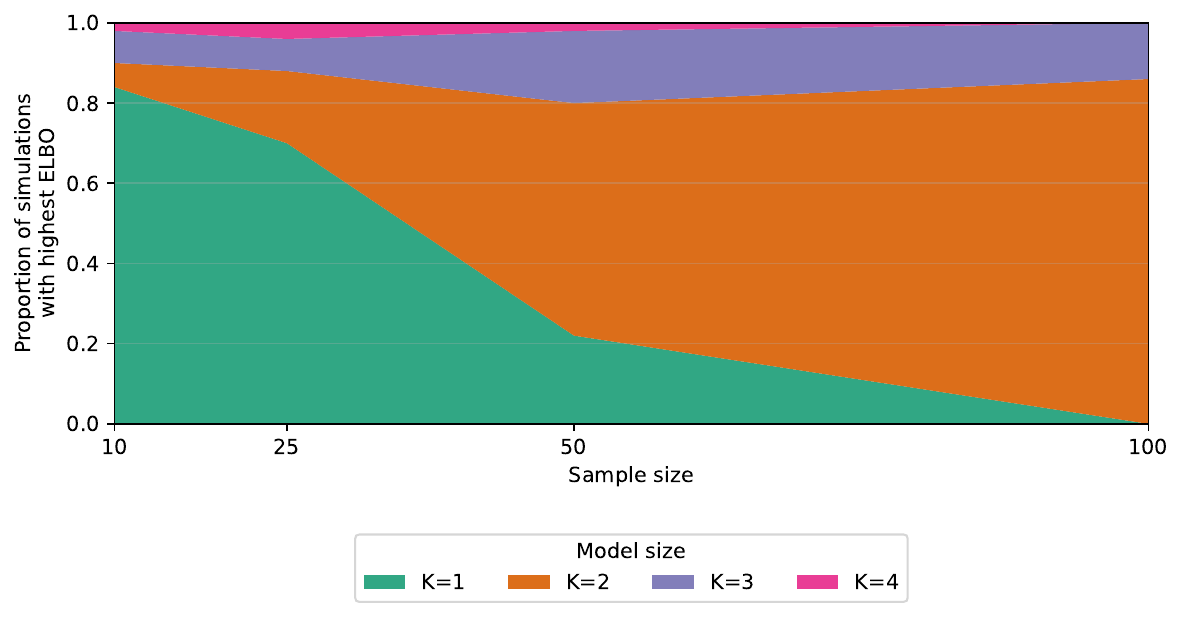}
        \caption{Data-generating process with $K^\star=2$, $d=2$}
        \label{fig:elbo_k2}
    \end{subfigure}
    \hfill
    \begin{subfigure}{0.48\textwidth}
        \centering
        \includegraphics[width=\linewidth]{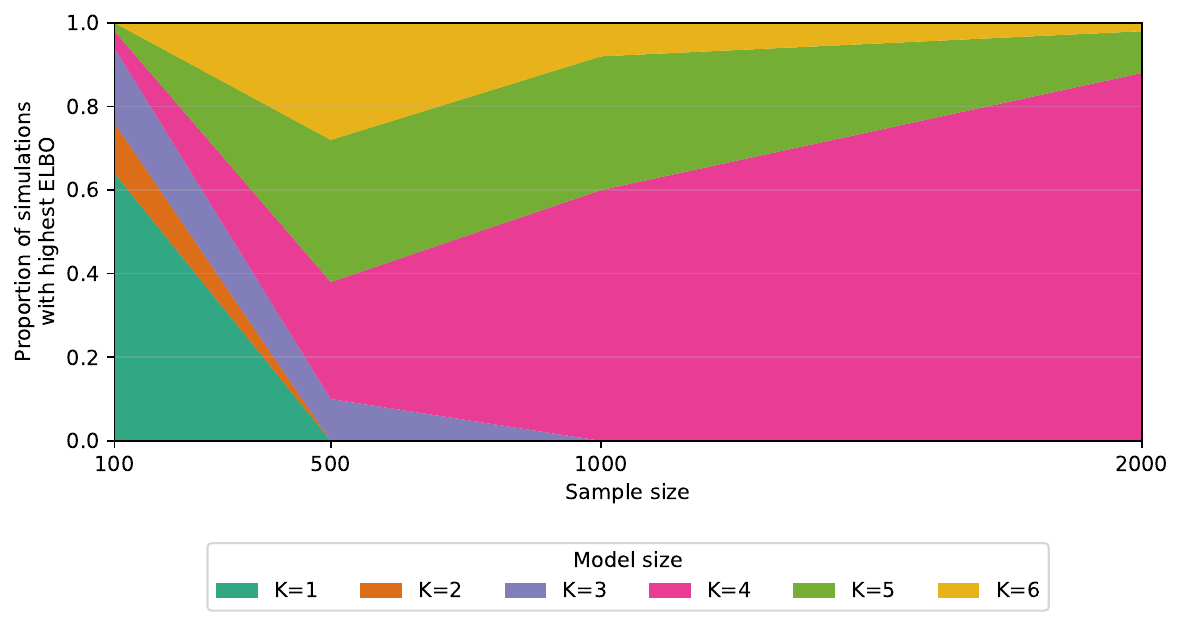}
        \caption{Data-generating process with $K^\star=4$, $d=6$}
        \label{fig:elbo_k4}
    \end{subfigure}
    \caption{Model selection by VI-derived ELBO maximization. As the sample size increases, the procedure selects the correct number of experts with high probability. Win proportions are approximated over 50 independent simulations per sample size.}
    \label{fig:elbo_selection}
\end{figure}

To further validate the robustness of the ELBO-based selection, we also evaluate its performance on datasets generated with a logit-based gating mechanism that closely resembles the softmax structure of true SMoGE models. In this setup, we fix the sample size to $n=500$ and vary the true number of regimes ($K^\star \in \{1, 2, 3\}$), the dimensionality ($d \in \{2, 4\}$), and the ``sharpness'' of the boundary separating the expert regions. Table~\ref{tab:elbo_tables} reports the proportion of 100 independent simulations in which each candidate $K \in \{1, \dots, 7\}$ attained the highest ELBO. The procedure reliably identifies the correct $K^\star$ as the most frequent winner across all configurations. Notably, selection accuracy improves when the gating sharpness is higher (Table~\ref{tab:elbo_softmax_high}), as the clearer separation between expert regimes naturally reduces the posterior overlap and makes the true model complexity easier to distinguish. Furthermore, we observe that increasing the true number of regimes $K^\star$ predictably slows down the convergence of the algorithm, as the variational optimization landscape becomes inherently more complex to navigate.

\begin{table}
    \centering
    \caption{Proportion of simulations in which each candidate $K$ attained the highest ELBO over 100 independent trials. The true data-generating process relies on a max-logit gating structure evaluated at two different separation scales. Bold values indicate the most frequently selected model size.}
    \label{tab:elbo_tables}
    
    \begin{subtable}{\textwidth}
        \centering
        \caption{Low gating sharpness (separation $= 5.0$)}
        \label{tab:elbo_softmax_low}
        \begin{tabular}{ccccccccc}
        \toprule $d$ & $K^\star$ & $K=1$ & $K=2$ & $K=3$ & $K=4$ & $K=5$ & $K=6$ & $K=7$\\
        \midrule
        2 & 1 & \textbf{1.00} & 0.00 & 0.00 & 0.00 & 0.00 & 0.00 & 0.00\\
        2 & 2 & 0.00 & \textbf{0.76} & 0.20 & 0.04 & 0.00 & 0.00 & 0.00\\
        4 & 3 & 0.00 & 0.00 & \textbf{0.73} & 0.20 & 0.06 & 0.01 & 0.00\\
        \bottomrule
        \end{tabular}
    \end{subtable}
    
    \vspace{0.5cm} % Add some vertical space between the two subtables
    
    \begin{subtable}{\textwidth}
        \centering
        \caption{High gating sharpness (separation $= 10.0$)}
        \label{tab:elbo_softmax_high}
        \begin{tabular}{ccccccccc}
        \toprule $d$ & $K^\star$ & $K=1$ & $K=2$ & $K=3$ & $K=4$ & $K=5$ & $K=6$ & $K=7$\\
        \midrule
        2 & 1 & \textbf{1.00} & 0.00 & 0.00 & 0.00 & 0.00 & 0.00 & 0.00\\
        2 & 2 & 0.00 & \textbf{0.84} & 0.14 & 0.02 & 0.00 & 0.00 & 0.00\\
        4 & 3 & 0.00 & 0.00 & \textbf{0.78} & 0.21 & 0.01 & 0.00 & 0.00\\
        \bottomrule
        \end{tabular}
    \end{subtable}
\end{table}

\section{Extension to mixture-of-experts for classification}\label{sec:classification}

Throughout the article, we have focused on SMoGE models for regression tasks and have analyzed their theoretical properties in depth. Another fundamental supervised learning problem, for which Bayesian MoE models may be effectively employed to increase model flexibility, is classification. Although we have not treated this setting explicitly, it can be incorporated into our framework with only minor modifications, as follows.

Suppose that the response variable takes values in a finite set of classes $\{1,\dots,S\}$. In this case, the data may be assumed to be drawn according to the following Softmax Mixture of Multinomial Experts (SMoME) model:
\begin{align*}
    G^\star & := \sum_{j=1}^{K^\star} \exp(\alpha_{0j}^\star) 
    \delta_{(\alpha_{1j}^\star, \beta_{j}^\star)}, \\
    f_{G^\star}(Y = s \mid X) 
    & := \sum_{j=1}^{K^\star} 
    \frac{\exp(\alpha_{0j}^\star + X^\top\alpha_{1j}^\star)}
         {\sum_{\ell=1}^{K^\star}\exp(\alpha_{0\ell}^\star + X^\top\alpha_{1\ell}^\star)}
    \times 
    \frac{\exp\!\big(E_s(X,\beta_j^\star)\big)}
         {\sum_{r=1}^{S}\exp\!\big(E_r(X,\beta_j^\star)\big)},
\end{align*}
for $s=1,\dots,S$, where $E_s(X,\beta)$ denotes the expert function associated with class $s$, parametrized by $\beta$. In this case, $f_{G^\star}$ is a density with respect to the counting measure on $\{1,\dots,S\}$, and the usual Bayesian pipeline involving prior specification and posterior computation can be implemented without difficulty.

While we omit the details for brevity, extending much of our theoretical analysis to this setting presents no particular difficulty. In particular, posterior contraction for density estimation may still be established using the general theory developed by \cite{ghosal2000convergence}, while parameter estimation results based on Voronoi losses can be adapted from the recent contribution of \cite{nguyen2024multinomial}.

\section{Discussion}\label{sec:discussion}

In this paper, we have presented a theoretical study of Bayesian MoE models for Gaussian regression with softmax gating. We focused on three key aspects: density estimation, parameter estimation, and model selection. For density estimation, we derived posterior contraction rates for both fixed and random numbers of experts. For parameter estimation, we established convergence guarantees using tailored Voronoi-type losses, which account for the complex identifiability structure of MoE models in both correctly specified and over-specified settings. Finally, we studied model selection using a random-a-priori number of experts and variational inference.

\subsection{Implications of our work}
Our results have several practical implications for working with MoE models. In particular, we have highlighted the role of Voronoi-type losses in characterizing prior support conditions that ensure Hellinger contraction and in guiding parameter convergence; this provides a useful tool to design priors with good asymptotic behavior and to assess convergence speed as a function of the sample size.

Moreover, regarding model selection, our analysis shows that placing a prior with sufficiently large support on the number of experts can lead to model selection consistency, illustrating the robustness of the Bayesian approach for SMoGE models. At the same time, our parameter estimation results warrant some caution: overly diffuse priors on the number of experts can harm performance, as parameter contraction rates deteriorate when too many redundant experts are included. In light of these considerations, our preliminary empirical results indicate that using the ELBO from VI procedures may offer a practical alternative for model selection. This approach avoids the computational overhead of treating the number of experts as random, while remaining effective provided that one focuses on a reasonable range of candidate values when running VI.

\subsection{Future directions}

Several directions for future work remain. First, while we studied densely activated MoE models—where all experts are active for each data point—modern architectures often use sparsely activated experts \citep{shazeer2017topk, nguyen2024topk} to increase capacity while keeping computation manageable. Extending our theoretical results to sparse MoE models would help provide a solid understanding of these widely used, scalable architectures.

Second, all of our results are based on an inherent well-specification assumption, which may be hard to justify in practice. Given the remarkable practical success of MoE models in capturing complex data-generating processes, it would be interesting to investigate their behavior under misspecification (e.g., in an under-specified setting with fewer components than $K^\star$), and to study whether this class of models offers benefits, in terms of improved approximation capacity, compared to more traditional models like mixtures with input-free gating.

Third, the connections between MoE models and Bayesian nonparametric models with covariate dependence are still not fully explored. It would be interesting to see if the tools we developed here for finite mixtures can be applied to more complex, infinite-dimensional settings.

Finally, while we focused on traditional Bayesian inference, recent approaches to uncertainty quantification—sometimes called ``post-Bayesian'' or ``generalized Bayesian'' methods \citep{bissiri2016general, knoblauch2022optimization, fong2023martingale, fortini2024predictive}—offer alternative ways to perform inference using loss functions and predictive rules rather than the standard prior-likelihood framework. These methods have been applied successfully to mixture models \citep{rodriguez2024martingale} and deep learning pipelines \citep{lee2023martingale, wu2024posterior, ng2025tabmgp, bariletto2026scalable}, which are closely connected to modern MoE architectures. Extending such approaches to MoE models could provide scalable and robust alternatives to conventional posterior inference.

\clearpage
\begin{center}

{\bf{\LARGE{Appendices}}}
\end{center}

\appendix
\section{Proofs}
\label{app:theorem}

% Note: in this sample, the section number is hard-coded in. Following
% proper LaTeX conventions, it should properly be coded as a reference:

In this appendix, we report the proofs of the theoretical results presented in the main body of the article.

\subsection{Proof of Theorem~\ref{thm:consistency}}\label{proof:consistency}

As for the first statement, the assumption it requires is equivalent to the classical KL support condition by \cite{schwartz1965bayes}. Moreover, one easily checks that, if a sequence of SMoGE models $g_{G_1}, g_{G_2},\ldots$ converges weakly to $g_{G^\star}$, it must converge in Hellinger distance as well. In the terminology of \cite{bariletto2025posterior}, the model is sequentially identifiable and consistency immediately follows \citep[see also][]{bariletto2025necessary,walker2005data}.

As for the second statement, if Assumptions~\ref{ass:prior_compact_support} and \ref{ass:lipschitz_expert} also hold, then the upcoming proof of Theorem~\ref{thm:density_estim_finite_prior} shows that any set of the form $\{G\in\mathscr G_{K^\star} : \kl(g_{G^\star},g_G)<\delta\}$ always includes a set of the form $\{G\in\mathscr G_{K^\star} : \mathcal L_1(G^\star,G)<t\}$ for some $t>0$, as long as $\delta>0$ is small enough. Therefore, by the assumption stated in the theorem, we have
\begin{equation*}
    \Pi(\{G\in\mathscr G_{K^\star} : \kl(g_{G^\star},g_G)<\delta\})\geq \Pi(\{G\in\mathscr G_{K^\star} : \mathcal L_1(G^\star,G)<t\})>0.
\end{equation*}
The arguments presented in the last paragraph conclude the proof of the second statement in the theorem.

\subsection{Proof of Theorem~\ref{thm:our_abstract_rate_SMoGE}}\label{proof:our_abstract_rate_SMoGE}

Clearly, the statement of Theorem~\ref{thm:our_abstract_rate_SMoGE} is a particular case of the statement of Theorem~\ref{thm:density_estim_finite_prior} (setting $\pi_\kappa(\{K^\star\})=1$), so we refer the reader to the next proof for a more general treatment.

\subsection{Proof of Theorem~\ref{thm:density_estim_finite_prior}}

% %
% Using the notation from the proof of Theorem~\ref{thm:our_abstract_rate_SMoGE}, $\mathscr G_k$ can be covered with $N_k = \mathcal O(\varepsilon)$ $\mathscr L^1$-balls of radius $\varepsilon>0$ for all $k\in\{1,\ldots, K\}$, and therefore $\mathscr O_K$ can be covered with $\sum_{k=1}^K N_k = \mathcal O(\varepsilon)$ such balls as well. Therefore, Condition \ref{cond_1} is satisfied. Moreover, as for Condition \ref{cond_3}, it is enough to rewrite
% %
% \begin{align*}
%     \Pi\left(\left\{ 
%     G\in\mathscr O_K : \kl(g_{G^\star},g_G) \leq \varepsilon_n^2 \right\}\right)&= \sum_{k=1}^K \Pi_k\left(\left\{ 
%     G\in\mathscr G_k : \kl(g_{G^\star}, g_G)  \leq \varepsilon_n^2 \right\}\right) \cdot \pi_\kappa(\{k\}) \\
%     & \geq \Pi_{k^\star}\left(\left\{ 
%     G\in\mathscr G_{k^\star} : \kl(g_{G^\star},g_G)  \leq \varepsilon_n^2 \right\}\right) \cdot \pi_\kappa(\{k^\star\}).
% \end{align*}
% %
% So the same line of proof as for Theorem~\ref{thm:our_abstract_rate_SMoGE} shows that it is enough to impose Assumption~\ref{ass:prior_condition} on $\Pi_{k^\star}$.

We rely on the following result, which is an adaptation of Theorem 8.11 in \cite{ghosal2017fundamentals}.

\begin{theorem}[\cite{ghosal2017fundamentals}]\label{thm:abstract_rate}

Let $(\varepsilon_n)_{n\in\NN}$ and $(\bar\varepsilon_n)_{n\in\NN}$ be such that $\varepsilon_n,\bar\varepsilon_n\geq n^{-1/2}$ for all large $n\in\NN$. Moreover assume that, for all large $j\in\NN$ and some sequence of sets $\mathscr O^{n}\subseteq\mathscr O_{K}$, the following conditions are satisfied:
\begin{enumerate}
    \item $\sup_{\varepsilon \ge\varepsilon_n}\log N(\varepsilon/2, \{G\in\mathscr O^n : d_H(g_G, g_{G^\star})\leq 2\varepsilon\}, d_H) \leq n\varepsilon_n^2$;\label{cond_1}
    \item $\frac{\Pi(\mathscr O_{K}\setminus \mathscr O^n)}{\Pi(B(G^\star, \bar\varepsilon_n))} = o\big(e^{-D_n n\bar\varepsilon_n^2}\big)$ for some sequence $D_n\to\infty$;\label{cond_2}
    \item $\frac{\Pi(\{G\in\mathscr O^n : \,j\varepsilon_n < d_H(g_G, g_{G^\star})\leq 2j\varepsilon_n\})}{\Pi(\left\{ G\in \mathscr O_K : \,\kl(g_{G^\star}, g_{G})\leq \varepsilon_n^2\right\})}\leq e^{n\varepsilon_n^2 j^2/16}$.\label{cond_3}
\end{enumerate}
Then, for any positive sequence $M_n\to\infty$, we have
\begin{equation*}
    \lim_{n\to\infty} \Pi\big(\{G\in\mathscr O_{K} : \,d_H(g_G, g_{G^\star})\geq M_n\varepsilon_n\} \mid (X_i, Y_i)_{i=1}^n\big)  = 0.
\end{equation*}
in $g_{G^\star}^n$-probability.
\end{theorem}

Preliminarily, we note that the statement of Theorem 8.11 in \cite{ghosal2017fundamentals} is based on distances more general than $d_H$, as long as they satisfy certain testing conditions for some parameters $\xi$ and $K$ (in the notation of the book). Nevertheless, on page 197, the authors show that for the Hellinger metric these conditions are satisfied for $\xi=1/2$ and $K=1/8$, which we directly incorporate in the above version of the result.

In what follows, we show that Conditions \ref{cond_1}, \ref{cond_2}, and \ref{cond_3} in Theorem \ref{thm:abstract_rate} are satisfied by our SMoGE model under Assumptions \ref{ass:prior_compact_support}, \ref{ass:lipschitz_expert} and \ref{ass:prior_condition} and the following choices:
\begin{itemize}
    \item $\varepsilon_n=\bar\varepsilon_n=M\sqrt{\log n/n}$ for some large enough $M>0$
    \item $\mathscr O^n\equiv\mathscr O_K$ for all $n\in\NN$
\end{itemize}

\paragraph{Condition \ref{cond_1}.}

We begin by noting that
\begin{align*}
    \sup_{\varepsilon \ge\varepsilon_n}\log N(\varepsilon/2, \{G\in\mathscr O^n : d_H(g_G, g_{G^\star})\leq 2\varepsilon\}, d_H) & \leq \sup_{\varepsilon \ge\varepsilon_n}\log N(\varepsilon/2, \mathscr O_K, d_H) \\
    & = \log N(\varepsilon_n/2, \mathscr O_K, d_H),
\end{align*}
so it is enough to show
\begin{equation*}
    \log N(\varepsilon_n/2, \mathscr G_K, d_H)\leq n\varepsilon_n^2.
\end{equation*}
Moreover, it is easily checked that, by choosing $M$ large enough, the above condition is implied by
\begin{equation*}
    \log N(\varepsilon, \mathscr O_K, d_H)\lesssim \log(1/\varepsilon) \quad \textnormal{for all } \varepsilon>0.
\end{equation*}
Finally, a further simplification comes from noting that $d_H^2(f,f')\leq\| f-f'\|_{\mathscr L^1}$ for any two density functions $f,f'$, where $\|\cdot\|_{\mathscr L^1}$ denotes the $\mathscr L^1$ norm on the space of integrable functions. In particular, this implies that we are left to show
\begin{equation*}
    \log N(\varepsilon, \mathscr O_K, \| \cdot\|_{\mathscr L^1})\lesssim \log(1/\varepsilon).
\end{equation*}

We now proceed as follows. We will explicitly construct an $\varepsilon$-cover of $\mathscr G_K$ in $\| \cdot\|_{\mathscr L^1}$, whose cardinality $N_K$ we show to be $\mathcal O(\varepsilon^{-1})$. Because this $O(\varepsilon^{-1})$ result does not depend on the specific choice of $K\in\NN$, it holds for all $k\in\{1,\ldots, K\}$. Therefore, we will have shown that $\mathscr O_K=\bigcup_{k=1}^K\mathscr G_k$ itself has an $\varepsilon$-cover in $\| \cdot\|_{\mathscr L^1}$ of cardinality $\sum_{k=1}^K N_k = \mathcal O(\varepsilon^{-1})$. Therefore, Condition \ref{cond_1} will be satisfied. 

So let
\begin{align*}
    \Delta & := \{(\alpha_0,\alpha_1) \in \RR^{d+1} : \exists (\beta, \sigma^2)\in\RR^{p+1} \textnormal{ s.t. } (\alpha_0,\alpha_1, \beta, \sigma^2) \in \Theta\}, \\
    \Gamma & := \{(\beta, \sigma^2)\in\RR^{p+1} : \exists (\alpha_0,\alpha_1) \in \RR^{d+1} \textnormal{ s.t. } (\alpha_0,\alpha_1, \beta, \sigma^2) \in \Theta\}
\end{align*}
be the projections of $\Theta$ onto the subspaces of gate parameters and expert parameters, respectively. Because $\Theta$ is compact in the Euclidean norm $\|\cdot\|$, so are $\Delta$ and $\Omega$. This implies that one can cover them with at most $C/\varepsilon^{d+1}$ and $C'/\varepsilon^{p+1}$ Euclidean balls of radius $\varepsilon>0$, where $C,C'>0$ are scaling constants \citep{wainwright2019high}. Denote by $O_\Delta$ and $O_\Gamma$ the sets of centers of these $\varepsilon$-covers and, for each $G\in\mathscr G_K$ such that
\begin{equation*}
    G = \sum_{j=1}^K \exp(\alpha_{0j})\delta_{(\alpha_{1j}, \beta_{j}, \sigma^2_j)} \quad \textnormal{ for some } (\alpha_{0j},\alpha_{1j}, \beta_{j}, \sigma^2_j)_{j=1}^K \in\Theta^K,
\end{equation*}
define $\bar G\in\mathscr G_K$ as
\begin{equation*}
    G_1 := \sum_{j=1}^K \exp(\bar\alpha_{0j})\delta_{(\bar\alpha_{1j}, \beta_{j}, \sigma^2_j)},
\end{equation*}
where
\begin{equation*}
     (\bar\alpha_{0j}, \bar\alpha_{1j}) \in \argmin_{(a_0, a_1)\in O_\Delta} \| (a_1, a_2) - (\alpha_{0j}, \alpha_{1j})\|.
\end{equation*}
for all $j=1,\dots,K$. In turn, for the same $G\in\mathscr G_K$, define
\begin{equation*}
    G_2 := \sum_{j=1}^K \exp(\bar\alpha_{0j})\delta_{(\bar\alpha_{1j}, \bar\beta_{j}, \bar\sigma^2_j)},
\end{equation*}
where
\begin{equation*}
    (\bar\beta_{j}, \bar\sigma_j^2) \in \argmin_{(b, s^2)\in O_\Gamma} \| (b, s^2) - (\beta_{j}, \sigma^2_j)\|.
\end{equation*}
for all $j=1,\dots,K$. Intuitively, $G_1$ replaces the gate parameter vectors with the closest ones among the centers $O_\Delta$, and $G_2$ further replaces the expert parameter vectors with the closest ones among the centers $O_\Gamma$.

Now denote
\begin{align*}
    \textnormal{SoftMax} : \mathbb R^K & \to [0,1]^K,\\
    z & \mapsto \left(\frac{\exp(z_1)}{\sum_{\ell=1}^K\exp(z_\ell)}, \cdots,\frac{\exp(z_K)}{\sum_{\ell=1}^K\exp(z_\ell)} \right) 
\end{align*}
which is a 1-Lipschitz continuous function in the Euclidean norm \citep{gao2017softmax}, meaning that
\begin{equation*}
    \|\textnormal{SoftMax}(z) - \textnormal{SoftMax}(z')\| \leq \| z-z'\|
\end{equation*}
for all $z,z'\in\RR^K$. From this and the boundedness of $\XX$, we deduce
\begin{align*}
    &\| g_{G} - g_{G_1}\|_{\mathscr L^1}  = \int_\XX \int_\RR \vert f_G (Y\mid X) - f_{G_1}(Y\mid X) \vert\mathrm dY \, p(X)\mathrm dX \\
    & \leq \int_\XX \sum_{j=1}^K \left\vert \frac{\exp(\alpha_{0j} + X_i^\top\alpha_{1j})}{\sum_{\ell=1}^{K}\exp(\alpha_{0\ell} + X^\top\alpha_{1\ell} )} - \frac{\exp(\bar\alpha_{0j} + X^\top\bar\alpha_{1j})}{\sum_{\ell=1}^{K}\exp(\bar\alpha_{0\ell} + X^\top\bar\alpha_{1\ell})} \right\vert p(X) \mathrm dX \\
    & = \int_\XX \left \vert \left\vert\textnormal{SoftMax}\left((\alpha_{0j} + X^\top\alpha_{1j})_{j=1}^K\right) - \textnormal{SoftMax}\left((\bar\alpha_{0j} + X^\top\bar\alpha_{1j})_{j=1}^K\right)\right\vert^\top \boldsymbol{1}_K \right\vert p(X) \mathrm dX \\
    & \leq \sqrt K \int_\XX \left \| \textnormal{SoftMax}\left((\alpha_{0j} + X^\top\alpha_{1j})_{j=1}^K\right) - \textnormal{SoftMax}\left((\bar\alpha_{0j} + X^\top\bar\alpha_{1j})_{j=1}^K\right) \right\| \, p(X) \mathrm dX \\
    & \leq \sqrt K \int_\XX \left \| (\alpha_{0j} + X^\top\alpha_{1j})_{j=1}^K - (\bar\alpha_{0j} + X^\top\bar\alpha_{1j})_{j=1}^K \right\| \, p(X) \mathrm dX \\
    & \leq \sqrt K \int_\XX \left \| (\alpha_{0j} -\bar\alpha_{0j})_{j=1}^K\right\| + \left\| (X^\top(\alpha_{1j}-\bar\alpha_{1j}))_{j=1}^K \right\| \, p(X) \mathrm dX \\
    & =\sqrt{K}\left[ \int_\XX \sqrt{\sum_{j=1}^K \left\vert \alpha_{0j}-\bar\alpha_{0j}\right\vert^2}\, p(X)\mathrm dX + \int_\XX \sqrt{\sum_{j=1}^K \left\vert X^\top(\alpha_{1j}-\bar\alpha_{1j})\right\vert^2}\, p(X)\mathrm dX \right]\\
    &\lesssim \varepsilon.
\end{align*}
Moreover
\begin{align}\label{eq:L1dist_G1_G2}
    \| & g_{G_1} - g_{G_2}\|_{\mathscr L^1} \\
    & \leq \int_\XX \sum_{j=1}^K\frac{\exp(\bar\alpha_{0j} + X^\top\bar\alpha_{1j})}{\sum_{\ell=1}^{K}\exp(\bar\alpha_{0\ell} + X^\top\bar\alpha_{1\ell})} \|\mathcal{N}(\cdot\mid E(X,\beta_{j}), \sigma_j^2) - \mathcal{N}(\cdot\mid E(X,\bar\beta_{j}), \bar\sigma_j^2)\|_{\mathscr L^1} \,p(X)\mathrm dX.\nonumber
\end{align}
For all $j=1,\dots,K$ and $X\in\XX$,
\begin{align*}
    \|\mathcal{N}(\cdot \mid E(X,\beta_j), \sigma_j^2) & - \mathcal{N}(\cdot\mid E(X,\bar\beta_{j}), \bar\sigma_j^2)\|_{\mathscr L^1} \\
    &\leq \|\mathcal{N}(\cdot\mid E(X,\beta_{j}), \sigma_j^2) - \mathcal{N}(\cdot\mid E(X,\bar\beta_{j}), \sigma_j^2)\|_{\mathscr L^1} \\
    & + \|\mathcal{N}(\cdot \mid E(X,\bar\beta_{j}), \sigma_j^2) - \mathcal{N}(\cdot\mid E(X,\bar\beta_{j}), \bar\sigma_j^2)\|_{\mathscr L^1}.
\end{align*}
Using a first-order Taylor expansion on the mean of the normal density around $E(X,\bar\beta_{j})$ and invoking the boundedness of $\XX$ and the Lipschitz property of $E(X,\cdot)$, we get
\begin{align*}
    \|\mathcal{N}(\cdot\mid &E(X,\beta_{j}), \sigma_j^2) - \mathcal{N}(\cdot\mid E(X, \bar\beta_{j}), \sigma_j^2)\|_{\mathscr L^1}\\
    & \leq \left\vert E(X,\beta_{j})-E(X,\bar\beta_{j}) \right\vert\int_\RR\left\vert\frac{\partial}{\partial \mu}\mathcal{N}(\cdot\mid \mu, \sigma_j^2)\Big\vert_{\mu = \mu_{\star}}\right\vert\mathrm dY \\
    & \lesssim \varepsilon.
\end{align*}
In a similar fashion, we obtain
\begin{equation*}
    \|\mathcal{N}(\cdot \mid E(X,\bar\beta_{j}), \sigma_j^2) - \mathcal{N}(\cdot\mid E(X,\bar\beta_{j}), \bar\sigma_j^2)\|_{\mathscr L^1} \lesssim \varepsilon
\end{equation*}
and therefore, by \Eqref{eq:L1dist_G1_G2}, $\| g_{G_1} - g_{G_2}\|_{\mathscr L^1}\lesssim\varepsilon$. Thus, for all $G$, we have shown the existence of a measure $G_2\in\mathscr G_K$ such that
\begin{equation*}
    \| g_{G} - g_{G_2}\|_{\mathscr L^1} \leq \| g_{G} - g_{G_1}\|_{\mathscr L^1} +\| g_{G_1} - g_{G_2}\|_{\mathscr L^1} \lesssim \varepsilon,
\end{equation*}
with $G_2$ belonging to the set
\begin{equation*}
    \Xi := \left\{\sum_{j=1}^K\exp(a_{0j})\delta_{(a_{1j}, b_{j}, s^2_j)} \in \mathscr G_K :\, (a_{0j}, a_{1j}) \in O_\Delta,\, (b_{j}, s^2_j) \in O_\Gamma\right\}.
\end{equation*}
Therefore, for some large enough $C_1>0$, $\Xi$ yields a $C_1\varepsilon$-cover of $\mathscr G_K$ with respect to $\|\cdot\|_{\mathscr L^1}$. Moreover, recalling that the cardinalities of $O_\Delta$ and $O_\Gamma$ are upper-bounded respectively by $C/\varepsilon^{d+1}$ and $C'/\varepsilon^{p+1}$, the cardinality of $\Xi$ is at most $C_2 K /\varepsilon^{d+p+2}$ for some constant $C_2>0$. This yields the desired result
\begin{equation*}
    \log N(\varepsilon, \mathscr G_K, \|\cdot\|_{\mathscr L^1}) \lesssim \log(1/\varepsilon).
\end{equation*}

\paragraph{Condition \ref{cond_2}.} Because $\mathscr O_K = \mathscr O^n$ for all $n\in\NN$, Condition \ref{cond_2} is trivially satisfied.

\paragraph{Condition \ref{cond_3}.}
% Moreover, as for Condition \ref{cond_3}, it is enough to rewrite
% %
% \begin{align*}
%     \Pi\left(\left\{ 
%     G\in\mathscr O_K : \kl(g_{G^\star},g_G) \leq \varepsilon_n^2 \right\}\right)&= \sum_{k=1}^K \Pi_k\left(\left\{ 
%     G\in\mathscr G_k : \kl(g_{G^\star}, g_G)  \leq \varepsilon_n^2 \right\}\right) \cdot \pi_\kappa(\{k\}) \\
%     & \geq \Pi_{k^\star}\left(\left\{ 
%     G\in\mathscr G_{k^\star} : \kl(g_{G^\star},g_G)  \leq \varepsilon_n^2 \right\}\right) \cdot \pi_\kappa(\{k^\star\}).
% \end{align*}
% %
% So the same line of proof as for Theorem~\ref{thm:our_abstract_rate_SMoGE} shows that it is enough to impose Assumption~\ref{ass:prior_condition} on $\Pi_{k^\star}$.
We begin by noting that Condition \ref{cond_3} is implied by
\begin{equation*}
    \Pi(\left\{ G\in \mathscr O_K : \,\kl(g_{G^\star}, g_{G})\leq \varepsilon_n^2\right\})\geq e^{-n\varepsilon_n^2/16},
\end{equation*}
and moreover that
\begin{equation*}
    \Pi(\left\{ G\in \mathscr O_K : \,\kl(g_{G^\star}, g_{G})\leq \varepsilon_n^2\right\}) \geq \Pi(\left\{ G\in \mathscr G_{K^\star} : \,\kl(g_{G^\star}, g_{G})\leq \varepsilon_n^2\right\}).
\end{equation*}
Therefore, we focus on proving that
\begin{equation*}
    \Pi(\left\{ G\in \mathscr G_{K^\star} : \,\kl(g_{G^\star}, g_{G})\leq \varepsilon_n^2\right\})\geq e^{-n\varepsilon_n^2/16}
\end{equation*}
Notice that, for any $G\in\mathscr G_{K^\star}$, we have
\begin{equation*}
    \kl(g_{G^\star}, g_G) \equiv \int_{\RR^d} \kl(f_{G^\star}(\cdot\mid X), f_G(\cdot\mid X)) \, p(X)\mathrm dX,
\end{equation*}
where, for each fixed $X$, $f_{G^\star}(\cdot\mid X)$ and $f_G(\cdot\mid X)$ are Gaussian mixture densities with $K^\star$ components. The means and variances lie in compact sets (and the variances are bounded away from 0, so say they belong to $[v, V]$ with $0<v<V<\infty$). We now aim to show that there exists $\delta\in(0,1]$ (depending on the parameter space and $G^\star$, but not on $G$ nor $X$) such that
\begin{equation*}
    \int_{-\infty}^{+\infty} f_{G^\star}(Y\mid X) \left(\frac{f_{G^\star}(Y\mid X)}{f_{G}(Y\mid X)}\right)^\delta\mathrm dY <\infty
\end{equation*}
Because of the compactness of the space of means and the Gaussianity of each component, there exists $Y_1>0$ such that, for all $Y>Y_1$, we have $\sup_{X\in\XX}f_{G^\star}(Y\mid X)\lesssim e^{-\frac{Y^2}{2V}}$. Similarly, there exists $Y_2<0$ such that, for all $Y<Y_2$ and $G$, we have $\sup_{X\in\XX}f_{G}(Y\mid X)\lesssim e^{-\frac{Y^2}{2v}}$, where the constant in the inequality depends only on the parameter space, but not $X$ nor $G$. Defining $Y_3:=\max_{i=1,2} |Y_i|$, we can then write
\begin{align*}
    \int_{-\infty}^{+\infty} & f_{G^\star}(Y\mid X) \left(\frac{f_{G^\star}(Y\mid X)}{f_{G}(Y\mid X)}\right)^\delta\mathrm dY \\
    & = \int_{|Y|\leq Y_3} f_{G^\star}(Y\mid X) \left(\frac{f_{G^\star}(Y\mid X)}{f_{G}(Y\mid X)}\right)^\delta\mathrm dY + \int_{|Y|>Y_3} f_{G^\star}(Y\mid X) \left(\frac{f_{G^\star}(Y\mid X)}{f_{G}(Y\mid X)}\right)^\delta\mathrm dY \\
    & \lesssim C \,+\, \int_{-\infty}^{+\infty} e^{-(\frac{1+\delta}{2V} - \frac{\delta}{2v})Y^2}\mathrm dY
\end{align*}
for some constant $C>0$, where the bound on the first addendum comes from compactness of the parameter space and truncation. The second addendum is then made finite by choosing $\delta\in(0,\min\{1, [V-v]^{-1}v\})$. Because these inequalities are uniform over values of $X$, we also obtain
\begin{equation*}
    \int_{\RR^d}\int_{-\infty}^{+\infty}  g_{G^\star}(Y, X) \left(\frac{g_{G^\star}(Y, X)}{g_{G}(Y, X)}\right)^\delta\mathrm dY\mathrm dX \leq M_\delta <\infty
\end{equation*}
where $M_\delta$ is constant across $G$.

The above result allows us to invoke Theorem 5 in \cite{wong1995probability}, which yields that, for all $G$ and $\varepsilon$ such that $d^2_H(g_G,g_{G^\star})\leq \varepsilon^2<\frac{1}{2}(1-e^{-1})$, we have $\kl(g_{G^\star}, g_G)\leq C\varepsilon^2\log(1/\varepsilon)$ for some $C>0$ depending on the parameter space. Now define the function $H(\varepsilon):=\varepsilon/\sqrt{2\log(1/\varepsilon)}$, which is such that $H^2(\varepsilon)\log(1/H(\varepsilon))\leq\varepsilon^2$ for all $\varepsilon>0$ small enough. Therefore, for $n$ large enough,
\begin{align*}
    B(G^\star, \varepsilon_n) &= \left\{ G\in \mathscr G_{K^\star} : \kl(g_{G^\star}, g_{G})\leq \varepsilon_n^2\right\} \\
    & \supseteq \left\{ G\in \mathscr G_{K^\star} : \kl(g_{G^\star}, g_{G})\leq H^2(\varepsilon_n)\log(1/H(\varepsilon_n))\right\}\\
    &\supseteq \left\{ G\in \mathscr G_{K^\star} : d_H^2(g_{G^\star}, g_{G})\lesssim H^2(\varepsilon_n)\right\}\\
    & \supseteq \left\{ G\in \mathscr G_{K^\star} : \| g_{G^\star} - g_{G}\|_{\mathscr L^1}\lesssim H^2(\varepsilon_n)\right\} \\
    & \supseteq \left\{ G\in \mathscr G_{K^\star} : \mathcal L(G^\star,G)\lesssim H^2(\varepsilon_n)\right\},
\end{align*}
where the last inclusion holds for any loss function
\begin{equation}\label{eq:l1_upperbound_voronoi}
    \mathcal L(G^\star,G)\gtrsim \| g_{G^\star} - g_{G}\|_{\mathscr L^1}.
\end{equation}
In that case, we get
\begin{equation*}
    \Pi(B(G^\star, \varepsilon_n)) \geq \Pi\left(\left\{G\in\mathscr G_{K^\star} : \mathcal L(G, G^\star)\lesssim H^2(\varepsilon_n)\right\}\right).
\end{equation*}
So it suffices to require
\begin{equation*}
    \Pi\left(\left\{G\in\mathscr G_{K^\star} : \mathcal L(G, G^\star)\lesssim H^2(\varepsilon_n)\right\}\right)\geq e^{-n\varepsilon_n^2/16},
\end{equation*}
or equivalently, given our choice $\varepsilon_n=M\sqrt{\log (n)/n}$,
\begin{equation*}
    \Pi\left(\left\{G\in\mathscr G_{K^\star} : \mathcal L(G, G^\star)\leq  \frac{DM^2}{n}\frac{\log n}{\log n - \log\log n - 2\log M}\right\}\right)\geq \left(\frac{1}{n}\right)^{M^2/16},
\end{equation*}
where $D$ is an explicit constant. Because the fraction is of constant order, we can simply update $D$ and require
\begin{equation*}
    \Pi\left(\left\{G\in\mathscr G_{K^\star} : \mathcal L(G, G^\star)\leq  \frac{DM^2}{n}\right\}\right)\geq \left(\frac{1}{n}\right)^{M^2/16}.
\end{equation*}
Moreover, choosing $M$ large enough and requiring
\begin{equation*}
    \Pi\left(\left\{G\in\mathscr G_{K^\star} : \mathcal L(G, G^\star)\leq  \frac{c}{n}\right\}\right)\geq \frac{1}{n},
\end{equation*}
for some $c>0$, we would obtain
\begin{equation*}
    \Pi\left(\left\{G\in\mathscr G_{K^\star} : \mathcal L(G, G^\star)\leq  \frac{DM^2}{n}\right\}\right)\geq\Pi\left(\left\{G\in\mathscr G_{K^\star} : \mathcal L(G, G^\star)\leq  \frac{c}{n}\right\}\right)\geq \frac{1}{n}\geq \left(\frac{1}{n}\right)^{M^2/16}.
\end{equation*}
Hence, replacing $t$ with $1/n$, it is sufficient to require
\begin{equation*}
    \Pi\left(\left\{G\in\mathscr G_{K^\star} : \mathcal L(G, G^\star)\leq  ct\right\}\right)\geq t
\end{equation*}
for some $c>0$ and all $t>0$ small enough.

We are left to show that \Eqref{eq:l1_upperbound_voronoi} holds by taking $\mathcal L=\mathcal L_1$ as defined in \Eqref{eq:voronoi_loss_1}. To that end, it suffices to show that the relation holds for $\| f_{G^\star}(\cdot\mid X) - f_{G}(\cdot\mid X)\|_{\mathscr L^1}$ for all $X$. For any $j=1,\ldots, K^\star$ and $i=1,\ldots,K^\star$, define
\begin{align*}
    \tilde p_j^\star(X):=\frac{\exp(\alpha_{0j}^\star)}{\sum_{\ell=1}^{K^\star}\exp(\alpha_{0\ell}^\star + X^\top\alpha_{1\ell}^\star)}, \quad & \tilde f_j^\star(Y\mid X):=\frac{1}{\sqrt{2\pi\sigma_{j}^{\star 2}}}\exp\left(-\frac{(Y-E(X,\beta_j^\star))^2}{2\sigma_j^{\star 2}} + X^\top\alpha_{1j}^\star\right)\\
    \tilde p_i(X):=\frac{\exp(\alpha_{0i})}{\sum_{\ell=1}^{K^\star}\exp(\alpha_{0\ell} + X^\top\alpha_{1\ell})}, \quad & \tilde f_i(Y\mid X):=\frac{1}{\sqrt{2\pi\sigma_{i}^2}}\exp\left(-\frac{(Y-E(X,\beta_i))^2}{2\sigma_i^2} + X^\top\alpha_{1i}\right),
\end{align*}
so that
\begin{align*}
    \| & f_{G^\star}(\cdot\mid X) - f_{G}(\cdot\mid X)\|_{\mathscr L^1} = \int_\RR\left\vert\sum_{j=1}^{K^\star}\left(\tilde p_j^\star(X) \tilde f_j^\star(Y\mid X) - \sum_{i\in\mathcal C_j}\tilde p_i(X)\tilde f_i(Y\mid X)\right)\right\vert\mathrm dY\\
    & \leq \sum_{j=1}^{K^\star}\left(\int_\RR\tilde f_j^\star(Y\mid X)\mathrm dY \right)\left\vert\tilde p_j^\star(X)-\sum_{i\in\mathcal C_j}\tilde p_i(X)\right\vert \,\,+\,\,\sum_{j=1}^{K^\star}\sum_{i\in\mathcal C_j}\tilde p_i(X)\int_\RR\left\vert \tilde f_j^\star(Y\mid X) - \tilde f_i(Y\mid X)\right\vert\mathrm dY\\
    & \lesssim \sum_{j=1}^{K^\star}\left\vert\tilde p_j^\star(X)-\sum_{i\in\mathcal C_j}\tilde p_i(X)\right\vert \,\,+\,\, \sum_{j=1}^{K^\star}\sum_{i\in\mathcal{C}_j}\exp(\alpha_{0i})\Big(\|\alpha_{1i}-\alpha^\star_{1j}\|+\|\beta_{i}-\beta^\star_{j}\|+|\sigma^2_{i}-\sigma^{\star2}_{j}|\Big),
\end{align*}
where the last inequality follows from the boundedness of $1/\sum_{\ell=1}^{K^\star}\exp(\alpha_{0\ell} + X^\top\alpha_{1\ell})$ and a first order Taylor expansion of $\int_\RR\left\vert \tilde f_j^\star(Y\mid X) - \tilde f_i(Y\mid X)\right\vert\mathrm dY$ around $(\alpha_{1i},\beta_i,\sigma_i^2)=(\alpha_{1j}^\star, \beta_j^\star,\sigma_j^{\star 2})$. As for the first term, the triangle inequality yields
\begin{align*}
    &\sum_{j=1}^{K^\star}\left\vert\tilde p_j^\star(X)-\sum_{i\in\mathcal C_j}\tilde p_i(X)\right\vert  \lesssim \sum_{j=1}^{K^\star}\left\vert\exp(\alpha_{0j}^\star)-\sum_{i\in\mathcal C_j}\exp(\alpha_{0i})\right\vert \\
    & + \sum_{j=1}^{K^\star}\sum_{i\in\mathcal{C}_j}\exp(\alpha_{0i})\left|\frac{1}{\sum_{\ell=1}^{K^\star}\exp(\alpha_{0\ell}^\star + X^\top\alpha_{1\ell}^\star)} -\frac{1}{\sum_{\ell=1}^{K^\star}\exp(\alpha_{0\ell} + X^\top\alpha_{1\ell})}\right|
\end{align*}
As for the last term above, using the Lipschitz property of $t\mapsto 1/t$ over a domain bounded away from 0 and an appropriate triangle inequality replacing $X^\top\alpha_{1\ell}$ with $X^\top\alpha_{1\ell}^\star)$, one shows that is it is less than (up to a multiplicative constant)
\begin{equation*}
    \sum_{j=1}^{K^\star}\left\vert\exp(\alpha_{0j}^\star)-\sum_{i\in\mathcal C_j}\exp(\alpha_{0i})\right\vert + \sum_{j=1}^{K^\star}\sum_{i\in\mathcal{C}_j}\exp(\alpha_{0i})\|\alpha_{1i}-\alpha^\star_{1j}\|.
\end{equation*}
Combining all of the above, we obtain
\begin{align*}
    \| f_{G^\star}(\cdot\mid X) - f_{G}(\cdot\mid X)\|_{\mathscr L^1}& \lesssim\sum_{j=1}^{K^\star}\Bigg|\sum_{i\in\mathcal{C}_j}\exp(\alpha_{0i})-\exp(\alpha^\star_{0j})\Bigg|\\
    & +\sum_{j=1}^{K^\star}\sum_{i\in\mathcal{C}_j}\exp(\alpha_{0i})\Big(\|\alpha_{1i}-\alpha^\star_{1j}\|+\|\beta_{i}-\beta^\star_{j}\|+|\sigma^2_{i}-\sigma^{\star2}_{j}|\Big),
\end{align*}
as desired.

\subsection{Proof of Theorem~\ref{theorem:exact_parameter_rate}}
\label{appendix:exact_parameter_rate}

    \textbf{Overview.} Following from the result of Theorem~\ref{thm:our_abstract_rate_SMoGE}, it suffices to prove that
    \begin{align}
        d_H(g_{G},g_{G^\star})\gtrsim \mathcal{L}_{1}(G,G^\star),
    \end{align}
    for any mixing measure $G\in\mathscr G_{K^\star}$. For that purpose, we first demonstrate that 
    \begin{align}
        \label{eq:local_part}
        \lim_{\varepsilon\searrow 0}\inf_{G\in\mathscr G_{K^\star}:\mathcal{L}_{1}(G,G^\star)\leq\varepsilon}\frac{d_H(g_{G},g_{G^\star})}{\mathcal{L}_{1}(G,G^\star)}>0.
    \end{align}
    Assume that the above equation holds for now. Then, there exists a positive constant $\varepsilon'$ such that $\inf_{G\in\mathscr G_{K^\star}:\mathcal{L}_{1}(G,G^\star)\leq\varepsilon'}d_H(g_{G},g_{G^\star})/\mathcal{L}_{1}(G,G^\star)>0.$
    % \begin{align*}
    %     \inf_{G\in\mathscr G_{K^\star}:\mathcal{L}_{1}(G,G^\star)\leq\varepsilon'}\frac{d_H(g_{G},g_{G^\star})}{\mathcal{L}_{1}(G,G^\star)}>0.
    % \end{align*}
    As a consequence, we can complete the proof by showing that
    \begin{align}
        \label{eq:global_part}
        \inf_{G\in\mathscr G_{K^\star}:\mathcal{L}_{1}(G,G^\star)>\varepsilon'}\frac{d_H(g_{G},g_{G^\star})}{\mathcal{L}_{1}(G,G^\star)}>0.
    \end{align}
    Given these arguments, we will provide the proofs for \Eqref{eq:local_part} and \Eqref{eq:global_part}, respectively, in the sequel.

    \textbf{Proof of \Eqref{eq:local_part}.} Note that the Hellinger distance is bounded below by the $\mathscr L^1$-norm, that is, $d_H(g,g')\gtrsim\|g-g'\|_{\mathscr L^1}$. Therefore, it is sufficient to show that
    \begin{align}
        \label{eq:local_part_TV}
        \lim_{\varepsilon\searrow0}\inf_{G\in\mathscr G_{K^\star}:\mathcal{L}_{1}(G,G^\star)\leq\varepsilon}\frac{\|g_{G}-g_{G^\star}\|_{\mathscr L^1}}{\mathcal{L}_{1}(G,G^\star)}>0.
    \end{align}
    We will use a proof-by-contradiction method to prove this result. In particular, assume that \Eqref{eq:local_part_TV} does not hold true. Then, there exists a sequence of mixing measures $(G_n)$ of the form $G_n=\sum_{i=1}^{K^\star}\exp(\azni)\delta_{(\aoni,\bni,\sni)}$ that satisfies $\mathcal{L}_{1}(G_n,G^\star)\to0$ and
    \begin{align}
        \|g_{G_n}-g_{G^\star}\|_{\mathscr L^1}/\mathcal{L}_{1}(G_n,G^\star)\to0,
    \end{align}
    as $n\to\infty$. 
    % Since $K_n\leq K$ for all n, there exists a subsequence of $G_n$ such that $K_n$ does not change with $n$. Therefore, up to replacing $G_n$ by this subsequence, we may assume that $K_n=K'\leq K$ for all $n$. 
    Note that there are only a finite number of distinct sets $\mathcal{C}^n_1\times\ldots\times\mathcal{C}^n_{K^\star}$ over the range of $n\in\mathbb{N}$. Thus, up to replacing $G_n$ by its subsequence, we may assume without loss of generality that $\mathcal{C}_j=\mathcal{C}^n_j$ does not change with $n$ for all $j\in[K^\star]$. Additionally, since $K^*$ is known under the exactly-specified setting and $\mathcal{L}_1(G_n,G^\star)\to0$ as $n\to\infty$,
    each Voronoi cell $\mathcal{C}_j$ has only one element for any $j\in[K^*]$. Without loss of generality, we assume that $\mathcal{C}_j=\{j\}$ for simplicity.
    Thus, the Voronoi loss $\mathcal{L}_{1n}:=\mathcal{L}_{1}(G_n,G^\star)$ can be written as
    \begin{align}
        \mathcal{L}_{1n}&=\sum_{j=1}^{K^\star}\Big|\exp(\aznj)-\exp(\alpha^\star_{0j})\Big|\nonumber\\
    &+\sum_{j=1}^{K^\star}\exp(\aznj)(\|\aonj-\alpha^\star_{1j}\|+\|\bnj-\beta^\star_{1j}\|+|\snj-\sigma^{\star2}_{j}|),
    \end{align}
    Since $\mathcal{L}_{1n}\to0$ as $n\to\infty$, we have $\exp(\aznj)\to\exp(\alpha^\star_{0j})$ and $(\aonj,\bnj,\snj)\to(\aosj,\bsj,\ssj)$ as $n\to\infty$, for all $j\in[K^\star]$. Next, we divide the rest of this proof into three main steps.

    \textbf{Step 1 - Density Decomposition.} In this step, we decompose the density discrepancy $g_{G_n}(Y,X)-g_{G^\star}(Y,X)=[f_{G_n}(Y\mid X)-f_{G^\star}(Y\mid X)]p(X)$ into a combination of linearly independent terms. In particular, let us denote 
    \begin{align*}
        D_n&:=\left[\sum_{j=1}^{K^\star}\exp(\azsj+X^{\top}(\aosj))\right]\cdot[g_{G_n}(Y,X)-g_{G^\star}(Y,X)],\\
        F(Y\mid X;\alpha_1,\beta,\sigma^2)&:=\exp(X^{\top}\alpha_1)\mathcal{N}(Y\mid E(X,\beta),\sigma^2),\\
        H_n(Y\mid X;\alpha_1)&:=\exp(X^{\top}\alpha_1)f_{G_n}(Y\mid X).
    \end{align*}
    Then, we can decompose $D_n$ as $D_n=A_n-B_n+C_n$, where
    \begin{align*}
        A_n&:=\sum_{j=1}^{K^\star}\exp(\aznj)[F(Y\mid X;\aonj,\bnj,\snj)-F(Y\mid X;\aosj,\bsj,\ssj)]p(X),\\
        B_n&:=\sum_{j=1}^{K^\star}\exp(\aznj)[H_n(Y\mid X;\aonj)-H_n(Y\mid X;\aosj)]p(X),\\
        C_n&:=\sum_{j=1}^{K^\star}\left[\exp(\aonj)-\exp(\aosj)\right][F(Y\mid X;\aosj,\bsj,\ssj)-H_n(Y\mid X;\aosj)]p(X).
    \end{align*}
    Next, we will further decompose $A_n$ and $B_n$. 
    % Firstly, we rewrite $A_n$ as $A_n=A_{n,1}+A_{n,2}$, where
    % \begin{align*}
    %     A_{n,1}&=\sum_{j\in[K^\star]:|\mathcal{C}_j|=1}\sum_{i\in\mathcal{C}_j}\exp(\aznj)[F(Y\mid X;\aonj,\bnj,\snj)-F(Y\mid X;\aosj,\bzsj,\bsj,\ssj)]p(X),\\
    %     A_{n,2}&=\sum_{j\in[K^\star]:|\mathcal{C}_j|>1}\sum_{i\in\mathcal{C}_j}\exp(\aznj)[F(Y\mid X;\aonj,\bnj,\snj)-F(Y\mid X;\aosj,\bzsj,\bsj,\ssj)]p(X).
    % \end{align*}
    Let $\daonj:=\aonj-\aosj$, $\dbnj:=\bnj-\bsj$, and $\dsnj:=\snj-\ssj$. By means of a first-order Taylor expansion, we have
    \begin{align*}
        A_{n}&=\sum_{j=1}^{K^\star}\exp(\aznj)\sum_{\rho:|\rho|=1}\frac{1}{\rho!}(\daonj)^{\rho_1}(\dbnj)^{\rho_2}(\dsnj)^{\rho_3}\\
        &\hspace{3cm}\times \frac{\partial^{|\rho|}F}{\partial\alpha_1^{\rho_1}\partial\beta^{\rho_2}\partial(\sigma^2)^{\rho_3}}(Y\mid E(X,\bsj),\ssj)p(X)+R_{n,1}(Y\mid X)p(X).
    \end{align*}
    Above, we denote $\rho=(\rho_1,\rho_2,\rho_3)\in\mathbb{N}^{d}\times\mathbb{N}^{p}\times\mathbb{N}$. Additionally, for any vectors $a=(a_1,\ldots,a_{\bar{d}})\in\mathbb{R}^{\bar{d}}$ and $b=(b_1,\ldots,b_{\bar{d}})\in\mathbb{N}^{\bar{d}}$, we let $a^{b}=a_1^{b_1}\ldots a_{\bar{d}}^{b_{\bar{d}}}$, $|a|=a_1+\ldots+a_{\bar{d}}$, and $b!=b_1!\ldots b_{\bar{d}}!$. Lastly, $R_{n,1}(Y\mid X)$is a Taylor remainder such that $R_{n,1}(Y\mid X)/\mathcal{L}_{1n}\to0$ as $n\to\infty$. By taking the first-order partial derivatives of $F(Y\mid X;\alpha_1,\beta,\sigma^2)$ with respect to its parameters, we have
    \begin{align}
        \frac{\partial F}{\partial\alpha_1^{(u)}}(Y\mid X;\alpha_1,\beta,\sigma^2)&=X^{(u)}\exp(X^{\top}\alpha_1)\mathcal{N}(Y\mid E(X,\beta),\sigma^2),\nonumber\\
        \frac{\partial F}{\partial\beta^{(u')}}(Y\mid X;\alpha_1,\beta,\sigma^2)&=\frac{\partial E}{\partial\beta^{(u')}}(X,\beta)\exp(X^{\top}\alpha_1)\frac{\partial\mathcal{N}}{\partial E}(Y\mid E(X,\beta),\sigma^2),\nonumber\\
        \label{eq:first-order-F}
        \frac{\partial F}{\partial\sigma^2}(Y\mid X;\alpha_1,\beta,\sigma^2)&=\frac{1}{2}\exp(X^{\top}\alpha_1)\frac{\partial^2\mathcal{N}}{\partial E^2}(Y\mid E(X,\beta),\sigma^2),
    \end{align}
    for all $u\in[d]$ and $u'\in[d']$. 
Based on these derivatives, we can rewrite $A_{n,1}$ as
\begin{align*}
    A_{n}&=\sum_{j=1}^{K^\star}\sum_{\gamma=0}^{2}A^{(j)}_{n,\gamma}\frac{\partial^\gamma\mathcal{N}}{\partial E^\gamma}(Y\mid E(X,\beta),\sigma^2)p(X)+R_{n,1}(Y\mid X)p(X),
\end{align*}
where 
\begin{align*}
    A^{(j)}_{n,0}&:=\exp(\aznj)\sum_{u=1}^{d}(\dazni)^{(u)}X^{(u)}\exp(X^{\top}\aosj),\\
    A^{(j)}_{n,1}&:=\exp(\aznj)\sum_{u'=1}^{d'}(\dbnj)^{(u')}\frac{\partial E}{\partial\beta^{(u')}}(X,\bsj)\exp(X^{\top}\aosj),\\
    A^{(j)}_{n,2}&:=\exp(\aznj)\frac{1}{2}(\dsnj)\exp(X^{\top}\aosj),
\end{align*}
for all $j\in[K^\star]$. 
Next, by applying a first-order Taylor expansion, we rewrite $B_n$ as
\begin{align*}
    B_{n}&=\sum_{j=1}^{K^\star}\exp(\aznj)\sum_{u=1}^{d}(\daonj)^{(u)}X^{(u)}H_n(Y\mid X;\aosj)+R_{n,2}(Y\mid X),
\end{align*}
where $R_{n,2}(Y\mid X)$ is a Taylor remainder such that $R_{n,2}(Y\mid X)/\mathcal{L}_{1n}\to0$ as $n\to\infty$.

\textbf{Step 2 - Non-vanishing coefficients.} In this step, we prove by contradiction that at least one among the coefficients in the representations of $[A_{n}-R_{n,1}(Y\mid X)]p(X)/\mathcal{L}_{1n}$, $[B_{n}-R_{n,2}(Y\mid X)]p(X)/\mathcal{L}_{1n}$, $C_{n}/\mathcal{L}_{1n}$ does not go to zero as $n\to\infty$. Assume that all these coefficients converge to zero. From the coefficients of the terms:
\begin{itemize}
    \item $\exp(X^{\top}\aosj)\mathcal{N}(Y|X;\aosj,\bsj,\ssj)$ for $j\in[K^\star]$, we get
    \begin{align*}
        \frac{1}{\mathcal{L}_{1n}}\cdot\sum_{j=1}^{K^\star}\Big|\exp(\aznj)-\exp(\aosj)\Big|\to0;
    \end{align*}
    \item $X^{(u)}\exp(X^{\top}\aosj)\mathcal{N}(Y|X;\aosj,\bsj,\ssj)$ for $j\in[K^\star]$ and $u\in[d]$, we get
    \begin{align*}
        \frac{1}{\mathcal{L}_{1n}}\cdot\sum_{j=1}^{K^\star}\exp(\aznj)\|\daonj\|\to0;
    \end{align*}
    \item $\frac{\partial E}{\partial\beta^{(u')}}\exp(X^{\top}\aosj)\frac{\partial \mathcal{N}}{\partial E}(Y|X;\aosj,\bsj,\ssj)$ for $j\in[K^\star]$ and $u'\in[d']$, we get
    \begin{align*}
        \frac{1}{\mathcal{L}_{1n}}\cdot\sum_{j=1}^{K^\star}\exp(\aznj)\|\dbnj\|\to0;
    \end{align*}
    \item $\exp(X^{\top}\aosj)\frac{\partial^2 \mathcal{N}}{\partial E^2}(Y|X;\aosj,\bsj,\ssj)$ for $j\in[K^\star]$, we get
    \begin{align*}
        \frac{1}{\mathcal{L}_{1n}}\cdot\sum_{j=1}^{K^\star}\exp(\aznj)|\dsnj|\to0;
    \end{align*}
    % \item $X^{(u)}X^{(v)}\exp(X^{\top}\aosj)\mathcal{N}(Y|X;\aosj,\bsj,\ssj)$ for $j\in[K^\star]:|\mathcal{C}_{j}|>1$ and $u,v\in[d]$, we get
    % \begin{align*}
    %     \frac{1}{\mathcal{L}_{1n}}\cdot\sum_{j\in[K^\star]:|\mathcal{C}_{j}|>1}\exp(\aznj)\|\daonj\|^2\to0;
    % \end{align*}
    % \item $\big[\frac{\partial^2 E}{\partial\beta^{(u')}}(X,\bsj)\big]^2\exp(X^{\top}\aosj)\frac{\partial^2 \mathcal{N}}{\partial E^2}(Y|X;\aosj,\bsj,\ssj)$ for $j\in[K^\star]:|\mathcal{C}_{j}|>1$ and $u'\in[d']$, we get
    % \begin{align*}
    %     \frac{1}{\mathcal{L}_{1n}}\cdot\sum_{j\in[K^\star]:|\mathcal{C}_{j}|>1}\exp(\aznj)\|\dbnj\|\to0;
    % \end{align*}
    % \item $\exp(X^{\top}\aosj)\frac{\partial^4 \mathcal{N}}{\partial E^4}(Y|X;\aosj,\bsj,\ssj)$ for $j\in[K^\star]:|\mathcal{C}_{j}|>1$, we get
    % \begin{align*}
    %     \frac{1}{\mathcal{L}_{1n}}\cdot\sum_{j\in[K^\star]:|\mathcal{C}_{j}|>1}\exp(\aznj)|\dsnj|^2\to0;
    % \end{align*}
\end{itemize}
By taking the sum of these limits, we have $1=\frac{1}{\mathcal{L}_{1n}}\cdot\mathcal{L}_{1n}\to0$ as $n\to\infty$, which is a contradiction. Therefore, not all the coefficients in the representations of $[A_{n}-R_{n,1}(Y\mid X)]p(X)/\mathcal{L}_{1n}$, $[B_{n}-R_{n,2}(Y\mid X)]p(X)/\mathcal{L}_{1n}$, $C_{n}/\mathcal{L}_{1n}$ converge to zero as $n\to\infty$.

\textbf{Step 3 - Fatou's argument.} In this step, we will point out a contradiction to the results of Step 2. To begin, we denote by $m_n$ the maximum of the absolute values of the coefficients in the representations of $[A_{n}-R_{n,1}(Y\mid X)]p(X)/\mathcal{L}_{1n}$, $[B_{n}-R_{n,2}(Y\mid X)]p(X)/\mathcal{L}_{1n}$, $C_{n}/\mathcal{L}_{1n}$. Then, the results of Step 2 indicate that $1/m_n\not\to\infty$ as $n\to\infty$. Next, let us denote
\begin{align*}
    \frac{1}{m_n\mathcal{L}_{1n}}\cdot\Big(\exp(\aznj)-\exp(\azsj)\Big)\to t_{0,j},& \quad \frac{1}{m_n\mathcal{L}_{1n}}\cdot\exp(\aznj)(\daonj)^{(u)}\to t^{(u)}_{1,j},\\
    \frac{1}{m_n\mathcal{L}_{1n}}\cdot\exp(\aznj)(\dbnj)^{(u')}\to t^{(u')}_{2,j},& \quad \frac{1}{m_n\mathcal{L}_{1n}}\cdot\exp(\aznj)(\dsnj)\to t_{3,j},
\end{align*}
for all $j\in[K^\star]$ as $n\to\infty$. It follows from Step 2 that not all these limits are zero. Next, recall that we have $\|g_{G_n}-g_{G^\star}\|_{\mathscr L^1}/\mathcal{L}_{1}(G_n,G^\star)\to0$ as $n\to\infty$. Then, by means of the Fatou's lemma, we have
\begin{align*}
    \lim_{n\to\infty}\frac{\|g_{G_n}-g_{G^\star}\|_{\mathscr L^1}}{m_n\mathcal{L}_{1n}}\geq\int\liminf_{n\to\infty}\frac{|g_{G_n}(Y,X)-g_{G^\star}(Y,X)|}{m_n\mathcal{L}_{1n}}\mathrm d(Y,X),
\end{align*}
which implies that $[g_{G_n}(Y,X)-g_{G^\star}(Y,X)]/[m_n\mathcal{L}_{1n}]\to0$ as $n\to\infty$ along a subsequence (focus on such a subsequence from now on). Since the term $\sum_{j=1}^{K^\star}\exp(\azsj+X^{\top}\aosj)$ is bounded, we deduce $D_n/[m_n\mathcal{L}_{1n}]\to0$ as $n\to\infty$, for almost every $(Y,X)$. From the decomposition of $D_n$, it follows that
\begin{align}
    \label{eq:zero_limit}
    \frac{1}{m_n\mathcal{L}_{1n}}\cdot(A_{n}-B_{n}+C_{n})\to0.
\end{align}
For ease of presentation, let us denote $F_{\gamma,j}:=\exp(X^{\top}\aosj)\frac{\partial^{\gamma}\mathcal{N}}{\partial E^{\gamma}}(Y\mid X;\aosj,\bsj,\ssj)$ and $H_{j}(Y\mid X):=\lim_{n\to\infty}H_{n}(Y\mid X;\aosj)$, for all $0\leq\gamma\leq 2$ and $j\in[K^\star]$. Then, we have
\begin{align*}
    &\lim_{n\to\infty}\frac{A_{n}}{m_n\mathcal{L}_{1n}}=\sum_{j=1}^{K^\star}\Big[\sum_{u=1}^{d}t_{1,j}^{(u)}X^{(u)}F_{0,j}(Y|X)+\sum_{u'=1}^{d'}t_{2,j}^{(u')}\frac{\partial E}{\partial\beta^{(u')}}(X,\bsj)F_{1,j}(Y|X)\\
    &\hspace{9cm}+\frac{1}{2}t_{3,j}F_{2,j}(Y|X)\Big]p(X),\\
    &\lim_{n\to\infty}\frac{B_{n}}{m_n\mathcal{L}_{1n}}=\sum_{j=1}^{K^\star}\sum_{u=1}^{d}t_{1,j}^{(u)}X^{(u)}H_{j}(Y|X)p(X),\\
    &\lim_{n\to\infty}\frac{C_{n}}{m_n\mathcal{L}_{1n}}=\sum_{j=1}^{K^\star}t_{0,j}[F_{0,j}(Y|X)-H_{j}(Y|X)]p(X).
\end{align*}
Note that for almost every $X$, the set $\Big\{F_{\gamma,j}(Y|X), \ H_{j}(Y|X):0\leq\gamma\leq 2, \  j\in[K^\star]\Big\}$
is linearly independent with respect to $Y$. Therefore, it follows that the coefficients of these terms in the limit in \Eqref{eq:zero_limit} become zero. 

For $j\in[K^\star]$, by considering the coefficients of 
\begin{itemize}
    \item $F_{0,j}(Y|X)$, we have $t_{0,j}+\sum_{u=1}^{d}t_{1,j}^{(u)}X^{(u)}=0$, for almost every $X$. Then, we deduce $t_{0,j}=t_{1,j}^{(u)}=0$ for all $u\in[d]$;
    \item $F_{1,j}(Y|X)$, we have $\sum_{u'=1}^{d'}t_{2,j}^{(u')}\frac{\partial E}{\partial\beta^{(u')}}(X,\bsj)$, for almost every $X$. As the expert function $E$ is first-order strongly identifiable, we get $t_{2,j}^{(u')}=0$ for all $u'\in[d']$;
    \item $F_{2,j}(Y|X)$, we have $t_{3,j}=0$.
\end{itemize}
Putting the above results together, we have $t_{0,j}=t_{1,j}^{(u)}=t_{2,j}^{(u')}=t_{3,j}=0$ for all $j\in[K^\star]$, $u,v\in[d]$ and $u',v'\in[d']$. This contradicts to the fact that at least one among them is non-zero. Therefore, we achieve the result in \Eqref{eq:local_part}.

\textbf{Proof of \Eqref{eq:global_part}.} Assume by contrary that \Eqref{eq:global_part} does not hold. Then, we can find a sequence of mixing measure $(G_n)$ such that $\mathcal{L}_{1}(G_n,G^\star)>\varepsilon'$ and $\|g_{G_n}-g_{G^\star}\|_{\mathscr L^1}/\mathcal{L}_{1}(G_n,G^\star)\to0$ as $n\to\infty$. These two properties imply that 
\begin{align*}
    \|g_{G_n}-g_{G^\star}\|_{\mathscr L^1}\to0.
\end{align*}
Since the parameter space $\Theta$ is compact, we can substitute the sequence $(G_n)$ with its subsequence $(G'_n)$ that converges to some mixing measure $G'$. Recall that $\mathcal{L}_{1}(G_n,G^\star)>\varepsilon'$, then we also have $\mathcal{L}_{1}(G',G^\star)\geq\varepsilon'$. On the other hand, by the Fatou's lemma, we get
\begin{align*}
    0=\lim_{n\to\infty}\|g_{G'_n}-g_{G^\star}\|_{\mathscr L^1}&\geq\int\liminf_{n\to\infty}|g_{G'_n}(Y,X)-g_{G^\star}(Y,X)|\mathrm d(Y,X)\\
    &=\int|g_{G'}(Y,X)-g_{G^\star}(Y,X)|\mathrm d(Y,X).
\end{align*}
This inequality indicates that $g_{G'}(Y,X)=g_{G^\star}(Y,X)$ for almost every $(Y,X)$. Since the SMoGE model is identifiable, we deduce $G'\equiv G^\star$. As a result, we get $\mathcal{L}_{1}(G',G^\star)=0$, which contradicts the previous result that $\mathcal{L}_{1}(G',G^\star)>\varepsilon'>0$. Hence, the proof is completed.

\subsection{Proof of Theorem~\ref{theorem:parameter_rate}}
\label{appendix:parameter_rate}
    \textbf{Overview.} Following the result of Theorem~\ref{thm:density_estim_finite_prior} and using the same arguments for the proof of Theorem~\ref{theorem:exact_parameter_rate} in Appendix~\ref{appendix:exact_parameter_rate}, it suffices to establish the following inequality
    \begin{align}
        \label{eq:local_part_TV_over}
        \lim_{\varepsilon\searrow0}\inf_{G\in\mathscr O_{K}:\mathcal{L}_{2}(G,G^\star)\leq\varepsilon}\frac{\|g_{G}-g_{G^\star}\|_{\mathscr L^1}}{\mathcal{L}_{2}(G,G^\star)}>0.
    \end{align}
    Assume by contrary that \Eqref{eq:local_part_TV_over} is not true, that is, we can find a sequence $(G_n)$ such that $\mathcal{L}_{2}(G_n,G^\star)\to0$ and $\|g_{G_n}-g_{G^\star}\|_{\mathscr L^1}/\mathcal{L}_{2}(G_n,G^\star)\to0,$
    as $n\to\infty$. Since $K_n\leq K$ for all n, there exists a subsequence of $G_n$ such that $K_n$ does not change with $n$. Therefore, up to replacing $G_n$ by this subsequence, we may assume that $K_n=K'\leq K$ for all $n$. 
    Similarly, since there are only a finite number of distinct sets $\mathcal{C}^n_1\times\ldots\times\mathcal{C}^n_{K^\star}$ over the range of $n\in\mathbb{N}$, we may assume without loss of generality that $\mathcal{C}_j=\mathcal{C}^n_j$ does not change with $n$ for all $j\in[K^\star]$. Then, we can rewrite the Voronoi loss $\mathcal{L}_{2n}:=\mathcal{L}_{2}(G_n,G^\star)$ as
    \begin{align}
        \mathcal{L}_{2n}&=\sum_{j=1}^{K^\star}\Big|\sum_{i\in\mathcal{C}_j}\exp(\azni)-\exp(\alpha^\star_{0j})\Big|\nonumber\\
    &+\sum_{j\in[K^\star]:|\mathcal{C}_j|=1}\sum_{i\in\mathcal{C}_j}\exp(\azni)(\|\aoni-\alpha^\star_{1j}\|+\|\bni-\beta^\star_{1j}\|+|\sni-\sigma^{\star2}_{j}|)\nonumber\\
    &+\sum_{j\in[K^\star]:|\mathcal{C}_j|>1}\sum_{i\in\mathcal{C}_j}\exp(\azni)(\|\aoni-\alpha^\star_{1j}\|^2+\|\bni-\beta^\star_{1j}\|^2+|\sni-\sigma^{\star2}_{j}|^2),
    \end{align}
    Recall that we have $\mathcal{L}_{2n}\to0$ as $n\to\infty$, which implies $\sum_{i\in\mathcal{C}_j}\exp(\azni)\to\exp(\alpha^\star_{0j})$ and $(\aoni,\bni,\sni)\to(\aosj,\bsj,\ssj)$ as $n\to\infty$, for all $i\in\mathcal{C}_j$ and $j\in[K^\star]$. Subsequently, we separate the rest into three main steps.

    \textbf{Step 1 - Density Decomposition.} First, we decompose the density discrepancy $g_{G_n}(Y,X)-g_{G^\star}(Y,X)=[f_{G_n}(Y\mid X)-f_{G^\star}(Y\mid X)]p(X)$ into a combination of linearly independent terms through the quantity
    \begin{align*}
        D_n&:=\left[\sum_{j=1}^{K^\star}\exp(\azsj+X^{\top}(\aosj))\right]\cdot[g_{G_n}(Y,X)-g_{G^\star}(Y,X)].
    \end{align*}
    In particular, let $F(Y\mid X;\alpha_1,\beta,\sigma^2):=\exp(X^{\top}\alpha_1)\mathcal{N}(Y\mid E(X,\beta),\sigma^2)$ and $H_n(Y\mid X;\alpha_1):=\exp(X^{\top}\alpha_1)f_{G_n}(Y\mid X)$. Then, the quantity $D_n$ can be represented as $D_n=A_n-B_n+C_n$, where we define
    \begin{align*}
        A_n&:=\sum_{j=1}^{K^\star}\sum_{i\in\mathcal{C}_j}\exp(\azni)[F(Y\mid X;\aoni,\bni,\sni)-F(Y\mid X;\aosj,\bsj,\ssj)]p(X),\\
        B_n&:=\sum_{j=1}^{K^\star}\sum_{i\in\mathcal{C}_j}\exp(\azni)[H_n(Y\mid X;\aoni)-H_n(Y\mid X;\aosj)]p(X),\\
        C_n&:=\sum_{j=1}^{K^\star}\left[\sum_{i\in\mathcal{C}_j}\exp(\aoni)-\exp(\aosj)\right][F(Y\mid X;\aosj,\bsj,\ssj)-H_n(Y\mid X;\aosj)]p(X).
    \end{align*}
    Next, we rewrite $A_n$ as $A_n=A_{n,1}+A_{n,2}$, where
    \begin{align*}
        A_{n,1}&=\sum_{j\in[K^\star]:|\mathcal{C}_j|=1}\sum_{i\in\mathcal{C}_j}\exp(\azni)[F(Y\mid X;\aoni,\bni,\sni)-F(Y\mid X;\aosj,\bzsj,\bsj,\ssj)]p(X),\\
        A_{n,2}&=\sum_{j\in[K^\star]:|\mathcal{C}_j|>1}\sum_{i\in\mathcal{C}_j}\exp(\azni)[F(Y\mid X;\aoni,\bni,\sni)-F(Y\mid X;\aosj,\bzsj,\bsj,\ssj)]p(X).
    \end{align*}
    Let us denote $\daoni:=\aoni-\aosj$, $\dbni:=\bni-\bsj$, and $\dsni:=\sni-\ssj$. By applying first-order and second-order Taylor expansions, we get
    \begin{align*}
        A_{n,1}&=\sum_{j\in[K^\star]:|\mathcal{C}_j|=1}\sum_{i\in\mathcal{C}_j}\exp(\azni)\sum_{\rho:|\rho|=1}\frac{1}{\rho!}(\daoni)^{\rho_1}(\dbni)^{\rho_2}(\dsni)^{\rho_3}\\
        &\hspace{3cm}\times \frac{\partial^{|\rho|}F}{\partial\alpha_1^{\rho_1}\partial\beta^{\rho_2}\partial(\sigma^2)^{\rho_3}}(Y\mid E(X,\bsj),\ssj)p(X)+R_{n,1}(Y\mid X)p(X),\\
        A_{n,2}&=\sum_{j\in[K^\star]:|\mathcal{C}_j|>1}\sum_{i\in\mathcal{C}_j}\exp(\azni)\sum_{\rho:|\rho|=2}\frac{1}{\rho!}(\daoni)^{\rho_1}(\dbni)^{\rho_2}(\dsni)^{\rho_3}\\
        &\hspace{3cm}\times \frac{\partial^{|\rho|}F}{\partial\alpha_1^{\rho_1}\partial\beta^{\rho_2}\partial(\sigma^2)^{\rho_3}}(Y\mid E(X,\bsj),\ssj)p(X)+R_{n,2}(Y\mid X)p(X),
    \end{align*}
    Above, we denote $\rho=(\rho_1,\rho_2,\rho_3)\in\mathbb{N}^{d}\times\mathbb{N}^{p}\times\mathbb{N}$. Additionally, for any vectors $a=(a_1,\ldots,a_{\bar{d}})\in\mathbb{R}^{\bar{d}}$ and $b=(b_1,\ldots,b_{\bar{d}})\in\mathbb{N}^{\bar{d}}$, we let $a^{b}=a_1^{b_1}\ldots a_{\bar{d}}^{b_{\bar{d}}}$, $|a|=a_1+\ldots+a_{\bar{d}}$, and $b!=b_1!\ldots b_{\bar{d}}!$. Lastly, $R_{n,\ell}(Y\mid X)$, for $\ell\in\{1,2\}$, are Taylor remainders such that $R_{n,\ell}(Y\mid X)/\mathcal{L}_{2n}\to0$ as $n\to\infty$. Recall that the first-order partial derivatives of $F(Y\mid X;\alpha_1,\beta,\sigma^2)$ with respect to its parameters have been provided in \Eqref{eq:first-order-F}. 
    % \begin{align*}
    %     \frac{\partial F}{\partial\alpha_1^{(u)}}(Y\mid X;\alpha_1,\beta,\sigma^2)&=X^{(u)}\exp(X^{\top}\alpha_1)\mathcal{N}(Y\mid E(X,\beta),\sigma^2),\\
    %     \frac{\partial F}{\partial\beta^{(u')}}(Y\mid X;\alpha_1,\beta,\sigma^2)&=\frac{\partial E}{\partial\beta^{(u')}}(X,\beta)\exp(X^{\top}\alpha_1)\frac{\partial\mathcal{N}}{\partial E}(Y\mid E(X,\beta),\sigma^2),\\
    %     \frac{\partial F}{\partial\sigma^2}(Y\mid X;\alpha_1,\beta,\sigma^2)&=\frac{1}{2}\exp(X^{\top}\alpha_1)\frac{\partial^2\mathcal{N}}{\partial E^2}(Y\mid E(X,\beta),\sigma^2),
    % \end{align*}
    % for all $u\in[d]$ and $u'\in[d']$. 
    Meanwhile, the second-order partial derivatives of $F(Y\mid X;\alpha_1,\beta,\sigma^2)$ with respect to its parameters are given by
    \begin{align*}
    \frac{\partial^2 F}{\partial\alpha_1^{(u)}\partial\alpha_1^{(v)}}(Y\mid X;\alpha_1,\beta,\sigma^2)&=X^{(u)}X^{(v)}\exp(X^{\top}\alpha_1)\mathcal{N}(Y\mid E(X,\beta),\sigma^2),\\
    \frac{\partial^2 F}{\partial\beta^{(u')}\partial\beta^{(v')}}(Y\mid X;\alpha_1,\beta,\sigma^2)&=\frac{\partial^2 E}{\partial\beta^{(u')}\partial\beta^{(v')}}(X,\beta)\exp(X^{\top}\alpha_1)\frac{\partial\mathcal{N}}{\partial E}(Y\mid E(X,\beta),\sigma^2)\\
    &+\frac{\partial E}{\partial\beta^{(u')}}(X,\beta)\frac{\partial E}{\partial\beta^{(v')}}(X,\beta)\exp(X^{\top}\alpha_1)\frac{\partial^2\mathcal{N}}{\partial E^2}(Y\mid E(X,\beta),\sigma^2),\\
    \frac{\partial^2 F}{\partial\nu^2}(Y\mid X;\alpha_1,\beta,\sigma^2)&=\frac{1}{4}\exp(X^{\top}\alpha_1)\frac{\partial^4\mathcal{N}}{\partial E^4}(Y\mid E(X,\beta),\sigma^2),
\end{align*}
and
\begin{align*}
    \frac{\partial^2 F}{\partial\alpha_1^{(u)}\partial\beta^{(v')}}(Y\mid X;\alpha_1,\beta,\sigma^2)&=X^{(u)}\frac{\partial E}{\partial\beta^{(v')}}(X,\beta)\exp(X^{\top}\alpha_1)\frac{\partial\mathcal{N}}{\partial E}(Y\mid E(X,\beta),\sigma^2),\\
    \frac{\partial^2 F}{\partial\alpha_1^{(u)}\partial\nu}(Y\mid X;\alpha_1,\beta,\sigma^2)&=\frac{1}{2}X^{(u)}\exp(X^{\top}\alpha_1)\frac{\partial^2\mathcal{N}}{\partial E^2}(Y\mid E(X,\beta),\sigma^2),\\
    \frac{\partial^2 F}{\partial\beta^{(u')}\partial\nu}(Y\mid X;\alpha_1,\beta,\sigma^2)&=\frac{1}{2}\frac{\partial E}{\partial\beta^{(u')}}(X,\beta)\exp(X^{\top}\alpha_1)\frac{\partial^3\mathcal{N}}{\partial E^3}(Y\mid E(X,\beta),\sigma^2),
\end{align*}
for all $i\in[d]$ and $u'\in[d']$. From these derivatives, $A_{n,1}$ can be written as
\begin{align*}
    A_{n,1}&=\sum_{j\in[K^\star]:|\mathcal{C}_j|=1}\sum_{\gamma=0}^{2}A^{(j)}_{n,1,\gamma}\frac{\partial^\gamma\mathcal{N}}{\partial E^\gamma}(Y\mid E(X,\beta),\sigma^2)p(X)+R_{n,1}(Y\mid X)p(X),
\end{align*}
where 
\begin{align*}
    A^{(j)}_{n,1,0}&:=\sum_{i\in\mathcal{C}_j}\exp(\azni)\sum_{u=1}^{d}(\dazni)^{(u)}X^{(u)}\exp(X^{\top}\aosj),\\
    A^{(j)}_{n,1,1}&:=\sum_{i\in\mathcal{C}_j}\exp(\azni)\sum_{u'=1}^{d'}(\dbni)^{(u')}\frac{\partial E}{\partial\beta^{(u')}}(X,\bsj)\exp(X^{\top}\aosj),\\
    A^{(j)}_{n,1,2}&:=\sum_{i\in\mathcal{C}_j}\exp(\azni)\frac{1}{2}(\dsni)\exp(X^{\top}\aosj),
\end{align*}
for all $j\in[K^\star]:|\mathcal{C}_j|=1$. Analogously, we can write $A_{n,2}$ as
\begin{align*}
    A_{n,2}&=\sum_{j\in[K^\star]:|\mathcal{C}_j|>1}\sum_{\gamma=0}^{4}A^{(j)}_{n,1,\gamma}\frac{\partial^\gamma\mathcal{N}}{\partial E^\gamma}(Y\mid E(X,\beta),\sigma^2)p(X)+R_{n,2}(Y\mid X)p(X),
\end{align*}
where
\begin{align*}
    A^{(j)}_{n,2,0}&:=\sum_{i\in\mathcal{C}_j}\exp(\azni)\Bigg[\sum_{u=1}^{d}(\daoni)^{(u)}X^{(u)}+\sum_{u,v=1}^{d}\frac{(\daoni)^{(u)}(\daoni)^{(v)}}{1+1_{\{u=v\}}}X^{(u)}X^{(v)}\Bigg]\exp(X^{\top}\aosj),\\
    A^{(j)}_{n,2,1}&:=\sum_{i\in\mathcal{C}_j}\exp(\azni)\Bigg[\sum_{u'=1}^{d'}(\dbni)^{(u')}\frac{\partial E}{\partial\beta^{(u')}}(X,\bsj)\\
    &\hspace{4cm}+\sum_{u',v'=1}^{d'}\frac{(\dbni)^{(u')}(\dbni)^{(v')}}{1+1_{\{u'=v'\}}}\frac{\partial^2E}{\partial\beta^{(u')}\partial\beta^{(v')}}(X,\bsj)\\
    &\hspace{4cm}+\sum_{u=1}^{d}\sum_{v'=1}^{d'}(\daoni)^{(u)}(\dbni)^{(v')}X^{(u)}\frac{\partial E}{\partial\beta^{(v')}}(X,\bsj)\Bigg]\exp(X^{\top}\aosj),\\
    A^{(j)}_{n,2,2}&:=\sum_{i\in\mathcal{C}_j}\exp(\azni)\Bigg[\frac{1}{2}(\dsni)+\sum_{u',v'=1}^{d'}\frac{(\dbni)^{(u')}(\dbni)^{(v')}}{1+1_{\{u'=v'\}}}\frac{\partial E}{\partial\beta^{(u')}}(X,\bsj)\frac{\partial E}{\partial\beta^{(v')}}(X,\bsj)\\
    &\hspace{4cm}+\sum_{u=1}^{d}\frac{1}{2}(\daoni)^{(u)}(\dsni)X^{(u)}\Bigg]\exp(X^{\top}\aosj),\\
    A^{(j)}_{n,2,3}&:=\sum_{i\in\mathcal{C}_j}\exp(\azni)\sum_{u'=1}^{d'}\frac{1}{2}(\dbni)^{(u')}(\dsni)\frac{\partial E}{\partial\beta^{(u')}}(X,\bsj)\exp(X^{\top}\aosj),\\
    A^{(j)}_{n,2,4}&:=\sum_{i\in\mathcal{C}_j}\exp(\azni)\frac{1}{8}(\dsni)^2\exp(X^{\top}\aosj),
\end{align*}
for all $j\in[K^\star]:|\mathcal{C}_j|>1$.

Next, we rewrite $B_n$ as $B_n=B_{n,1}+B_{n,2}$, where
\begin{align*}
    B_{n,1}&:=\sum_{j\in[K^\star]:|\mathcal{C}_j|=1}\sum_{i\in\mathcal{C}_j}\exp(\azni)[H_n(Y\mid X;\aoni)-H_n(Y\mid X;\aosj)]p(X),\\
    B_{n,2}&:=\sum_{j\in[K^\star]:|\mathcal{C}_j|>1}\sum_{i\in\mathcal{C}_j}\exp(\azni)[H_n(Y\mid X;\aoni)-H_n(Y\mid X;\aosj)]p(X).
\end{align*}
By applying first-order and second-order Taylor expansions, we have
\begin{align*}
    B_{n,1}&=\sum_{j\in[K^\star]:|\mathcal{C}_j|=1}\sum_{i\in\mathcal{C}_j}\exp(\azni)\sum_{u=1}^{d}(\daoni)^{(u)}X^{(u)}H_n(Y\mid X;\aosj)+R_{n,3}(Y\mid X),\\
    B_{n,2}&=\sum_{j\in[K^\star]:|\mathcal{C}_j|>1}\sum_{i\in\mathcal{C}_j}\exp(\azni)\Big[\sum_{u=1}^{d}(\daoni)^{(u)}X^{(u)}H_n(Y\mid X;\aosj),\\
    &\hspace{2cm}+\sum_{u,v=1}^{d}\frac{(\daoni)^{(u)}(\daoni)^{(v)}}{1+1_{\{u=v\}}}X^{(u)}X^{(v)}H_n(Y\mid X;\aosj)\Big]+R_{n,4}(Y\mid X),
\end{align*}
where $R_{n,\ell}(Y\mid X)$, for $\ell\in\{3,4\}$, are Taylor remainders such that $R_{n,\ell}(Y\mid X)/\mathcal{L}_{2n}\to0$ as $n\to\infty$.

\textbf{Step 2 - Non-vanishing coefficients.} In this step, we prove by contradiction that at least one among the coefficients in the representations of $[A_{n,1}-R_{n,1}(Y\mid X)]p(X)/\mathcal{L}_{2n}$, $[A_{n,2}-R_{n,2}(Y\mid X)]p(X)/\mathcal{L}_{2n}$, $[B_{n,1}-R_{n,3}(Y\mid X)]p(X)/\mathcal{L}_{2n}$, $[B_{n,2}-R_{n,4}(Y\mid X)]p(X)/\mathcal{L}_{2n}$, $C_{n}/\mathcal{L}_{2n}$ does not go to zero as $n\to\infty$. Assume that all these coefficients converge to zero. From the coefficients of the terms:
\begin{itemize}
    \item $\exp(X^{\top}\aosj)\mathcal{N}(Y|X;\aosj,\bsj,\ssj)$ for $j\in[K^\star]$, we get
    \begin{align*}
        \frac{1}{\mathcal{L}_{2n}}\cdot\sum_{j=1}^{K^\star}\Big|\sum_{i\in\mathcal{C}_{j}}\exp(\azni)-\exp(\aosj)\Big|\to0;
    \end{align*}
    \item $X^{(u)}\exp(X^{\top}\aosj)\mathcal{N}(Y|X;\aosj,\bsj,\ssj)$ for $j\in[K^\star]:|\mathcal{C}_{j}|=1$ and $u\in[d]$, we get
    \begin{align*}
        \frac{1}{\mathcal{L}_{2n}}\cdot\sum_{j\in[K^\star]:|\mathcal{C}_{j}|=1}\sum_{i\in\mathcal{C}_{j}}\exp(\azni)\|\daoni\|\to0;
    \end{align*}
    \item $\frac{\partial E}{\partial\beta^{(u')}}\exp(X^{\top}\aosj)\frac{\partial \mathcal{N}}{\partial E}(Y|X;\aosj,\bsj,\ssj)$ for $j\in[K^\star]:|\mathcal{C}_{j}|=1$ and $u'\in[d']$, we get
    \begin{align*}
        \frac{1}{\mathcal{L}_{2n}}\cdot\sum_{j\in[K^\star]:|\mathcal{C}_{j}|=1}\sum_{i\in\mathcal{C}_{j}}\exp(\azni)\|\dbni\|\to0;
    \end{align*}
    \item $\exp(X^{\top}\aosj)\frac{\partial^2 \mathcal{N}}{\partial E^2}(Y|X;\aosj,\bsj,\ssj)$ for $j\in[K^\star]:|\mathcal{C}_{j}|=1$, we get
    \begin{align*}
        \frac{1}{\mathcal{L}_{2n}}\cdot\sum_{j\in[K^\star]:|\mathcal{C}_{j}|=1}\sum_{i\in\mathcal{C}_{j}}\exp(\azni)|\dsni|\to0;
    \end{align*}
    \item $X^{(u)}X^{(v)}\exp(X^{\top}\aosj)\mathcal{N}(Y|X;\aosj,\bsj,\ssj)$ for $j\in[K^\star]:|\mathcal{C}_{j}|>1$ and $u,v\in[d]$, we get
    \begin{align*}
        \frac{1}{\mathcal{L}_{2n}}\cdot\sum_{j\in[K^\star]:|\mathcal{C}_{j}|>1}\sum_{i\in\mathcal{C}_{j}}\exp(\azni)\|\daoni\|^2\to0;
    \end{align*}
    \item $\big[\frac{\partial^2 E}{\partial\beta^{(u')}}(X,\bsj)\big]^2\exp(X^{\top}\aosj)\frac{\partial^2 \mathcal{N}}{\partial E^2}(Y|X;\aosj,\bsj,\ssj)$ for $j\in[K^\star]:|\mathcal{C}_{j}|>1$ and $u'\in[d']$, we get
    \begin{align*}
        \frac{1}{\mathcal{L}_{2n}}\cdot\sum_{j\in[K^\star]:|\mathcal{C}_{j}|>1}\sum_{i\in\mathcal{C}_{j}}\exp(\azni)\|\dbni\|\to0;
    \end{align*}
    \item $\exp(X^{\top}\aosj)\frac{\partial^4 \mathcal{N}}{\partial E^4}(Y|X;\aosj,\bsj,\ssj)$ for $j\in[K^\star]:|\mathcal{C}_{j}|>1$, we get
    \begin{align*}
        \frac{1}{\mathcal{L}_{2n}}\cdot\sum_{j\in[K^\star]:|\mathcal{C}_{j}|>1}\sum_{i\in\mathcal{C}_{j}}\exp(\azni)|\dsni|^2\to0;
    \end{align*}
\end{itemize}
By taking the sum of these limits, we have $1=\frac{1}{\mathcal{L}_{2n}}\cdot\mathcal{L}_{2n}\to0$ as $n\to\infty$, which is a contradiction. Therefore, not all the coefficients in the representations of $[A_{n,1}-R_{n,1}(Y\mid X)]p(X)/\mathcal{L}_{2n}$, $[A_{n,2}-R_{n,2}(Y\mid X)]p(X)/\mathcal{L}_{2n}$, $[B_{n,1}-R_{n,3}(Y\mid X)]p(X)/\mathcal{L}_{2n}$, $[B_{n,2}-R_{n,4}(Y\mid X)]p(X)/\mathcal{L}_{2n}$, $C_{n}/\mathcal{L}_{2n}$ converge to zero as $n\to\infty$.

\textbf{Step 3 - Fatou's argument.} In this step, we will show a contradiction to the results of Step 2. To begin with, we denote by $m_n$ the maximum of the absolute values of the coefficients in the representations of $[A_{n,1}-R_{n,1}(Y\mid X)]p(X)/\mathcal{L}_{2n}$, $[A_{n,2}-R_{n,2}(Y\mid X)]p(X)/\mathcal{L}_{2n}$, $[B_{n,1}-R_{n,3}(Y\mid X)]p(X)/\mathcal{L}_{2n}$, $[B_{n,2}-R_{n,4}(Y\mid X)]p(X)/\mathcal{L}_{2n}$, $C_{n}/\mathcal{L}_{2n}$. Then, the results of Step 2 indicate that $1/m_n\not\to\infty$ as $n\to\infty$. Next, let us denote
{\small
\begin{align*}
    \frac{1}{m_n\mathcal{L}_{2n}}\cdot\Big(\sum_{i\in\mathcal{C}_{j}}\exp(\azni)-\exp(\azsj)\Big)\to t_{0,j},& \quad \frac{1}{m_n\mathcal{L}_{2n}}\cdot\sum_{i\in\mathcal{C}_{j}}\exp(\azni)(\daoni)^{(u)}\to t^{(u)}_{1,j},\\
    \frac{1}{m_n\mathcal{L}_{2n}}\cdot\sum_{i\in\mathcal{C}_{j}}\exp(\azni)(\dbni)^{(u')}\to t^{(u')}_{2,j},& \quad \frac{1}{m_n\mathcal{L}_{2n}}\cdot\sum_{i\in\mathcal{C}_{j}}\exp(\azni)(\dsni)\to t_{3,j},\\
    \frac{1}{m_n\mathcal{L}_{2n}}\cdot\sum_{i\in\mathcal{C}_{j}}\exp(\azni)(\daoni)^{(u)}(\daoni)^{(v)}\to t^{(uv)}_{4,j},& \quad \frac{1}{m_n\mathcal{L}_{2n}}\cdot\sum_{i\in\mathcal{C}_{j}}\exp(\azni)(\dbni)^{(u')}(\dbni)^{(v')}\to t^{(u'v')}_{5,j},\\
    \frac{1}{m_n\mathcal{L}_{2n}}\cdot\sum_{i\in\mathcal{C}_{j}}\exp(\azni)(\dsni)^2\to t_{6,j},& \quad \frac{1}{m_n\mathcal{L}_{2n}}\cdot\sum_{i\in\mathcal{C}_{j}}\exp(\azni)(\daoni)^{(u)}(\dbni)^{(v')}\to t^{(uv')}_{7,j},\\
    \frac{1}{m_n\mathcal{L}_{2n}}\cdot\sum_{i\in\mathcal{C}_{j}}\exp(\azni)(\daoni)^{(u)}(\dsni)\to t^{(u)}_{8,j},& \quad \frac{1}{m_n\mathcal{L}_{2n}}\cdot\sum_{i\in\mathcal{C}_{j}}\exp(\azni)(\dbni)^{(u')}(\dsni)\to t^{(u')}_{9,j},
\end{align*}
}
for all $j\in[K^\star]$ as $n\to\infty$. The results of Step 2 indicate that not all these limits are zero. According to the Fatou's lemma, we have
\begin{align*}
    \lim_{n\to\infty}\frac{\|g_{G_n}-g_{G^\star}\|_{\mathscr L^1}}{m_n\mathcal{L}_{2n}}\geq\int\liminf_{n\to\infty}\frac{|g_{G_n}(Y,X)-g_{G^\star}(Y,X)|}{m_n\mathcal{L}_{2n}}\mathrm d(Y,X),
\end{align*}
Since $\|g_{G_n}-g_{G^\star}\|_{\mathscr L^1}/\mathcal{L}_{2}(G_n,G^\star)\to0$ as $n\to\infty$, we get $[g_{G_n}(Y,X)-g_{G^\star}(Y,X)]/[m_n\mathcal{L}_{2n}]\to0$ as $n\to\infty$ along a subsequence for almost every $(Y,X)$ (work on this subsequence from now on). Furthermore, as the term $\sum_{j=1}^{K^\star}\exp(\azsj+X^{\top}\aosj)$ is bounded, it follows that $D_n/[m_n\mathcal{L}_{2n}]\to0$ as $n\to\infty$, for almost every $(Y,X)$. From the decomposition of $D_n$ in Step 1, we have
\begin{align}
    \label{eq:zero_limit_over}
    \frac{1}{m_n\mathcal{L}_{2n}}\cdot[A_{n,1}+A_{n,2}-B_{n,1}-B_{n,2}+C_{n}]\to0.
\end{align}
For ease of presentation, let $F_{\gamma,j}:=\exp(X^{\top}\aosj)\frac{\partial^{\gamma}\mathcal{N}}{\partial E^{\gamma}}(Y\mid X;\aosj,\bsj,\ssj)$ and $H_{j}(Y\mid X):=\lim_{n\to\infty}H_{n}(Y\mid X;\aosj)$, for all $0\leq\gamma\leq 4$ and $j\in[K^\star]$. Then, we have
\begin{align*}
    &\lim_{n\to\infty}\frac{A_{n,1}}{m_n\mathcal{L}_{2n}}=\sum_{j\in[K^\star]:|\mathcal{C}_{j}|=1}\Big[\sum_{u=1}^{d}t_{1,j}^{(u)}X^{(u)}F_{0,j}(Y|X)+\sum_{u'=1}^{d'}t_{2,j}^{(u')}\frac{\partial E}{\partial\beta^{(u')}}(X,\bsj)F_{1,j}(Y|X)\\
    &\hspace{10cm}+\frac{1}{2}t_{3,j}F_{2,j}(Y|X)\Big]p(X),\\
    &\lim_{n\to\infty}\frac{A_{n,2}}{m_n\mathcal{L}_{2n}}=\sum_{j\in[K^\star]:|\mathcal{C}_{j}|>1}\Big[\Big(\sum_{u=1}^{d}t_{1,j}^{(u)}X^{(u)}+\sum_{u,v=1}^{d}t_{4,j}^{(uv)}X^{(u)}X^{(v)}\Big)F_{0,j}(Y|X)\\
    &+\Big(\sum_{u'=1}^{d'}t_{2,j}^{(u')}\frac{\partial E}{\partial\beta^{(u')}}(X,\bsj)+\sum_{u',v'=1}^{d'}t_{5,j}^{(u'v')}\frac{\partial^2 E}{\partial\beta^{(u')}\partial\beta^{(v')}}(X,\bsj)\\
    &\hspace{6cm}+\sum_{u=1}^{d}\sum_{v'=1}^{d'}t_{7,j}^{(uv')}X^{(u)}\frac{\partial E}{\partial\beta^{(v')}}(X,\bsj)\Big)F_{1,j}(Y|X)\\
    &+\Big(\frac{1}{2}t_{3,j}+\sum_{u',v'=1}^{d'}t_{5,j}^{(u'v')}\frac{\partial E}{\partial\beta^{(u')}}(X,\bsj)\frac{\partial E}{\partial\beta^{(v')}}(X,\bsj)+\sum_{u=1}^{d}\frac{1}{2}t_{8,j}^{(u)}X^{(u)}\Big)F_{2,j}(Y|X)\\
    &+\Big(\sum_{u'=1}^{d'}\frac{1}{2}t_{9,j}^{(u')}\frac{\partial E}{\partial\beta^{(u')}}(X,\bsj)\Big)F_{3,j}(Y|X)+\frac{1}{8}t_{6,j}F_{4,j}(Y|X)\Big]p(X),
\end{align*}
and
\begin{align*}
    &\lim_{n\to\infty}\frac{B_{n,1}}{m_n\mathcal{L}_{2n}}=\sum_{j\in[K^\star]:|\mathcal{C}_{j}|=1}\sum_{u=1}^{d}t_{1,j}^{(u)}X^{(u)}H_{j}(Y|X)p(X),\\
    &\lim_{n\to\infty}\frac{B_{n,2}}{m_n\mathcal{L}_{2n}}=\sum_{j\in[K^\star]:|\mathcal{C}_{j}|>1}\Big(\sum_{u=1}^{d}t_{1,j}^{(u)}X^{(u)}+\sum_{u,v=1}^{d}t_{4,j}^{(uv)}X^{(u)}X^{(v)}\Big)H_{j}(Y|X)p(X),\\
    &\lim_{n\to\infty}\frac{C_{n}}{m_n\mathcal{L}_{2n}}=\sum_{j=1}^{K^\star}t_{0,j}[F_{0,j}(Y|X)-H_{j}(Y|X)]p(X).
\end{align*}
Note that for almost every $X$, the set $\Big\{F_{\gamma,j}(Y|X), \ H_{j}(Y|X):0\leq\gamma\leq 4, \  j\in[K^\star]\Big\}$
is linearly independent with respect to $Y$. Therefore, it follows that the coefficients of these terms in the limit in \Eqref{eq:zero_limit_over} become zero. 

For $j\in[K^\star]$ such that $|\mathcal{C}_{j}|=1$, by considering the coefficients of 
\begin{itemize}
    \item $F_{0,j}(Y|X)$, we have $t_{0,j}+\sum_{u=1}^{d}t_{1,j}^{(u)}X^{(u)}=0$, for almost every $X$. Then, we deduce $t_{0,j}=t_{1,j}^{(u)}=0$ for all $u\in[d]$;
    \item $F_{1,j}(Y|X)$, we have $\sum_{u'=1}^{d'}t_{2,j}^{(u')}\frac{\partial E}{\partial\beta^{(u')}}(X,\bsj)$, for almost every $X$. As the expert function $E$ is second-order strongly identifiable, we get $t_{2,j}^{(u')}=0$ for all $u'\in[d']$;
    \item $F_{2,j}(Y|X)$, we have $t_{3,j}=0$.
\end{itemize}
For $j\in[K^\star]$ such that $|\mathcal{C}_{j}|>1$, by considering the coefficients of 
\begin{itemize}
    \item $F_{0,j}(Y|X)$, we have $t_{0,j}+\sum_{u=1}^{d}t_{1,j}^{(u)}X^{(u)}+\sum_{u,v=1}^{d}t_{4,j}^{(uv)}X^{(u)}X^{(v)}=0$, for almost every $X$. Then, we get $t_{0,j}=t_{1,j}^{(u)}=t_{4,j}^{(uv)}$ for all $u,v\in[d]$.
    \item $F_{1,j}(Y|X)$, we have
    \begin{align*}
        \sum_{u'=1}^{d'}t_{2,j}^{(u')}\frac{\partial E}{\partial\beta^{(u')}}(X,\bsj)+\sum_{u',v'=1}^{d'}t_{5,j}^{(u'v')}\frac{\partial^2 E}{\partial\beta^{(u')}\partial\beta^{(v')}}(X,\bsj)\\
        +\sum_{u=1}^{d}\sum_{v'=1}^{d'}t_{7,j}^{(uv')}X^{(u)}\frac{\partial E}{\partial\beta^{(v')}}(X,\bsj)=0,
    \end{align*}
    for almost every $X$. As the expert function $E$ meets the second-order strong identifiability condition, we get $t_{2,j}^{(u')}=t_{5,j}^{(u'v')}=t_{7,j}^{(uv')}=0$ for all $u',v'\in[d']$ and $u\in[d]$;
    \item $F_{2,j}(Y|X)$, we have 
    \begin{align*}
        \frac{1}{2}t_{3,j}+\sum_{u',v'=1}^{d'}t_{5,j}^{(u'v')}\frac{\partial E}{\partial\beta^{(u')}}(X,\bsj)\frac{\partial E}{\partial\beta^{(v')}}(X,\bsj)+\sum_{u=1}^{d}\frac{1}{2}t_{8,j}^{(u)}X^{(u)}=0,
    \end{align*}
    for almost every $X$. Since $t_{5,j}^{(u'v')}=0$ for all $u',v'\in[d']$, it follows that $\frac{1}{2}t_{3,j}+\sum_{u=1}^{d}\frac{1}{2}t_{8,j}^{(u)}X^{(u)}=0$, for almost every $X$. Then, we get $t_{3,j}=t_{8,j}^{(u)}=0$ for all $u',v'\in[d']$ and $u\in[d]$;
    \item $F_{3,j}(Y|X)$, we have $\sum_{u'=1}^{d'}\frac{1}{2}t_{9,j}^{(u')}\frac{\partial E}{\partial\beta^{(u')}}(X,\bsj)=0$, for almost every $X$. As the expert function $E$ is strongly identifiable, we get $t_{9,j}^{(u')}$ for all $u'\in[d']$;
    \item $F_{4,j}(Y|X)$, we have $t_{6,j}=0$.
\end{itemize}
Combining the above results, we deduce $t_{0,j}=t_{1,j}^{(u)}=t_{2,j}^{(u')}=t_{3,j}=t_{4,j}^{(uv)}=t_{5,j}^{(u'v')}=t_{6,j}=t_{7,j}^{(uv')}=t_{8,j}^{(u)}=t_{9,j}^{(u')}=0$ for all $j\in[K^\star]$, $u,v\in[d]$ and $u',v'\in[d']$. This contradicts the fact that at least one among them is non-zero. Therefore, we achieve the result in \Eqref{eq:local_part_TV_over}. Hence, the proof is completed.

% \textbf{Proof of \Eqref{eq:global_part}.} Assume by contrary that \Eqref{eq:global_part} does not hold. Then, we can find a sequence of mixing measure $(G_n)$ such that $\mathcal{L}_{2}(G_n,G^\star)>\varepsilon'$ and $\|g_{G_n}-g_{G^\star}\|_{\mathscr L^1}/\mathcal{L}_{2}(G_n,G^\star)\to0$ as $n\to\infty$. These two properties imply that 
% \begin{align*}
%     \|g_{G_n}-g_{G^\star}\|_{\mathscr L^1}\to0.
% \end{align*}
% Since the parameter space $\Theta$ is compact, we can substitute the sequence $(G_n)$ with its subsequence $(G'_n)$ that converges to some mixing measure $G'$. Recall that $\mathcal{L}_{2}(G_n,G^\star)>\varepsilon'$, then we also have $\mathcal{L}_{2}(G',G^\star)>\varepsilon'$. On the other hand, by the Fatou's lemma, we get
% \begin{align*}
%     0=\lim_{n\to\infty}\|g_{G'_n}-g_{G^\star}\|_{\mathscr L^1}&\geq\frac{1}{2}\int\liminf_{n\to\infty}|g_{G'_n}(Y,X)-g_{G^\star}(Y,x)|\mathrm d(Y,X)\\
%     &=\frac{1}{2}\int|g_{G'}(Y,X)-g_{G^\star}(Y,X)|\mathrm d(Y,X).
% \end{align*}
% This inequality indicates that $g_{G'}(Y,X)=g_{G^\star}(Y,X)$ for almost surely $(Y,X)$. Since the SMoGE model is identifiable, we deduce $G'\equiv G^\star$. As a result, we get $\mathcal{L}_{2}(G',G^\star)=0$, which contradicts the previous result that $\mathcal{L}_{2}(G',G^\star)>\varepsilon'>0$. Hence, the proof is completed.

\subsection{Proof of Theorem~\ref{thm:consistency_number_experts}}

Following the proof strategy of \cite{miller2023consistency} for simple mixture models, we prove the result by leveraging the famous posterior consistency theorem by \cite{doob1949application}.\footnote{Throughout the proof, we refer to the formulation of Doob's theorem presented in \cite{miller2018detailed} (Theorem 2.4), requiring measurability of the model with respect to parameters and parameter identifiability.} However, the latter requires a stronger notion of identifiability \textit{with respect to SMoGE parameters}. This, despite the identifiability of SMoGE densities \textit{with respect to mixing measures}, does not hold on $\Theta_\infty$ because mixing measures are invariant under (i) permutations of expert labels, and (ii) mergers and separations of equal mixing measure atoms. Therefore, we first construct a restricted space $\tilde\Theta_\infty$ where the assumptions of Doob's theorem (measurability of the model and parameter identifiability) hold; then we show that consistency on that space implies consistency on the original space $\Theta_\infty$, thanks to the assumptions on the prior.

Before delving into the proof, here is some useful notation. Given any $\theta\in\Theta^k\subseteq \Theta_\infty$ and any permutation $\rho:\{1,\dots, k\}\to\{1,\dots, k\}$, define $\theta[\rho]:= (\theta_{\rho(1)}, \dots, \theta_{\rho(k)})$. Moreover, recall the decomposition $\Theta = \Omega \times A$, where $\Omega$ denotes the space of parameters $\omega = (\alpha_1, \beta, \sigma^2)$, and $A$ is the space of gating bias parameters $\alpha_0$. Therefore, for all $k\in\NN$, $\Theta^k$ can be identified with $\Theta^k \equiv \Omega^k \times A^k$. Also, for all $\theta, \theta'\in \Theta_\infty$, let
\begin{equation*}
    d_{\Theta_\infty}(\theta, \theta'):=\begin{cases}
        \min\{\Vert \theta - \theta'\Vert_2, 1\} & \textnormal{if } K(\theta)=K(\theta'),\\
        1 & \textnormal{otherwise}.
    \end{cases}
\end{equation*}
By Propositions A.1 and A.2 of \cite{miller2023consistency}, this definition makes $(\Theta_\infty, d_{\Theta_\infty})$ a Borel measurable subset of a Polish space, which we endow with the corresponding Borel sigma-algebra.

\paragraph{Construction of the restricted model.} For all $k\in\NN$, define the restricted expert space
\begin{equation*}
    \Omega_k := \{(\omega_1,\dots,\omega_k)\in \Omega^k : \omega_1\prec\cdots \prec \omega_k\},
\end{equation*}
where, for all $\omega,\omega'\in \Omega$, $\omega\prec \omega'$ means that $\omega$ precedes $\omega'$ lexicographically, and $\omega\neq \omega'$. Now define $\Theta_k := \Omega_k \times A^k$ and $\tilde \Theta_\infty := \bigcup_{\ell\in \NN}\Theta_\ell$. For all $k\in\NN$ and $\theta\in\Theta^k$, let $T(\theta) = \theta[\rho]$, where the permutation $\rho$ is chosen such that $\theta[\rho]\in\Theta_k$ if possible (that is, if no two or more experts in $\theta$ share the same $\omega$ parameters), otherwise $\theta[\rho] = \theta$. Then, denoting $\boldsymbol{\theta}\sim \Pi$, we obtain  $\Pi(T(\boldsymbol{\theta})\in\tilde \Theta_\infty) = 1$ because the subset of $\Omega^k$ where two or more experts share the same parameters has prior probability zero by the assumptions of the Theorem. Denoting $B[\rho] = \{\theta\in\Theta^k : \theta[\rho]\in B\}$ for all $B\subseteq\Theta_k$, the definition of $T$ implies
\begin{equation}\label{eq:relationship}
T^{-1}(B) = \{\theta\in\Theta^k : T(\theta)\in B\} = \bigcup_{\rho\in R_k} B[\rho],
\end{equation}
where $R_k$ denotes the space of all permutations of $\{1,\dots,k\}$. Letting $\tilde{Q}$ denote the pushforward of $\Pi$ through $T$ (i.e., the distribution of $T(\boldsymbol{\theta})$ when $\boldsymbol{\theta}\sim \Pi$) restricted to $\tilde\Theta_\infty$, we obtain the \textit{restricted model}\footnote{We denote by $\PP$ the joint probability measure over infinite data sequences and parameters implied by this model.}
\begin{equation}\label{eq:restricted_model}
\begin{aligned}
(X_i, Y_i)\mid T(\boldsymbol{\theta}) & \overset{\textnormal{iid}}{\sim} g_{G(T(\boldsymbol{\theta}))}, \quad i=1,\dots,n\\
T(\boldsymbol{\theta}) & \sim \tilde{Q}
\end{aligned}
\end{equation}
by Theorem 10.2.1 of \cite{dudley2002real}. This is because $g_{G(\theta)} = g_{G(T(\theta))}$ and, for all measurable $C\subseteq (\XX\times \RR)^n$ and
$B\subseteq\tilde\Theta_\infty$, the following holds:
\begin{equation*}
    \PP((X_i, Y_i)_{i=1}^n\in C,\, T(\boldsymbol{\theta})\in B) = \PP((X_i, Y_i)_{i=1}^n\in C,\, \boldsymbol{\theta}\in T^{-1}(B)) = \int_B g_{G(\theta)}^{n}(C) \, \tilde{Q}(\mathrm d\theta),
\end{equation*}
where measurability of $\theta \mapsto g_{G(\theta)}^{(n)}(C)$ for all measurable $C\subseteq (\XX\times \RR)^n$ follows from the measurability of $\theta\mapsto g_{G(\theta)}(C')$ for all measurable $C'\subseteq\XX\times \RR$ (shown in the next paragraph) and from an application of Lemma 5.2 of \cite{miller2018detailed}.

\paragraph{Proof of measurability.} Let $C\subseteq \XX\times \RR$ be measurable and let $C_X$ denote its section at $X\in\XX$. Then for any $k\in\NN$ and $\theta\in \Theta^k$, by Fubini's theorem we can write
\begin{align}
\label{eq:measurable}
\theta\mapsto g_{G(\theta)}(C) & =\int_\XX \int_{C_X} f_{G(\theta)}(Y\mid X) \mathrm dY \,p(X) \mathrm dX.
\end{align}
By the compactness of $\XX$ and the Gaussianity of the expert densities, $f_{G(\theta)}(Y\mid X)$ can be bounded above by a function of $Y$ (constant across $X$) that is integrable on $\RR$ (e.g., a large enough multiple of a Gaussian density with high enough variance). So the dominated convergence theorem and the continuity of $\theta\mapsto f_{G(\theta)}(Y\mid X)$ for all $(X,Y)\in\XX\times \RR$ imply the continuity (hence measurability) of $\theta \mapsto g_{G(\theta)}(C)$ as a function on $\Theta^k$. Therefore, this mapping is also measurable as a function on $\Theta_k$ and, as a consequence, it is measurable as a function on $\tilde\Theta_\infty$ (since the pre-image of a measurable subset of $\RR$ is a union of measurable subsets of $\Theta_1,\Theta_2,\dots$, and is thus measurable by Proposition A.2 of \cite{miller2023consistency}).

\paragraph{Proof of parameter identifiability.} Choose $\theta\in\Theta_k\subset \tilde\Theta_\infty$, $\theta'\in\Theta_{k'}\subset \tilde\Theta_\infty$ such that $g_{G(\theta)} = g_{G(\theta')}$, where
\begin{align*}
    \theta & = (\alpha_{1j}, \beta_{j}, \sigma^2_j, \alpha_{0j})_{j=1}^k \equiv (\omega_j,\alpha_{0j})_{j=1}^k\\
    \theta' & = (\alpha'_{1j}, \beta'_{j}, \sigma'^2_j, \alpha'_{0j})_{j=1}^{k'} \equiv (\omega'_j,\alpha'_{0j})_{j=1}^{k'}
\end{align*}
By the identifiabilty assumption in the statement of the Theorem, this implies $G(\theta) = G(\theta')$. By the definition of $\tilde\Theta_\infty$, we have $\omega_j\neq \omega_\ell$ and $\omega_{j'}\neq \omega_{\ell'}$ for all $j\neq \ell$ and $j'\neq \ell'$ (with $j,\ell\in\{1,\dots,k\}$ and $j',\ell'\in\{1,\dots,k'\}$). Moreover, the boundedness of $A$ implies that all weights of $G(\theta)$ and $G(\theta')$ are strictly positive and finite. This implies that $k=k'$, $(\alpha_{01},\dots,\alpha_{0k}) = (\alpha'_{0\rho(1)},\dots,\alpha'_{0\rho(k)})$, and $(\omega_{1},\dots,\omega_k) = (\omega'_{\rho(1)},\dots,\omega'_{\rho(k')})$ for some $\rho\in R_k$. Further, because $\omega_1\prec\cdots\prec \omega_k$ and $\omega_1'\prec\cdots\prec \omega_k'$ by the definition of $\tilde\Theta_\infty$, it must be the case that $\rho$ is the identity function, implying $\theta = \theta'$.
Therefore, parameter identifiability holds on the restricted space $\tilde \Theta_\infty$.

\paragraph{Application of Doob's theorem on the restricted model.} We have shown that the restricted model in \Eqref{eq:restricted_model} satisfies the conditions of Doob's theorem. Therefore, there exists $\tilde\Theta_\star\subseteq\tilde\Theta_\infty$ such that 
$\PP(T(\boldsymbol{\theta})\in \tilde\Theta_\star) = 1$ and the restricted model is consistent at all $T(\theta_\star)\in\tilde\Theta_\star$. Equivalently, for any neighborhood $B\subseteq\tilde\Theta_\infty$ of $T(\theta_\star)$, we have $\PP(T(\boldsymbol{\theta})\in B \mid (X_i, Y_i)_{i=1}^n) \to 1$  a.s.-$g_{G(T(\theta_0))}^\infty$. Now define $\Theta_\star := \bigcup_{k\in\NN} \bigcup_{\rho\in R_k} (\tilde\Theta_\star\cap\Theta_k)[\rho]$. Then, by \Eqref{eq:relationship},
\begin{align*}
\Pi(\Theta_\star)\equiv\PP(\boldsymbol{\theta}\in\Theta_\star) = \PP(T(\boldsymbol{\theta})\in\tilde\Theta_\star) = 1.
\end{align*}

\paragraph{Implications for the unrestricted model.}
Let $\theta_\star\in\Theta_\star$ and define $K^\star = K(\theta_\star)$. 
Let $(X_i, Y_i) \overset{\textnormal{iid}}{\sim} g_{G(\theta_\star)}$  for $i\in\NN$, and define $B := \{\theta\in\tilde\Theta_\infty : d_{\Theta_\infty}(\theta,T(\theta_\star)) < \varepsilon\} \subseteq \Theta_{K^\star}$ for $\varepsilon\in (0,1)$. Observe that 
$\bigcup_{\rho\in R_{K^\star}} B[\rho] \subseteq \tilde{B}(\theta_\star,\varepsilon)$, where
\begin{equation*}
    \tilde{B}(\theta_\star,\varepsilon) := \bigcup_{\rho\in R_{K^\star}}\{\theta \in \Theta_\infty : d_{\Theta_\infty}(\theta, \theta_\star[\rho]) < \varepsilon\}.
\end{equation*}
So by \Eqref{eq:relationship},
\begin{equation}\label{eq:convergence}
\begin{aligned}
\Pi(\boldsymbol{\theta}\in\tilde{B}(\theta_\star,\varepsilon) \mid (X_i, Y_i)_{i=1}^n) & \equiv
\PP(\boldsymbol{\theta}\in\tilde{B}(\theta_\star,\varepsilon) \mid (X_i, Y_i)_{i=1}^n)\\
& \geq \PP\bigg(\boldsymbol{\theta}\in\bigcup_{\rho\in R_{K^\star}} B[\rho] \mid (X_i, Y_i)_{i=1}^n\bigg) \\
& = \PP(T(\boldsymbol{\theta})\in B \mid (X_i, Y_i)_{i=1}^n) \\
& \to 1 \quad \textnormal{as } n\to \infty,
\end{aligned}
\end{equation}
a.s.-$g_{G(\theta_\star)}^\infty$, since $g_{G(\theta_\star)} = g_{G(T(\theta_\star))}$ and the restricted model is consistent at all $T(\theta_\star)\in\tilde \Theta_\star$. Finally, consistency for the number of experts follows immediately from \Eqref{eq:convergence}, since $\varepsilon<1$ implies $\tilde{B}(\theta_\star,\varepsilon) \subseteq \Theta^{K^\star}$ and therefore
\begin{align*}
\Pi(K=K^\star \mid (X_i, Y_i)_{i=1}^n) & \equiv
\PP(\boldsymbol{\theta} \in \Theta^{K^\star} \mid (X_i, Y_i)_{i=1}^n) \\
& \geq \PP(\boldsymbol{\theta}\in\tilde{B}(\theta_\star,\varepsilon) \mid (X_i, Y_i)_{i=1}^n)\\
& \to 1 \quad \textnormal{as } n\to \infty,
\end{align*}
a.s.-$g_{G(\theta_\star)}^\infty$

\subsection{Proof of Corollary~\ref{cor:density_estimation_doobs}}
Note that
\begin{align*}
    \Pi & \left(\left\{G\in \bigcup_{j\in\NN}\mathscr G_{j} : \,d_H(g_G, g_{G^\star}) \geq M_n\sqrt{\log n/n}\right\} \mid (X_i, Y_i)_{i=1}^n\right) \\
    & = \sum_{j\in\NN} \Pi\left(\left\{G\in \mathscr G_{j} : \,d_H(g_G, g_{G^\star}) \geq M_n\sqrt{\log n/n}\right\} \mid (X_i, Y_i)_{i=1}^n\right) \Pi(K=j\mid (X_i, Y_i)_{i=1}^n) \\
    & \leq \Pi\left(\left\{G\in \mathscr G_{K^\star} : \,d_H(g_G, g_{G^\star}) \geq M_n\sqrt{\log n/n}\right\} \mid (X_i, Y_i)_{i=1}^n\right) \\
    &+ \Pi(K\neq K^\star\mid (X_i, Y_i)_{i=1}^n),
\end{align*}
where the second addendum asymptotically vanishes almost surely by Theorem \ref{thm:consistency_number_experts}, while the first term can be dealt with exactly as in the proof of Theorem \ref{thm:our_abstract_rate_SMoGE}.

\subsection{Proof of Proposition~\ref{prop:identifiability_up_to_translation}}
\label{appendix:identifiability}
Firstly, we expand equation $g_{G}(Y,X)=g_{G^\star}(Y,X)$ for almost every $(Y,X)$ as follows:
    \begin{align}
         \sum_{j=1}^{K} \frac{\exp(\alpha_{0j} + X^\top\alpha_{1j})}{\sum_{\ell=1}^{K}\exp(\alpha_{0\ell} + X^\top\alpha_{1\ell})} \mathcal{N}(Y \mid E(X,\beta_{j}), \sigma_j^{2})\nonumber\\
        \label{eq:identifiable_equation}
        = \sum_{j=1}^{K^\star} \frac{\exp(\alpha_{0j}^\star + X^\top\alpha_{1j}^\star)}{\sum_{\ell=1}^{K^\star}\exp(\alpha_{0\ell}^\star + X^\top\alpha_{1\ell}^\star)} \mathcal{N}(Y \mid E(X,\beta_{j}^\star), \sigma_j^{\star 2}).
    \end{align}
As the mixture of location-scale Gaussian distributions is identifiable \citep{Teicher-1963}, it follows that $K=K^\star$ and
\begin{align*}
    \left\{\frac{\exp(\alpha_{0j} + X^\top\alpha_{1j})}{\sum_{\ell=1}^{K}\exp(\alpha_{0\ell} + X^\top\alpha_{1\ell})}:j\in[K]\right\}=\left\{\frac{\exp(\alpha_{0j}^\star + X^\top\alpha_{1j}^\star)}{\sum_{\ell=1}^{K^\star}\exp(\alpha_{0\ell}^\star + X^\top\alpha_{1\ell}^\star)}:j\in[K]\right\},
\end{align*}
for almost every $X$. Without loss of generality, we may assume that
\begin{align*}
    \frac{\exp(\alpha_{0j} + X^\top\alpha_{1j})}{\sum_{\ell=1}^{K}\exp(\alpha_{0\ell} + X^\top\alpha_{1\ell})}=\frac{\exp(\alpha_{0j}^\star + X^\top\alpha_{1j}^\star)}{\sum_{\ell=1}^{K^\star}\exp(\alpha_{0\ell}^\star + X^\top\alpha_{1\ell}^\star)},
\end{align*}
for almost every $X$ and for all $j\in[K]$. Since the softmax function is invariant to translation, we deduce $\alpha_{0j}=\alpha^\star_{0j}+t_0$ and $\alpha_{1j}=\alpha^\star_{1j}+t_1$, for some $t_0\in\mathbb{R}$ and $t_1\in\mathbb{R}^{d}$, for all $j\in[K]$. Thus, we can rewrite \Eqref{eq:identifiable_equation} as
\begin{align}
    \label{eq:identifiable_equation_2}
    \sum_{j=1}^{K}\exp(\alpha_{0j})F(Y|X;\alpha_{1j},\beta_{j},\sigma^2_{j})=\sum_{j=1}^{K^\star}\exp(\alpha_{0j})F(Y|X;\alpha_{1j},\beta^\star_{j},\sigma^{\star 2}_{j}),
\end{align}
for almost every $(Y,X)$, where we define $F(Y|X;\alpha_{1},\beta,\sigma^2):=\exp(X^{\top}\alpha_{1})\mathcal{N}(Y|E(X,\beta),\sigma^2)$. Next, we partition the index set $[K]$ into $m$ subsets $J_1,J_2,\ldots,J_m$ such that $\exp(\alpha_{0j})=\exp(\alpha_{0j'})$ for any $j,j^\star\in J_i$ and for $i\in[m]$. Meanwhile, if $j$ and $j^\star$ are not in the same subset, then we let $\exp(\alpha_{0j})\neq\exp(\alpha_{0j'})$. Then, we rewrite \Eqref{eq:identifiable_equation_2} as
\begin{align*}
    \sum_{i=1}^{m}\sum_{j\in J_i}\exp(\alpha_{0j})F(Y|X;\alpha_{1j},\beta_{j},\sigma^2_{j})=\sum_{i=1}^{m}\sum_{j\in J_i}\exp(\alpha_{0j})F(Y|X;\alpha_{1j},\beta^\star_{j},\sigma^{\star 2}_{j}),
\end{align*}
for almost every $(Y,X)$. The above equation implies that 
\begin{align*}
    \{(E(X,\beta_{j}),\sigma^2_{j}):j\in J_i\}=\{(E(X,\beta^\star_{j}),\sigma^{\star 2}_{j}):j\in J_i\},
\end{align*}
for almost every $X$ for all $i\in[m]$. Since the expert function $E$ is identifiable, it follows that $\{(\beta_j,\sigma^2_j:j\in J_i\}=\{(\beta^\star_{j},\sigma^{\star 2}_{j}):j\in J_i\}$ for all $i\in[m]$. As a result, we have
\begin{align*}
    G=\sum_{i=1}^{m}\sum_{j\in J_i}\exp(\alpha_{0j})\delta_{(\alpha_{1j},\beta_j,\sigma^2_j)}=\sum_{i=1}^{m}\sum_{j\in J_i}\exp(\alpha^\star_{0j}+t_0)\delta_{(\alpha^\star_{1j}+t_1,\beta^\star_j,\sigma^{\star 2}_j)}=G^\star_{t_0,t_1}.
\end{align*}
Hence, the proof is completed.

\section{Experiment Details}\label{app:elbo_experiments}

This appendix details the experimental design and variational inference methodology used to generate the model selection results presented in Section~\ref{sec:model_selection} of the main text.

\subsection{Black-Box Variational Inference Implementation}
To fit the candidate models, we utilize a black-box variational inference (BBVI) framework across all experiments. 

\textbf{Variational family:} We specify a fully factorized (mean-field) Gaussian variational family $\mathcal{Q}$ over the transformed parameter space. Specifically, we place Gaussian variational distributions over the gating intercepts $\alpha_0$, gating slopes $\alpha_1$, regression coefficients $\beta$, and the log-variances $\log \sigma^2$. The variational parameters $\mu$ therefore consist of the means and log-standard-deviations for each of these independent Gaussian distributions. 

\textbf{Optimization:} The ELBO is maximized using a pathwise (reparameterization) gradient estimator. During training, we draw standard normal samples $\epsilon \sim \mathcal{N}(0, 1)$ and apply the location-scale transformation (e.g., $\theta = \mu_\theta + \epsilon \exp(\log \sigma_\theta)$) to obtain differentiable samples from the variational posterior. We perform full-batch optimization using the Adam optimizer. 

\textbf{Hyperparameters and evaluation:} Zero-mean Gaussian priors are placed on the gating and regression parameters (with prior variances $\tau_1=\tau_2=\tau_3=10.0$), alongside an Inverse-Gamma(2, 2) prior for the expert variances. Upon completion of the optimization routine, the final ELBO for each candidate model is estimated using 200 Monte Carlo samples from the optimized variational posterior $q_{\mu^\star}$. The model attaining the highest estimated ELBO is selected as the winner for that simulation round.

\subsection{Figure~\ref{fig:elbo_k2} experiment}
For the first experiment, we simulate covariates uniformly from a two-dimensional square, $X_i \sim \mathcal{U}(-1, 1)^2$. The true gating mechanism assigns observations to the first expert if $X_{i1} > X_{i2}$, and to the second expert otherwise. Conditional on the expert assignment, the responses $Y_i$ are drawn from a Gaussian linear regression model. The true parameters for the first expert are set to intercept $\beta_{0,1}=2.0$, slopes $\beta_{1,1}=[1.0, 1.0]^\top$, and variance $\sigma^2_1=1.0$. For the second expert, the parameters are $\beta_{0,2}=-2.0$, slopes $\beta_{1,2}=[-1.0, -1.0]^\top$, and variance $\sigma^2_2=2.0$. This creates a sharp discontinuity in the conditional mean surface.

We evaluate candidate models $K \in \{1, 2, 3, 4\}$ across sample sizes $n \in \{10, 25, 50, 100\}$, running 50 independent replications per size. For this setup, the Adam optimizer is run for 50,000 iterations. To ensure stable convergence, learning rates are tailored to each $(n, K)$ pair, ranging from 0.0036 to 0.015.

\subsection{Figure~\ref{fig:elbo_k4} experiment}
For the more complex experiment, we increase the dimensionality to $d=6$ and simulate covariates uniformly, $X_i \sim \mathcal{U}(-1, 1)^6$. The true model relies on $K^\star=4$ experts. The gating mechanism is deterministic and assigns each observation based on the maximum value among the first four covariates: $Z_i = \text{argmax}_{k \in \{1, 2, 3, 4\}} X_{i,k}$. 

Conditional on the assignment $Z_i=k$, the responses $Y_i$ are drawn from a Gaussian linear regression model. The true intercepts alternate in sign, set to $\beta_{0,k} = 2(-1)^{k-1}$. The slope vectors $\beta_{1,k}$ are predominantly diagonal, with the $k$-th component set to $2(-1)^{k-1}$, while all components are additionally perturbed by independent Gaussian noise $\mathcal{N}(0, 0.2^2)$. The expert variances $\sigma^2_k$ are linearly spaced between 1.0 and 2.0. 

We evaluate this configuration across larger sample sizes $n \in \{100, 500, 1000, 2000\}$, running 50 independent replications per size. Candidate models are extended to $K \in \{1, \dots, 6\}$. Due to the increased complexity and larger sample sizes, the Adam optimizer is run for 10,000 iterations per model, and the base learning rates are scaled linearly with the sample size to ensure stable convergence.

\subsection{Table~\ref{tab:elbo_tables} experiment}
We simulate $n=500$ covariate vectors uniformly, $X_i \sim \mathcal{U}(-1, 1)^d$, and evaluate three combinations of dimensionality and true expert counts: $(d=2, K^\star=1)$, $(d=2, K^\star=2)$, and $(d=4, K^\star=3)$. The gating assignments are determined by evaluating a set of linear logits, $X_i W^\top + b$. The weight matrix $W$ is strictly diagonal for the first $K^\star$ dimensions, with $W_{k,k}$ set to a predefined `separation` constant. The biases are set to $b_k = -0.2 \times \text{separation} \times k$. The final expert assignment relies on a hard-max over these logits: $Z_i = \text{argmax}_k (X_i W^\top + b)_k$. We test two separation scales: 5.0 (low sharpness, Table~\ref{tab:elbo_softmax_low}) and 10.0 (high sharpness, Table~\ref{tab:elbo_softmax_high}).

Conditional on the assignment $Z_i=k$, the responses $Y_i$ are drawn from a Gaussian linear regression model. The true intercepts are linearly spaced in $[-2.0, 2.0]$. The slopes $\beta_{1,k}$ are given by a base diagonal value of $2(-1)^{k-1}$, corrupted by independent Gaussian noise $\mathcal{N}(0, 0.3^2)$ across all dimensions. The true variances are fixed at $\sigma^2_k = 0.8$ for all experts. 

Candidate models evaluated range from $K=1$ to $K=7$. BBVI is run for 4,000 iterations per model (10,000 for models evaluated against the $K^\star=1$ baseline to ensure adequate convergence). The Adam learning rate is set to an adaptive schedule, $0.1 + 0.000015 n + 0.001 K$, with the exception of the $K^\star=1$ baseline models, which are fixed at a rate of 0.06. Each configuration is run over 100 independent replications.

\clearpage
\bibliographystyle{apalike}
\bibliography{arxiv.bib}

\end{document}